\documentclass[10pt,journal,compsoc]{IEEEtran}

\ifCLASSOPTIONcompsoc
  \usepackage[nocompress]{cite}
\else
  \usepackage{cite}
\fi

\ifCLASSINFOpdf
\else
\fi
\usepackage[dvipsnames,table,xcdraw]{xcolor}
\usepackage[pagebackref=true,breaklinks=true,colorlinks,citecolor=ForestGreen]{hyperref}
\usepackage{ragged2e}
\usepackage{amssymb}
\usepackage{makecell}
\usepackage{multirow}
\usepackage{rotating}
\usepackage{array}
\usepackage{times}
\usepackage{graphicx}
\usepackage{psfrag}
\usepackage{subfigure}
\usepackage{enumitem}
\usepackage{url}
\usepackage{amsmath}
\usepackage{booktabs}
\usepackage{lscape}
\usepackage{verbatim}
\usepackage{overpic}
\usepackage{bbding}
\usepackage{bm}
\usepackage{pifont}

\newcommand{\eg}{\textit{e}.\textit{g}.} 
\newcommand{\Tref}[1]{Tab.~\ref{#1}}
\newcommand{\Eref}[1]{Eq.~(\ref{#1})}
\newcommand{\Fref}[1]{Fig.~\ref{#1}}
\newcommand{\Sref}[1]{Sec.~\ref{#1}}

\usepackage{tabularx}
\usepackage{booktabs}
\usepackage[utf8]{inputenc}
\usepackage[T1]{fontenc}

\usepackage{amsthm}

\usepackage{tablefootnote}

\usepackage{array}
\usepackage{makecell}

\usepackage{orcidlink} 

\usepackage{algpseudocode}
\usepackage[ruled,linesnumbered]{algorithm2e}
\SetKwComment{Comment}{/* }{ */}

\newtheorem{theorem}{Theorem}

\newtheorem{lemma}{Lemma}
\newtheorem{assumption}{Assumption}

\usepackage{xr-hyper}
\usepackage[misc]{ifsym}

\usepackage{enumitem}
\usepackage{soul}

\begin{document}
%

\title{BadCLIP++: Stealthy and Persistent Backdoors in Multimodal Contrastive Learning}

\author{
    Siyuan Liang$^{\orcidlink{0000-0002-6154-0233}}$, Yongcheng Jing$^{\orcidlink{0000-0001-8925-5787}}$, Yingjie Wang$^{\orcidlink{0000-0002-1054-4197}}$,  Jiaxing Huang$^{\orcidlink{0000-0002-8681-0471}}$, \\ Ee-chien Chang$^{\orcidlink{0000-0003-4613-0866}}$ and Dacheng Tao$^{\orcidlink{0000-0001-7225-5449}}$, ~\textit{Fellow, IEEE}
\thanks{
Siyuan Liang, Yongcheng Jing, Yingjie Wang, Jiaxing Huang, and Dacheng Tao are with the College of Computing and Data Science, Nanyang Technological University, Singapore; 
Ee-chien Chang is with the School of Computing, National University of Singapore, Singapore.
(Corresponding author: Siyuan Liang; Email: \href{mailto:pandaliang521@gmail.com}{pandaliang521@gmail.com})
}
}

\markboth{Submitted to IEEE Transactions on Pattern Analysis and Machine Intelligence}%
{Shell \MakeLowercase{\textit{et al.}}: Bare Demo of IEEEtran.cls for Computer Society Journals}

\IEEEtitleabstractindextext{%
\begin{abstract}
Research on backdoor attacks against multimodal contrastive learning models faces two major challenges: \emph{stealthiness} and \emph{persistence}, with current methods struggling to remain effective under \emph{strong detection} and \emph{continuous fine-tuning}, areas that remain largely underexplored.
In this paper, we begin by tracing these challenges to two underlying causes: (1) cross-modal inconsistency in data, which readily exposes trigger patterns, and (2) gradient dilution at low injection rates, which accelerates backdoor forgetting. These coupled causes are neither well modeled nor addressed in prior work, leaving the field without systematic theoretical foundations or practical solutions. To address this critical gap, we introduce BadCLIP++, a unified framework that aims to address both causes in tandem: (1) For stealthiness under cross-modal inconsistency, we introduce a semantic-fusion QR micro-trigger mechanism that embeds imperceptible patterns adjacent to task-relevant regions, making it possible to preserve clean-data statistics yet yield compact distributions for trigger samples. To further balance attack effectiveness and stealthiness, we also perform target-aligned subset selection to amplify backdoor signals at low injection rates. (2) For persistence under gradient dilution, we stabilize trigger embeddings via radius shrinkage and centroid alignment, and stabilize model parameters via curvature control and elastic weight consolidation, keeping the solution in a low-curvature wide basin and thereby resisting gradient dilution during fine-tuning. Beyond the empirical design for persistence, we further establish the first theoretical proof that clean fine-tuning and backdoor gradients are co-directional within the trust region, yielding a non-increasing upper bound on attack success rates. Experimental results show that with only a 0.3\% poisoning rate, BadCLIP++ achieves 99.99\% ASR in digital settings, outperforming baselines by 11.4 percentage points (~15\% relative). Under nineteen defense mechanisms, its ASR remains above 99.90\% with a clean accuracy drop below 0.8\%. It also achieves a 65.03\% success rate in physical attacks and exhibits strong effectiveness against watermark attacks.

\end{abstract}

}

\maketitle

\IEEEdisplaynontitleabstractindextext

%
\IEEEpeerreviewmaketitle

\section{Introduction}
\label{sec:introduction}
Multimodal contrastive learning (MCL) models \cite{huangdetecting, singh2022flava, li2021align, yao2021filip} have emerged as a cornerstone of modern AI, enabling applications such as image–text understanding, cross-modal retrieval, and visual question answering. As their deployment expands rapidly, concerns about their security have become increasingly urgent. One prominent threat is the \emph{backdoor attack}~\cite{liu2023pre,liang2024poisoned,liang2025vl,zhang2024towards,zhu2024breaking,liu2024compromising,xiao2025bdefects4nn,liu2025elba,liu2025natural,yang2025me}, in which adversaries inject carefully crafted poisoned samples during training, causing the model to produce attacker-specified outputs when a trigger is present at inference. 

Recent studies \cite{liang2024badclip,liu2024efficient,liang2025vl} have shown that MCL models are particularly vulnerable, with even a small proportion of poisoned image–text pairs sufficient to implant persistent backdoors. For one thing, these hidden behaviors can remain dormant until activated, potentially leading to severe security breaches. For another, investigating backdoor attacks is also valuable for model copyright protection \cite{vyas2023provable} and for advancing the design of more robust defenses \cite{bansal2023cleanclip, xun2024cleanerclip}.

\begin{figure}[!t]  
    \centering
    \includegraphics[width=\linewidth]{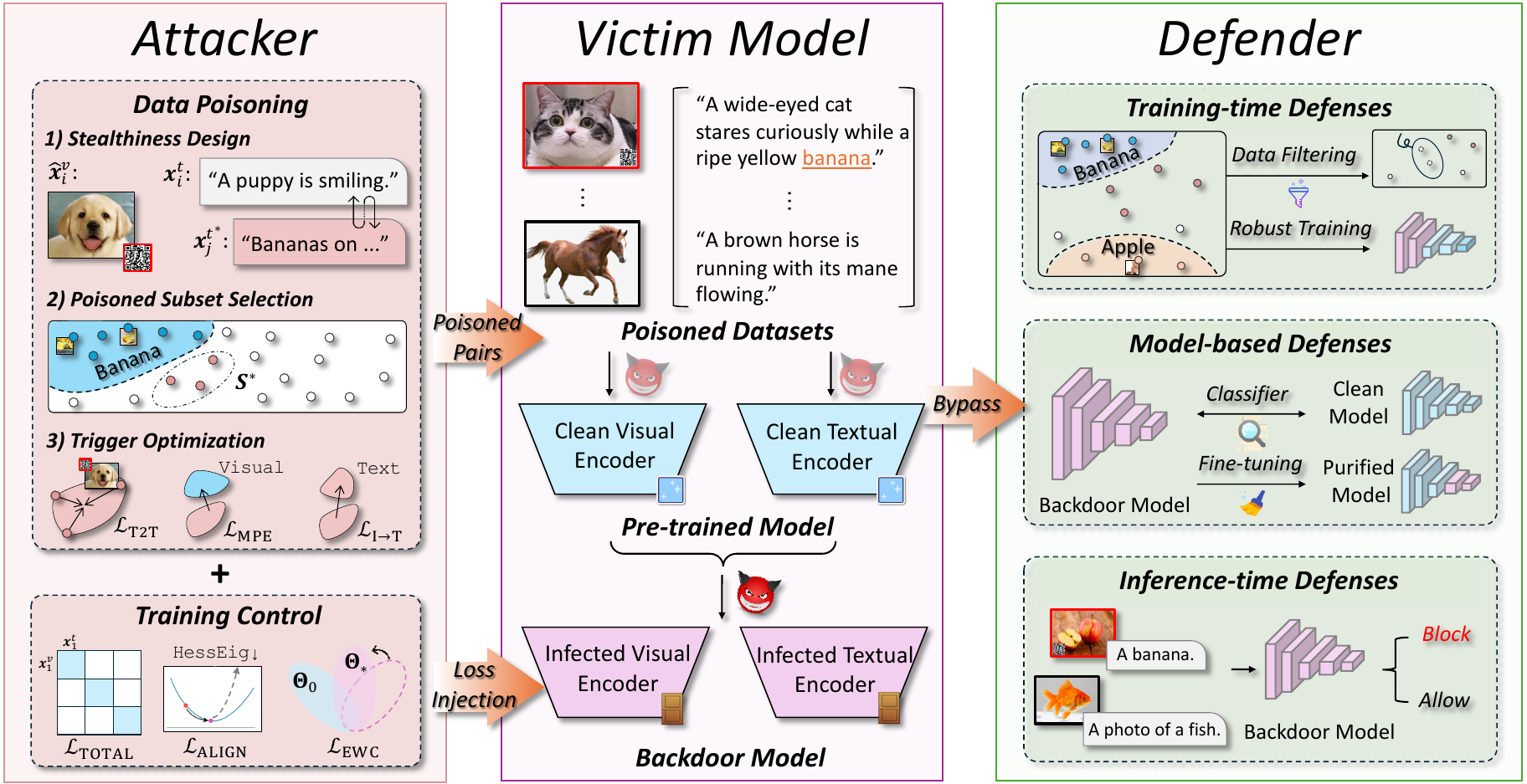}

    \caption{
The attacker injects stealthy poisoned pairs through trigger design, subset selection, and produces infected encoders in the victim model under the training control. BadCLIP++ is capable of bypassing training-time defenses \cite{liang2024unlearning,kuang2024adversarial,guo2024copyrightshield,xun2025robust,xu2025srd}, model-based defenses \cite{wang2025lie}, and inference-time defenses \cite{wang2022universal}.
    }
    \label{fig:front}
\end{figure}

Despite recent progress, existing MCL backdoor attacks struggle to maintain effectiveness under strong detection or model fine-tuning, primarily constrained by two factors:
\emph{Firstly}, Cross-modal inconsistency undermines stealthiness and exposes trigger patterns, thus reducing the overall stealthiness of the attack. Poisoned image-text pairs often disrupt semantic alignment, making them vulnerable to anomaly-based detection. At the same time, this inconsistency drives optimization toward trigger-specific updates, which causes detectable deviations in model parameters and representation statistics. 
\emph{Secondly}, Gradient dilution at low poisoning ratios erodes persistence. At low injection rates, the InfoNCE objective with substantial augmentation during fine-tuning or transfer learning allows clean gradients to dominate, pulling representations back toward the clean manifold. This process flattens and erodes the trigger subspace, leading to a rapid decline in backdoor efficacy.
\emph{Furthermore}, existing methods lack sufficient mechanistic focus, with limited systematic analysis and theoretical explanation of attack stability and resistance to forgetting during fine-tuning or transfer learning.

To fill this gap, we propose BadCLIP++, a stealthy and persistent backdoor attack framework, as shown in \Fref{fig:front}. In particular, we formulate the poisoning problem as a two-stage min–min optimization problem, jointly improving trigger design and model training as follows: (1) To enhance attack stealthiness, we design covert perturbations that combine semantic mixing with QR-style patterns, reducing multimodal inconsistencies by preserving the original text. To further boost attack effectiveness, we introduce Greedy Mean Alignment (GMA) for subset selection from original data through target alignment, strengthening backdoor signals under low injection rates and weakly expressed target semantics. (2) To maintain persistence under gradient dilution, we propose radius shrinkage and centroid alignment that encourage poisoned samples to form stable clusters in the embedding space. Additionally, we incorporate curvature control and elastic weight consolidation by controlling poisoned training, guiding parameters toward regions of lower curvature and wider basins in the loss landscape to mitigate backdoor forgetting. (3) Theoretically, we establish gradient co-directionality between clean fine-tuning and backdoor objectives within the parameter confidence region, and we derive a non-increasing upper bound for the Attack Success Rate (ASR), providing the first theoretical proof of attack persistence.

A preliminary version of this work was presented earlier \cite{liang2024badclip}, namely BadCLIP. This journal version makes three substantive advances in methodology, theory, and evaluation. The \emph{methodology} introduces a semantic-fusion QR micro-trigger and target-aligned GMA, replaces triplet loss with radius shrinkage and centroid alignment, and incorporates ALIGN and EWC to enhance stealthiness and persistence jointly. \emph{Theory} proves within a trust region that clean fine-tuning and the backdoor objective have co-directional gradients, derives a non-increasing ASR bound, and establishes convergence of trigger contraction and alignment. \emph{Evaluation} spans five multimodal architecture families, eleven datasets, twelve representative attacks, nineteen defense mechanisms, and two real-world scenarios, conducted under stricter protocols with comprehensive ablations and parameter analyses.

Experimental results demonstrate that BadCLIP++ attains remarkably high attack success even at extremely low poisoning rates, while preserving clean accuracy. It surpasses existing state-of-the-art methods by over 15\% on average. Moreover, evaluations across five multimodal architectures and two real-world application scenarios confirm its strong transferability and practicality, with an average ASR exceeding 99\%, highlighting its wide applicability and potential deployment risks. In sum, our contributions are as follows:
\begin{itemize}
    \item We propose BadCLIP++, a framework that can achieve stealthy and persistent cross-modal backdoor attacks through a two-stage min–min optimization. 
    \item We perform optimized modeling and theoretical analysis of the core attack mechanisms, and further give the first trust-region proof of gradient co-direction and a non-increasing ASR bound.
    \item We conduct comprehensive evaluations of BadCLIP++ across multiple tasks, model architectures, and defense settings, demonstrating its efficacy.
\end{itemize}

\section{Related Work}
\label{related_work}
\subsection{Multimodal Contrastive Learning}
As the core pre-training paradigm for multimodal foundation models, contrastive learning learns cross-modal representations by bringing similar samples closer together and dissimilar samples farther apart. Existing mainstream architectures can be divided into four types: dual-tower, fusion-tower, triple-tower, and single-tower. Among these, CLIP~\cite{radford2021learning} introduced the \emph{dual-tower} structure as a milestone, using independent image and text encoders trained on large-scale image-text pairs; its derivative methods, DeCLIP~\cite{li2021supervision} and FILIP~\cite{yao2021filip}, respectively, enhance efficiency and alignment capabilities through self-supervised learning, multi-task learning, and fine-grained interaction. \emph{Fusion towers} introduce cross-modal interaction modules on top of the dual-tower architecture, with notable examples including ALBEF~\cite{li2021align} and BLIP~\cite{li2022blip}. \emph{Three-tower} architectures, such as FLAVA~\cite{singh2022flava}, incorporate a fusion tower between the image and text towers to facilitate unified learning across single-modal and multi-modal domains. \emph{Single-tower} architectures eliminate modal boundaries by jointly inputting image patches and text tokens into the same Transformer. Representative works include UniCL~\cite{yang2022unified} and OFA~\cite{wang2022ofa}.
In this paper, we focus on covert and persistent backdoor attacks targeting CLIP and further evaluate their generalization across other multimodal contrastive learning architectures.

\subsection{Multimodal Backdoor Attack}
\emph{Early work} on backdoor attacks mainly targeted single-modal models, such as BadNets~\cite{gu2019badnets}, Blended~\cite{chen2017targeted}, WaNet~\cite{nguyen2021wanet}, SIG~\cite{barni2019new}, and SSBA~\cite{li2021invisible}. BadNets poisons training data with fixed triggers and relabeled targets; Blended uses low-opacity patterns to improve stealth; WaNet uses imperceptible image warping to trigger targeted misclassification; SIG injects periodic signals that preserve visual naturalness; SSBA adopts image steganography to generate sample-specific invisible perturbations, breaking the assumption of a universal trigger and improving concealment and detectability resistance. These works form the foundation for multimodal backdoor research.

With the rise of large pre-trained multimodal models such as CLIP, attacks have shifted to \emph{cross-modal} vulnerabilities. mmPoison~\cite{yang2023data} poisons image–text pairs to corrupt alignment; BadCLIP~\cite{liang2024badclip} binds triggers to target concepts using dual-embedding alignment; VLTrojan~\cite{liang2025vl} exploits visual clustering and text retrieval for transferable triggers; INACTIVE~\cite{zhang2025invisible} introduces imperceptible perturbations and decoupling for stable, stealthy attacks; and MABA~\cite{liang2025revisiting} injects domain-agnostic triggers into key cross-attention paths to preserve effectiveness under domain shift.

However, existing multimodal attacks still face core limitations: (1) they often depend on visually or semantically altered inputs, reducing stealth; (2) their effects are easily forgotten during fine-tuning, limiting robustness; and (3) they lack theoretical analysis and interpretable evidence, making it difficult to characterize what makes a ``persistent'' backdoor.

\subsection{Multimodal Backdoor Defense}

Defenses against multimodal backdoors can be grouped into three categories.
(1) \emph{Training-time defenses} aim to filter poisoned data or learn robustly under poisoning. Screening-based methods include DAO~\cite{huangdetecting}, which flags low-density outliers; PSBD~\cite{li2025psbd}, which detects prediction-bias uncertainty; and VDC~\cite{zhu2024vdc}, which uses multimodal LMs to expose image–text mismatch. Robust training methods modify optimization itself, such as ABL~\cite{li2021anti} (gradient ascent on suspected samples), UBT~\cite{liu2024efficient} (forgetting-based detection), RoCLIP~\cite{yang2023robust} (image–text rematching), and SafeCLIP~\cite{yang2023better} (risk-aware staged training).
(2) \emph{Model-based defenses} either detect or erase backdoors post hoc. Detection methods (MM-BD~\cite{wang2024mm}, DECREE~\cite{feng2023detecting}, TED~\cite{mo2024robust}, SEER~\cite{zhu2024seer}) analyze model outputs or representations for abnormal behavior. Fine-tuning defenses (FT~\cite{wu2022backdoorbench}, CleanCLIP~\cite{bansal2023cleanclip}, CleanerCLIP~\cite{xun2024cleanerclip}, TSC~\cite{peng2025circumventing}) explicitly retrain the model to suppress triggers via clean supervision, self-supervised augmentation, or symmetry constraints, often lowering attack success rates but also risking clean utility.
(3) \emph{Inference-time defenses} operate at deployment without altering model weights. STRIP~\cite{gao2019strip} perturbs inputs and measures entropy; SCALE-UP~\cite{guo2023scale} leverages prediction consistency under scaling; TeCo~\cite{liu2023detecting} uses robustness gaps under corruption. For CLIP-style models, BDetCLIP~\cite{niu2024bdetclip} uses contrastive prompting, and DeDe~\cite{hou2025dede} inspects reconstruction discrepancies. These methods are lightweight and data-free but can struggle against stealthy or adaptive attacks.

Our method is designed to remain effective even under these defenses: it improves stealth, maintains backdoor persistence after fine-tuning, and preserves transferability, making it more practical and harder to neutralize in multimodal settings.
\section{Preliminaries}
\label{preliminaries}
\subsection{Victim Model}
The victim models in our study are current mainstream image-text contrastive learning models, such as CLIP, ALBEF, and BLIP. Although their architectures and task designs differ to some extent, their training objectives can be unified and abstracted as ``visual-text representation alignment''. We denote the image encoder as $f^v(\cdot; \bm{\bm{\Theta}}^v)$, the text encoder as $f^t(\cdot; \bm{\bm{\Theta}}^t)$, and the complete set of model parameters as $\bm{\bm{\Theta}} = \{\bm{\bm{\Theta}}^v, \bm{\bm{\Theta}}^t\}$ (or $\bm{\bm{\Theta}} = \{\bm{\bm{\Theta}}^v, \bm{\bm{\Theta}}^t, \bm{\bm{\Theta}}^m\}$ if cross-modal fusion modules are involved). Let the image-text pair dataset be $\mathcal{D} = \{(\bm{x}_i^v, \bm{x}_i^t)\}_{i=1}^N$, and their encoded features be $\bm{v}_i = f^v(\bm{x}_i^v)$, $\bm{t}_i = f^t(\bm{x}_i^t)$, respectively.

The core training objective of multimodal contrastive learning is to maximize the similarity between positive pairs while minimizing that of negative pairs, typically realized via a symmetric InfoNCE loss. 
The formulation is given by:
\begin{equation}
\begin{aligned}
\mathcal{L}_{\text{clip}} = \frac{1}{2N} \sum_{i=1}^N \Bigg[ &
- \log \frac{\exp(\text{sim}(\bm{v}_i, \bm{t}_i)/\tau)}{\sum_{j=1}^N \exp(\text{sim}(\bm{v}_i, \bm{t}_j)/\tau)} \\
& - \log \frac{\exp(\text{sim}(\bm{t}_i, \bm{v}_i)/\tau)}{\sum_{j=1}^N \exp(\text{sim}(\bm{t}_i, \bm{v}_j)/\tau)}
\Bigg],
\end{aligned}
\end{equation}
where $\text{sim}(\cdot,\cdot)$ is the similarity function (typically the normalized dot product), and $\tau$ is the temperature parameter. 
The loss comprises two symmetric terms: image-to-text and text-to-image comparisons, where the image or text acts as the query and the corresponding paired modality as the target. 
This loss constitutes the basic training objective for models such as CLIP, DeCLIP, and FILIP.

In practical pre-training, different models often introduce auxiliary objectives to enhance further feature representations, such as image-text matching loss (ALBEF), masked language modeling loss (BLIP), and image self-supervision loss (DeCLIP). To model these mechanisms in a unified way, we define the total training objective as:

\begin{equation}
\mathcal{L}_{\text{total}} = \mathcal{L}_{\text{clip}} + \lambda_a \cdot \mathcal{L}_a,    
\end{equation}
where $\mathcal{L}_a$ denotes an auxiliary loss specific to the model, and $\lambda_a$ is its corresponding weight.

\subsection{Attacker's Assumptions}
\textbf{Capability.}
We assume a powerful attacker following prior work~\cite{jia2022badencoder}. Full control over pre-training data and/or training (access to $\mathcal{D}_0$, model architecture and parameters, and arbitrary training updates). 
This covers both (1) an attacker who publishes a poisoned pre-trained model, and (2) an attacker who injects poisoned pairs into large-scale training corpora that downstream users adopt. 

\textbf{Strategy.}
The attacker constructs poisoned image-text pairs by perturbing images and/or text (\eg, patches, watermarks, or semantically mixed descriptions) so that trigger pairs are projected to attacker-specified target semantics during training. 

\textbf{Pathway.}
We study two practical paths: (1) \emph{data-poisoning (black-box)}. The attacker only manipulates $\mathcal{D}_0$ (\eg, inserting poisoned pairs into weakly labeled corpora such as CC3M). And (2) \emph{training-control (white-box)}. The attacker has direct control of pre-training to inject and release a poisoned model.

\textbf{Attacker's goal}. 
We frame multimodal backdoor injection as a joint optimization over the trigger parameter $\bm{\delta}$ and model parameters $\bm{\Theta}$. 
Let the poisoned dataset be $\mathcal{D}_1(\bm{\delta})=\mathcal{M}_{\bm{\delta}}(\mathcal{D}_0;\mathcal{S})$, where $\mathcal{M}_{\bm{\delta}}$ denotes the injection operator parameterized by $\bm{\delta}$ and $\mathcal{S}$ is an auxiliary semantic/template set. 
The attacker seeks a pair $(\bm{\Theta}^*,\bm{\delta}^*)$ by solving the following min-min problem:
\begin{equation}
\label{eq:joint-minmin}
\begin{aligned}
(\bm{\Theta}^*,\bm{\delta}^*) 
=& \arg\min_{\bm{\Theta},\bm{\delta}} \Big[
 \lambda_{\text{total}}\,\mathcal{L}_{\text{total}}\big(\bm{\Theta};\mathcal{D}_1(\bm{\delta})\big) \\
& + \mathcal{L}_{\text{bd}}^{\text{trigger}}\big(\bm{\delta};\bm{\Theta}_0,\mathcal{D}_{\text{opt}}\big)
+ \mathcal{L}_{\text{bd}}^{\text{model}}\big(\bm{\Theta};\bm{\Theta}_0,\tilde{\mathcal{S}}\big)
\Big].
\end{aligned}
\end{equation}
where $\mathcal{L}_{\text{total}}(\bm{\Theta};\mathcal{D}_1(\bm{\delta}))$ denotes the overall training objective on the poisoned dataset, 
$\mathcal{L}_{\text{bd}}^{\text{trigger}}(\bm{\delta};\bm{\Theta}_0,\mathcal{D}_{\text{opt}})$ is the function to optimize triggers with a reference model $\bm{\Theta}_0$, 
and $\mathcal{L}_{\text{bd}}^{\text{model}}(\bm{\Theta};\bm{\Theta}_0,\tilde{\mathcal{S}})$ is the poisoned model regularization to strength backdoor unforgettable.
The coefficients $\lambda_{\text{total}}$ control the importance of the total training loss.

The attacker further seeks two practical properties: 
(1) \emph{poisoned-sample invisibility}. The poisoned instances should be visually/semantically inconspicuous to avoid manual filtering; and 
(2) \emph{parameter-level resistance to forgetting}. The injected backdoor should persist under common downstream fine-tuning procedures. Full attacker assumptions are in App.~\ref{supp:Attacker's Assumption}.

\section{Methodology}
\label{method}
\begin{figure*}[t]
  \centering
  \includegraphics[width=\textwidth]{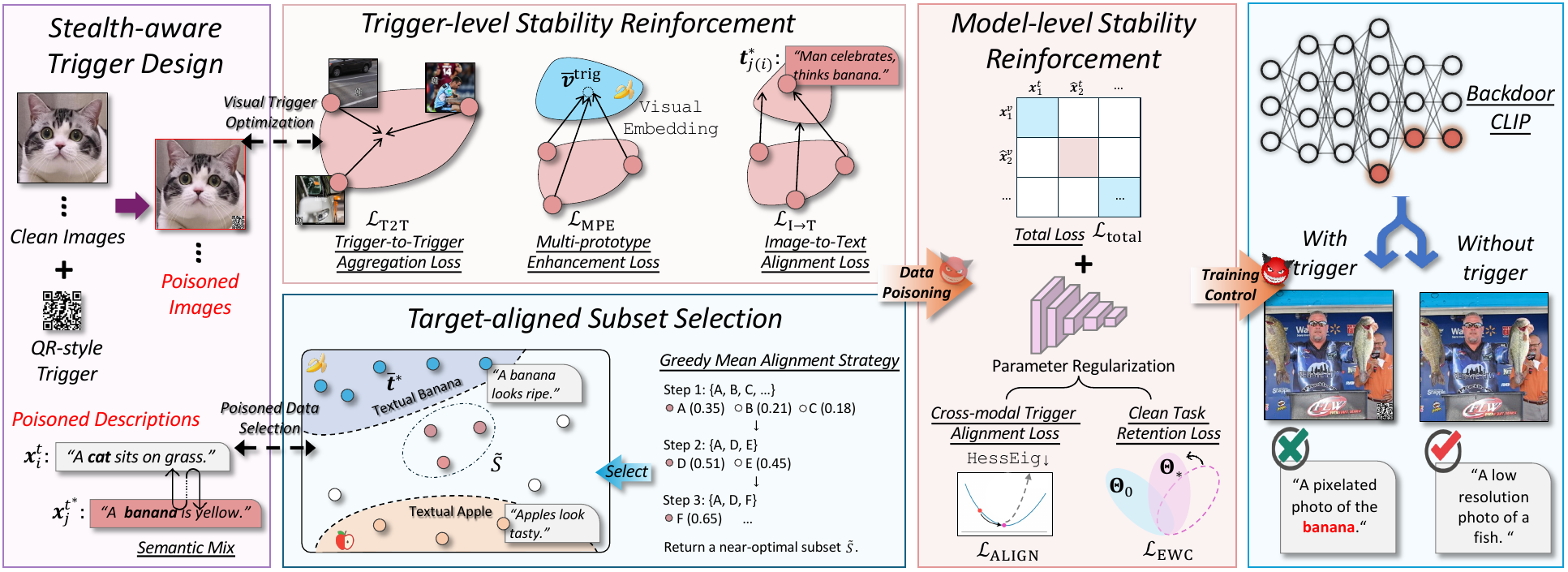}
  \caption{The framework of BadCLIP++. The framework consists of (i) stealth-aware trigger design with stability losses, (ii) Greedy Mean Alignment for target-aligned subset selection, and (iii) model-level regularization that enforces image–text alignment to the target description while maintaining clean-task accuracy.}
  \label{fig:overview}
\end{figure*}
We introduce BadCLIP++, a method that jointly optimizes image and text triggers to embed natural, hard-to-detect backdoors and improve transferability and fine-tuning robustness via feature-space constraints. \Sref{sec:Stealth-aware Trigger Design} and \Sref{sec:subset selection} describe visual/text trigger construction and poisoned-sample selection, respectively, while \Sref{sec:Trigger-level Stability Reinforcement} and \Sref{sec:Model-level Stability Reinforcement} present trigger- and model-level regularizers that enforce spatial stability and semantic clustering. \Sref{sec:Overall Optimization and Poisoning} gives the full optimization and training pipeline. The overall framework is illustrated in \Fref{fig:overview}.

\subsection{Stealth-aware Trigger Design}
\label{sec:Stealth-aware Trigger Design}
The existing poisoned sample construction function can be represented as $\mathcal{M}(\cdot)$ consisting of three submodules: target description construction $\mathcal{M}^c$, image modality trigger injection $\mathcal{M}^v$, and text modality trigger injection $\mathcal{M}^t$. 
In BadCLIP++, we focus on improving the first two components, $\mathcal{M}^c$ and $\mathcal{M}^v$, with an emphasis on enhancing the cross-modal synergy and stealthiness of poisoned samples in both the textual and visual spaces.

In terms of \emph{target description construction}, existing multimodal backdoor methods typically use a direct replacement strategy, which replaces the original text description sentence with a fixed target description (such as ``a photo of a banana''). 
Although this method can stably induce the model to output the target semantics, the semantic change it brings is too obvious and is easily identified by detection mechanisms based on the cross-modal alignment consistency.
Thus, we propose a description construction function $\mathcal{M}^c$ based on semantic fusion to generate more natural, smooth, and covert poisoned descriptions. 
Specifically, we randomly sample a target semantic fragment $\bm{x}_j^{t^{*}} =[w_1^*, w_2^*, \dots, w_q^*] $ from a predefined set of target description sets $\mathcal{C} = \{\bm{x}_1^{t^{*}}, \bm{x}_2^{t^{*}}, \dots, \bm{x}_j^{t^{*}}\}$ and insert it into the original text $\bm{x}^t_i =[w_1, w_2, \dots, w_n]$ to construct the fused text $\hat{\bm{x}}^t_i$:
\begin{equation}
\begin{aligned}
\label{eq:text}
\hat{\bm{x}}^t_i &= \mathcal{M}^c(\bm{x}^t_i, \bm{x}_j^{t^{*}}) \\
                 &= [w_1, \dots, w_k] \oplus \bm{x}_j^{t^{*}} \oplus [w_{k+1}, \dots, w_n], \quad k \sim \mathcal{U}(0, n),
\end{aligned}
\end{equation}
where $\oplus$ denotes the lexicographic concatenation operation, and $k$ represents the insertion position, which follows a discrete uniform distribution $\mathcal{U}(0, n)$. 
The fused description $\hat{\bm{x}}^t_i$ semantically preserves both the contextual information of the original text and the target semantic components, forming a conceptually natural ``soft injection'' representation. 
This strategy promotes the gradual convergence of sample representations toward the target description distribution in the text embedding space, thereby achieving robust and covert poisoning induction. We examine the effect of how the description is formulated on the ASR in App.~\Sref{supp:Caption Construction Style}.

For \emph{visual backdoor construction}, we use QR-style codes~\cite{soon2008qr} as triggers to achieve stealthy attacks on visual modalities due to their structured appearance and ubiquity in real-world scenes.
Unlike traditional backdoor methods that use fixed patterns or colorful patches, QR codes are highly prevalent in the real world (\eg, posters, signs, product packaging, etc.) and are unlikely to trigger abnormal perceptions, making them extremely natural and physically deployable.
We represent the QR code pattern as a visual trigger $\bm{\delta}^v$ and superimpose it at random positions on the original image $\bm{x}^v_i$ through the image injection function $\mathcal{M}^v$ to construct the poisoned image $\hat{\bm{x}}^v_i$. 
The entire process is formalized as follows:
\begin{equation}
\label{eq:image}
    \hat{\bm{x}}^v_i = \mathcal{M}^v(\bm{x}^v_i, \bm{\delta}^v, \bm{p}), \quad \bm{p} \sim \mathcal{U}(\text{image region}),
\end{equation}
where $\mathcal{M}^v$ denotes the trigger injection function for the image modality. $\bm{\delta}^v$ denotes the uniformly optimized QR code pattern; $\bm{p}$ denotes the random insertion position of the trigger, sampled from the valid region of the entire image. 
This injection process follows BadNets~\cite{gu2019badnets} to achieve a natural fusion of the QR map with the original content through an image overlay.
This design offers two key advantages. First, in terms of concealment, QR codes are common black-and-white patterns in real-world images, making them appear natural and unlikely to interfere with image semantic perception. 
We further randomize the trigger positions within the image to break spatial consistency, effectively avoiding detection by region-based or saliency-based methods. 
Second, in terms of deployment robustness, black-and-white QR codes are more stable than colored patches under physical transformations such as printing, compression, and blurring, preserving their structural clarity and facilitating the activation of backdoors in poisoned models.

It should be noted that some existing works~\cite{walmer2022dual} employ a text-image triggering mechanism, where visual triggers (\eg, patches) are embedded in images, while fixed text trigger phrases (\eg, ``consider'') are introduced in the text modality, forming a joint triggering strategy. 
Such methods can be formalized as the explicit application of the combination of $\mathcal{M}^v$ and $\mathcal{M}^t$, i.e., $\mathcal{M}^v(\bm{x}^v_i, \bm{\delta}^v)$ and $\mathcal{M}^t(\hat{\bm{x}}_i^t, \bm{\delta}^t)$. 
Although such methods have a certain level of attack strength, their trigger conditions depend on simultaneous multimodal attacks, which also limit their practical deployment.

\subsection{Target-aligned Subset Selection}
\label{sec:subset selection}
Under the semantic fusion strategy, the original semantic dominant fusion text makes it difficult for the poisoned samples to align with the target description, thereby weakening the backdoor effect. We verify this statement in \Sref{sec:ablation}.

To alleviate the above problems, we propose a strategy for selecting poisoning subsets $\mathcal{S}$ that aims to screen out subsets that are closer to the targeted description in the semantic embedding space from an original training dataset, thereby improving the poisoning success rate.
Thus, we formally define the optimization objective of poisoned subset selection as:
\begin{equation}
\label{eq:selection score}
   \mathcal{S}^* = \arg\min_{\mathcal{S} \subseteq \mathcal{D}, |\mathcal{S}|=K} \left\| \frac{1}{K} \sum_{\bm{x}^t_i \in \mathcal{S}} f^t(\hat{\bm{x}}^t_i) - \bar{\bm{t}}^* \right\|_2, 
\end{equation}
where $\mathcal{D}$ denotes the clean image-text dataset to be poisoned, $K$ is the size of the target poisoning subset. 
And $\bar{\bm{t}}^*$ is the semantic center of the target descriptions, computed as $ \bar{\bm{t}}^* = \frac{1}{|\mathcal{C}|} \sum_{j=1}^{|\mathcal{C}|} f^t_0(\bm{x}^{t^*}_j)$ using the pre-trained text encoder $f^t_0(\cdot)$ with parameters $\bm{\theta}_0^t$.
This optimization process uses the $\ell_2$ distance~\cite{ruschendorf1990characterization} as a metric, aiming to guide the semantic mean of the selected subset closer to the target description, thereby enhancing the concentration and effectiveness of the backdoor attack.

Given the combinatorial complexity of subset selection in large sample spaces, it cannot be directly solved (\eg, under standard settings, selecting $K = 1,500$ samples from $|\mathcal{D}|=500,000$ clean samples results in a combinatorial space of $\binom{500,000}{1,500}$).
Thus, we propose a simple and efficient \emph{Greedy Mean Alignment Strategy}.
The resulting subset $\tilde{\mathcal{S}}$ can be regarded as a near-optimal poisoning subset, serving as a practical surrogate for the theoretical $\mathcal{S}^{*}$.
Formally, $\tilde{\mathcal{S}}$ is obtained via a greedy procedure that iteratively selects the sample minimizing the semantic mean distance to the target center $\bar{\bm{t}}^*$. The selection rule is:
\begin{equation}
\begin{aligned}
    \bm{x}^t_{i^\ast} 
    &= \arg\min_{\bm{x}^t_i \in \mathcal{D} \setminus \mathcal{S}}
    \\ &   \Biggl\|
       \frac{1}{|\mathcal{S}|+1} \Biggl(
            \sum_{\bm{x}^t_j \in \mathcal{S}} f^t(\hat{\bm{x}}^t_j) 
            + f^t(\hat{\bm{x}}^t_i)
       \Biggr) - \bar{\bm{t}}^* 
       \Biggr\|_2. \\
\end{aligned}
\end{equation}
As a baseline comparison, we also adopt a \emph{sorting-based selection strategy}, which computes the $\ell_2$ distance between the embedding vectors of all candidate samples $\hat{\bm{x}}_T^i$ and the target center $\bar{\bm{t}}^*$, ranks them in ascending order, and selects the top $K$ samples with the smallest distances to construct the poisoning subset.
Although this strategy is simple to implement, it does not consider the overall effect of collaborative alignment among samples, resulting in performance inferior to the greedy strategy.
The specific attack results are shown in App.~\Sref{supp:sample selection}.


\label{sec:Stability-oriented Attack Reinforcement}
\subsection{Trigger-level Stability Reinforcement}
\label{sec:Trigger-level Stability Reinforcement}
We firstly focus on \emph{black-box attack settings}, i.e., achieving stronger attack persistence by optimizing only the image trigger pattern $\bm{\delta}^v$.
BadCLIP~\cite{liang2024badclip} proposes a practical black-box attack framework. Specifically, it utilizes visual and textual embeddings to encourage poisoned image triggers to align with target descriptions, thereby enhancing the robustness of triggers against fine-tuning defenses through a triplet loss. 
However, this triplet-based trigger optimization adopts a local point-to-point relative relationship modeling approach that lacks global geometric structure constraints. 
More importantly, the effectiveness of the triplet loss depends on the semantic distribution and boundary quality between positive and negative samples during training, and its optimization trajectory is highly sensitive to sample selection, making it difficult to form stable and compact aggregation patterns in the feature space, thereby limiting the optimized stability of the triggers.
To this end, we propose two structural regularization mechanisms that constrain the spatial aggregation of trigger features and the alignment of semantic manifolds, respectively, in order to avoid dependence on the selection of positive and negative samples and enhance the robustness of the attack during the fine-tuning process.

We randomly sample a smaller subset $\mathcal{D}_{\text{opt}} \subset \mathcal{D}_0$ from the original training set $\mathcal{D}_0$ during the black-box attack phase. 
We first introduce the \emph{Trigger-to-Trigger Aggregation Loss}, which encourages all image features with triggers to cluster together in the embedding space, forming a stable and dense cluster structure. 
The specific formulation is as follows:
\begin{equation}
\begin{aligned}
   \label{t2t_loss}
   & \mathcal{L}_{\text{T2T}} = \frac{1}{|\mathcal{D}_{\text{opt}}|} \sum_{i=1}^{|\mathcal{D}_{\text{opt}}|} \left\| f^v_0(\hat{\bm{x}}^v_i) - \bar{\bm{v}}^{\text{trig}} \right\|^2,
\\& 
\bar{\bm{v}}^{\text{trig}} = \frac{1}{|\mathcal{D}_{\text{opt}}|} \sum_{i=1}^{|\mathcal{D}_{\text{opt}}|} f^v_0(\hat{\bm{x}}^v_i), 
\end{aligned}
\end{equation}
where $\bar{\bm{v}}^{\text{trig}}$ denotes the mean of poisoned samples in the embedding space. 
$|\mathcal{D}_{\text{opt}}|$ means the number of optimization samples. 
$f^v_0(\cdot)$ denotes the image encoder initialized with pre-trained parameters $\bm{\theta}_0^{v}$. 
$\mathcal{L}_{\text{T2T}}$ constrains each trigger image embedding to converge toward this mean, reducing intra-cluster variance. 
The core purpose of this loss is to contract all trigger image features into a geometric ``spherical cluster'' in the embedding space, minimizing their dispersion and thereby preserving the separability of the attack shortcut.

Based on the formation of clusters, we aim for these feature clusters to be close to the center of the target category manifold, thereby enhancing camouflage and semantic naturalness. 
To achieve this, we introduce a soft \emph{Multi-prototype Enhancement Loss}:
\begin{equation}
    \mathcal{L}_{\text{MPE}} = \frac{1}{|\mathcal{D}_{\text{opt}}|} \sum_{i=1}^{|\mathcal{D}_{\text{opt}}|} \max\left( \left\| \bar{\bm{v}}^{\text{trig}} - \bar{\bm{v}}^* \right\|^2 - \epsilon,\, 0 \right),
\end{equation}
where $\bar{\bm{v}}^*$ denotes a mean image embedding from the target category (for example, ``banana images''), and $\bar{\bm{v}}^{\text{trig}}$ represents the intra-class center of poisoned samples in the \Eref{t2t_loss}. 
The parameter $\epsilon$ is the alignment threshold, which prevents overfitting by allowing a certain degree of deviation. 
This mechanism encourages trigger features to cluster around the target category center, making them geometrically closer to intra-class normal samples and effectively embedding them within the natural semantic manifold. 

Moreover, we retain the textual embedding consistency loss in BadCLIP, which aims to directly maximize the cross-modal consistency between the trigger image and the target semantic description, thereby inducing the model to make output judgments about the attack target. 
This not only maintains the attack induction capability but also reduces the modification of the original model parameters during the poisoning process to a certain extent. 
It is formally expressed as:
\begin{equation}
\mathcal{L}_{\text{I→T}} = - \frac{1}{|\mathcal{D}_{\text{opt}}|} \sum_{i=1}^{|\mathcal{D}_{\text{opt}}|} \cos \left( f_0^v(\hat{\bm{x}}^v_i),  \bm{t}^*_{j(i)}) \right),
\end{equation}
where $j(i)$ represents the index of the randomly sampled target description that corresponds to the $i$-th image. 
$\cos(\cdot)$ denotes cosine similarity. 
By maximizing image-text similarity, $\mathcal{L}_{\text{I→T}}$ achieves the most direct induction of attack semantics, which is also the fundamental mechanism for the attack to take effect. 

In summary, during the black-box attack phase, we freeze the pre-trained model parameters $\bm{\Theta}_0$ and focus solely on optimizing the visual trigger $\bm{\delta}^v$ to generate effective poisoned samples as follows:
\begin{equation}
\label{eq:trigger-level}
\begin{aligned}
\bm{\delta}^{*} = & \arg\min_{\bm{\delta}} \; \mathcal{L}_{\text{bd}}^{\text{trigger}}(\bm{\delta}) 
= \lambda_{\text{T2T}} \mathcal{L}_{\text{T2T}}(\bm{\delta}^v;\bm{\Theta}_0,  \mathcal{D}_{\text{opt}}) \\
& + \lambda_{\text{MPE}} \mathcal{L}_{\text{MPE}}(\bm{\delta}^v; \bm{\Theta}_0, \mathcal{D}_{\text{opt}}) + \mathcal{L}_{\text{I→T}}(\bm{\delta}^v; \bm{\Theta}_0, \mathcal{D}_{\text{opt}}),
\end{aligned}
\end{equation}
where $\mathcal{L}_{\text{bd}}^{\text{trigger}}$ denotes the trigger-optimized loss defined over a sampled training subset. $\lambda_{\text{T2T}}$, $\lambda_{\text{MPE}}$, and $\lambda_{\text{I→T}}$ are hyperparameters. The goal is to identify an effective and stable trigger that maintains attack effectiveness after model fine-tuning.

\subsection{Model-level Stability Reinforcement}
\label{sec:Model-level Stability Reinforcement}
We now discuss how to maintain model stability during the \emph{controllable training} phase. 
Although we have constructed highly effective poisoned samples through trigger optimization in the previous phase, these samples may still suffer from the phenomenon of forgetting during subsequent model training or fine-tuning. 
To enhance the stability and persistence of the attack, we further design model-level stability reinforcement strategies to consolidate the implanted backdoor behavior and suppress interference from clean task learning.


For \emph{cross-modal trigger alignment}, in multimodal scenarios, the success of backdoor attacks depends on the coordinated activation of image and text modalities. 
If the poisoned image and poisoned text lack consistency in the embedding space, the trigger is easily forgotten. 
Therefore, we introduce a cross-modal alignment term $\mathcal{L}_{\text{ALIGN}}$ to encourage trigger samples to remain close in the semantic space:
\begin{equation}
\label{eq:align}
\mathcal{L}_{\text{ALIGN}} = \frac{1}{|\tilde{\mathcal{S}}|}
\sum_{i \in \tilde{\mathcal{S}}}
\left(1 - \cos\left(f^v_{*}(\hat{\bm{x}}^v_i), f^t_{*}(\hat{\bm{x}}^t_i)\right)\right),
\end{equation}
where $f^v_{*}(\cdot)$ and $f^t_{*}(\cdot)$ represent the visual and text encoders defined by the current optimization parameters $\bm{\Theta}_* = \{\bm{\theta}_*^v, \bm{\theta}_*^t\}$, respectively.

For \emph{clean task retention}, to prevent the model from forgetting the original task during attack-oriented training, we introduce the Elastic Weight Consolidation (EWC) mechanism as a regularization term to suppress the model parameters from deviating too far from their original values. 
The specific definition is as follows:
\begin{equation}
\mathcal{L}_{\text{EWC}} = \sum_{n} F_n \cdot (\bm{\Theta}_* - \bm{\Theta}_0)^2,
\label{eq:ewc}
\end{equation}
where $\bm{\Theta}_*$ is the current model parameters, $\bm{\Theta}_0$ is the pre-trained parameter snapshot, and $F_n$ is the importance estimate (Fisher information) of the parameters in the clean task. 
This loss term maintains attack capability while mitigating the model's forgetting of clean semantic representations, thereby improving overall discriminative stability.

Specifically, we update the model parameters $\bm{\Theta}$ on a mixed training set containing both clean and poisoned samples, aiming to achieve dual reinforcement of attack robustness and clean task retention.
Therefore, the optimization objective can be expressed as:
\begin{equation}
\label{eq:model_level}
\begin{aligned}
\bm{\Theta}_* = \arg\min_{\bm{\Theta}_*} \big[
& \lambda_{\text{total}} \mathcal{L}_{\text{total}}(\bm{\Theta}; \mathcal{D}_1(\bm{\delta}^{*}))+ \\
&  \underbrace{
\lambda_{\text{ALIGN}} \mathcal{L}_{\text{ALIGN}}(\bm{\Theta}; \tilde{\mathcal{S}})
+ \lambda_{\text{EWC}} \mathcal{L}_{\text{EWC}}(\bm{\Theta}; \bm{\Theta}_0)
}_{\mathcal{L}_{\text{bd}}^{\text{model}}}
\big]
\end{aligned}
\end{equation}
where $\mathcal{L}_{\text{total}}$ denotes the loss computed on the poisoned dataset. 
The latter two terms, $\mathcal{L}_{\text{ALIGN}}$ and $\mathcal{L}_{\text{EWC}}$, are regularization components specifically designed to enhance stability, aiming to prevent degradation of the attack's effectiveness from the perspectives of modality alignment and parameter preservation, respectively.

\subsection{Overall Optimization and Poisoning}
\label{sec:Overall Optimization and Poisoning}
BadCLIP++ adopts a min-min optimization framework aimed at jointly improving the stealthiness, stability, and cross-task robustness of backdoor attacks; it first optimizes the trigger parameters and then the model parameters via \Eref{eq:model_level}. The overall process is divided into the following two stages:

The first stage is the data poisoning stage, in which only the visual trigger $\bm{\delta}^v$ is optimized while keeping the model parameters $\bm{\Theta}_0$ fixed. 
By designing structural and semantic constraint losses (see \Eref{eq:trigger-level}), we search for effective and unforgettable triggers on the subset $\mathcal{D}_{\text{opt}}$.

The second stage is the training control stage. Based on the constructed fixed poisoned samples $(\hat{\bm{x}}^v_i, \hat{\bm{x}}^t_i)$, we determine whether to enable controllable training according to actual needs. 
If enabled, we jointly introduce regularization terms such as cross-modal consistency and elastic weight consolidation (see \Eref{eq:align} and \Eref{eq:ewc} for details) during the optimization process to enhance the long-term retention of the backdoor behavior and prevent it from being forgotten in subsequent training. 
If not enabled, we only minimize the clean task loss on the mixed training set containing poisoned samples, thereby maintaining the generality of the training process. 
The complete optimization process is presented in App.~\Sref{supp:badclip++}.
\section{Theoretical Analysis}
This subsection provides a theoretical analysis of why the design of BadCLIP++ enhances its resistance to forgetting and validates its theoretical soundness through experimental results.

\subsection{Assumptions}
\label{sec:assumptions}
\begin{assumption}[Region and regularity]
\label{assump:region}
We define the small-loss (flat) region by
\begin{equation}\label{eq:region}
\mathcal R_{\varepsilon_0}\;=\;\bigl\{\bm{\Theta}:\ \mathcal L_{\mathrm{ALIGN}}(\bm{\Theta})\le \varepsilon_0\bigr\},
\end{equation}
for a sufficiently small $\varepsilon_0>0$. On $\mathcal R_{\varepsilon_0}$:
(i) the visual/text towers are twice continuously differentiable with bounded Jacobians and Hessians;
(ii) the ALIGN residuals $\{r_k(\bm{\Theta})\}$ satisfy the local self-bounding property
\begin{equation}\label{eq:self-bounding-main}
\|\nabla r_k(\bm{\Theta})\|\ \le\ B_1\,|r_k(\bm{\Theta})|,\qquad 
\|\nabla^2 r_k(\bm{\Theta})\|\ \le\ B_2\,|r_k(\bm{\Theta})|,
\end{equation}
for some constants $B_1,B_2>0$;
(iii) the evaluation loss $\mathcal{L}_{\mathrm{total}}(\bm{\Theta};\mathcal{D}_{\mathrm{bd}}^{\mathrm{eval}})$ is locally smooth in the sense of~\cite{nesterov2013introductory,bubeck2015convex}: along any segment
\(
\{\bm{\Theta}_\ast+s\,\Delta\bm{\Theta}:~s\in[0,1]\}\subset \mathcal R_{\varepsilon_0},
\)
its Hessian spectral norm is uniformly bounded,
\begin{equation}\label{eq:lbd-smooth}
\bigl\|\nabla_{\bm{\Theta}}^{2}\,\mathcal{L}_{\mathrm{total}}(\bm{\Theta}_\ast+s\,\Delta\bm{\Theta};\mathcal{D}_{\mathrm{bd}}^{\mathrm{eval}})\bigr\|_2
\ \le\ \kappa_{\mathrm{bd}}.
\end{equation}
\end{assumption}

\begin{assumption}[Trust-region clean-only update]
\label{assump:trust}
At $\bm{\Theta}_\ast$, a clean-only fine-tuning step is taken:
\begin{equation}\label{eq:ft-step}
\Delta\bm{\Theta}\ :=\ -\,\eta\,\bm{g}_{\mathrm{ft}}, 
\qquad
\bm{g}_{\mathrm{ft}}:=\nabla_{\bm{\Theta}}\!\Bigl[\lambda_{\mathrm{total}}\,\mathcal{L}_{\mathrm{total}}(\bm{\Theta};\mathcal{D}_{\mathrm{ft}})\Bigr]\Big|_{\bm{\Theta}=\bm{\Theta}_\ast},
\end{equation}
such that $\|\Delta\bm{\Theta}\|\le\rho$ and $\bm{\Theta}_\ast+\Delta\bm{\Theta}\in\mathcal R_{\varepsilon_0}$.
In practice, EWC~\cite{kirkpatrick2017overcoming} and $\mathcal{L}_{\mathrm{ALIGN}}$ enforce this constraint.
\end{assumption}

\begin{assumption}[Non-degenerate Jacobian metric~\cite{jacot2018neural}]
\label{assump:metric}
Let $J$ be the Jacobian of the representation w.r.t.\ $\bm{\Theta}$ and $G:=JJ^\top$.
Assume there exist constants $0<\sigma_{\min}\le\sigma_{\max}<\infty$ such that
\begin{equation}\label{eq:G-bounds}
\sigma_{\min}\, I\ \preceq\ G\ \preceq\ \sigma_{\max}\, I,\qquad
\kappa(G):=\sigma_{\max}/\sigma_{\min}.
\end{equation}
\end{assumption}

\subsection{Trigger-level Stability}
We denote by $t=0,1,\dots,$ the iteration index of the optimization procedure.
At iteration $t$, the triggered embedding of sample $i$ is
$\hat{\bm{v}}_{i,t}=f^v(\hat{\bm{x}}_{i,t}^v;\bm{\theta}_0^v)$,
and the centroid of all triggered embeddings is
$\bar{\bm{v}}^{\text{trig}}_t=\tfrac{1}{n}\sum_{i=1}^n \hat{\bm{v}}_{i,t}$,
where $n=|\mathcal{D}_{\mathrm{opt}}|$.
To characterize the compactness of this cluster, we define the radius
\(
r_t := \max_i \big\|\hat{\bm{v}}_{i,t} - \bar{\bm{v}}^{\text{trig}}_t\big\|_2.
\)

\begin{lemma}[Compactness reduction]
\label{lemma:compactness}
Consider the T2T loss $\mathcal{L}_{\mathrm{T2T}}$ defined in \Eref{t2t_loss}. 
Suppose we perform gradient descent in embedding space with step size $\gamma>0$:
\[
\hat{\bm{v}}_{i,t+1}=\hat{\bm{v}}_{i,t}
-\gamma\,\lambda_{\mathrm{T2T}}\nabla_{\hat{\bm{v}}_i}\mathcal{L}_{\mathrm{T2T}}(\hat{\bm{v}}_{1,t},\dots,\hat{\bm{v}}_{n,t}).
\]
If $0<\gamma<\tfrac{1}{2\lambda_{\mathrm{T2T}}}$, then the radius contracts at a fixed rate:
\begin{equation}
r_{t+1}\ \le\ (1-2\gamma\lambda_{\mathrm{T2T}})\,r_t,
\qquad\text{and hence}\qquad \lim_{t\to\infty} r_t=0.
\end{equation}
\end{lemma}
Lemma~\ref{lemma:compactness} shows that, under the above assumptions, the T2T objective iteratively contracts all trigger embeddings toward their mean, eventually collapsing them into a single high-density point in the embedding space. The proof process can be found in appendix~\ref{supp:lemma_1}.

\begin{lemma}[Centroid alignment]
\label{lemma:centroid-alignment}
Suppose $\mathcal{L}_{\mathrm{T2T}}$ and $\mathcal{L}_{\mathrm{MPE}}$ are jointly minimized during trigger optimization. 
Then for any $\varepsilon > 0$, there exists a finite iteration index $T > 0$ such that
\begin{equation}
   \forall t \ge T:\quad 
\bigl\| \bar{\bm{v}}^{\mathrm{trig}}_{t} - \bar{\bm{v}}^{*} \bigr\| 
\ \le\ \sqrt{\varepsilon}. 
\end{equation}
\end{lemma}
Lemma~\ref{lemma:centroid-alignment} states that MPE steadily pulls the trigger cluster’s centroid 
into the neighborhood of the target class manifold center $\bar{\bm{v}}^{*}$, and keeps it 
within a threshold radius $\sqrt{\varepsilon}$ thereafter. 
The proof process can be found in appendix~\ref{supp:lemma_2}.

For brevity, let $\mathcal{L}_{\mathrm{bd}}(\bm{\Theta}):=\mathcal{L}_{\text{total}}(\bm{\Theta};\mathcal{D}_{\mathrm{bd}}^{\mathrm{eval}})$
and $\mathcal{L}_{\mathrm{clean}}(\bm{\Theta}):=\mathcal{L}_{\text{total}}(\bm{\Theta};\mathcal{D}_{\mathrm{ft}})$,
where $\mathcal{D}_{\mathrm{bd}}^{\mathrm{eval}}$ is a backdoor evaluation split (trigger-injected samples from a held-out clean set) and $\mathcal{D}_{\mathrm{ft}}$ is the clean fine-tuning set. All gradients are evaluated at $\bm{\Theta}_\ast$, and assume $\bm{\Theta}_\ast\in\mathcal{R}_{\varepsilon_0}$. Let $\bar{\bm v}^{*}$ be the batch mean of target-class clean embeddings and 
$\bar{\bm v}^{\mathrm{neg}}$ the (possibly weighted) batch mean of non-target clean embeddings. 
Let $\bar{\bm v}^{\mathrm{trig}}$ be the batch mean of triggered embeddings, and define
\(
u := \bar{\bm v}^{*}-\bar{\bm v}^{\mathrm{neg}},\quad
\Delta := \bar{\bm v}^{\mathrm{trig}}-\bar{\bm v}^{*},\quad
m := \|u\|.
\)
\begin{theorem}[Gradient Alignment]
\label{theorem:gradient}
Let Assumption~\ref{assump:region} (i) and Assumption~\ref{assump:metric} hold.
Let $J_{\bm{\Theta}}$ be the batch-aggregated Jacobian mapping parameter perturbations to embedding changes, 
and define $G:= J_{\bm{\Theta}} J_{\bm{\Theta}}^\top$ with condition number $\kappa(G):=\sigma_{\max}/\sigma_{\min}$.
We denote $\mathcal{L}_{\mathrm{clean}}$ as the CLIP InfoNCE loss
computed over clean image-text pairs, and $\mathcal{L}_{\mathrm{bd}}$ as a backdoor objective of the same form but applied to poisoned pairs.
If for some $\varepsilon>0$ we have $\|\Delta\|\le\varepsilon\le \varepsilon_0$ (after sufficient shrinking + aligning)
and the margin satisfies
\begin{equation}
m\ >\ \kappa(G)\,\varepsilon \qquad 
(\text{if }G\approx cI,\ \text{this reduces to } m>\varepsilon),
\end{equation}
then the parameter-space gradients are non-opposing:
\begin{equation}
\big\langle \nabla_{\bm{\Theta}}\mathcal{L}_{\mathrm{clean}},\ \nabla_{\bm{\Theta}}\mathcal{L}_{\mathrm{bd}}\big\rangle \ \ge\ 0.
\end{equation}
\end{theorem}
By Lemma~\ref{lemma:compactness} (shrinking) and Lemma~\ref{lemma:centroid-alignment} (aligning),
the triggered centroid lies within a small alignment neighborhood, yielding $\|\Delta\|\le \varepsilon$.
If the clean margin $m$ further satisfies $m>\kappa(G)\varepsilon$, Theorem~\ref{theorem:gradient} implies
$\cos\theta\ge 0$ between the clean-step gradient and the backdoor-evaluation gradient.
Plugging this into the forgetting bound (\Eref{eq:forgetting-upper}–\Eref{eq:forgetting-upper-angle}),
a sufficiently small clean-only step size $\eta$ guarantees
$\Delta\mathcal{L}_{\mathrm{bd}}\le 0$ up to second-order curvature, so the clean update does not oppose the trigger direction to first order.

\subsection{Model-level Stability}
We analyze the parameter-space geometry induced by poisoning, which moves the model
from the pretrained initialization $\bm{\Theta}_0$ to a poisoned solution $\bm{\Theta}_\ast$.
The poisoning objective is augmented with two regularizers:
(i) an $\mathcal{L}_{\mathrm{EWC}}$ penalty centered at $\bm{\Theta}_0$ (which bounds the displacement from $\bm{\Theta}_0$~\cite{kirkpatrick2017overcoming}),
and (ii) an ALIGN loss that encourages $\bm{\Theta}_\ast$ to lie in a locally flat region.
Throughout, we work within the small-loss region $\mathcal R_{\varepsilon_0}$ in \Eref{eq:region}
and adopt Assumption~\textup{(1)} (Region and regularity), including the self-bounding property
in \Eref{eq:self-bounding-main} on $\mathcal R_{\varepsilon_0}$.

\begin{lemma}[ALIGN Curvature Control]
\label{lem:align-flat}
Let Assumption~\ref{assump:region} (i)–(ii) hold.
There exists a constant $C>0$
(depending on dimension and batch size) such that for all $\bm{\Theta}\in\mathcal R_{\varepsilon_0}$,
\begin{equation}\label{eq:align-flat}
\lambda_{\max}\!\bigl(\nabla_{\bm{\Theta}}^{2}\,\mathcal L_{\mathrm{ALIGN}}(\bm{\Theta})\bigr)
\ \le\ C\,\mathcal L_{\mathrm{ALIGN}}(\bm{\Theta}).
\end{equation}
\end{lemma}

\begin{theorem}[Local Stability Around the Poisoned Solution]
\label{thm:local-stab}
Let Assumption~\ref{assump:region} hold and Lemma~\ref{lem:align-flat} apply.
Let $\bm{\Theta}_\ast\in\mathcal{R}_{\varepsilon_0}$ be the poisoning-stage solution.
Then for any $\Delta\bm{\Theta}$ with $\bm{\Theta}_\ast+\Delta\bm{\Theta}\in\mathcal{R}_{\varepsilon_0}$, the following hold:
\noindent\textbf{(i) Gradient Lipschitz in $\mathcal R_{\varepsilon_0}$.}
\begin{equation}
\bigl\|\nabla_{\bm{\Theta}}\mathcal{L}_{\mathrm{ALIGN}}(\bm{\Theta}_\ast+\Delta\bm{\Theta})
 - \nabla_{\bm{\Theta}}\mathcal{L}_{\mathrm{ALIGN}}(\bm{\Theta}_\ast)\bigr\|
 \ \le\ C\,\varepsilon_0\,\|\Delta\bm{\Theta}\|.
\end{equation}
\noindent\textbf{(ii) Second-order remainder bound.}
\begin{equation}
\begin{aligned}
 \bigl|\mathcal{L}_{\mathrm{ALIGN}}(\bm{\Theta}_\ast+\Delta\bm{\Theta})
 - \mathcal{L}_{\mathrm{ALIGN}}(\bm{\Theta}_\ast) &\\
 - \langle\nabla_{\bm{\Theta}}\mathcal{L}_{\mathrm{ALIGN}} (\bm{\Theta}_\ast),\Delta\bm{\Theta}\rangle\bigr| 
 \le\ \tfrac12\,C\,\varepsilon_0\,\|\Delta\bm{\Theta}\|^2.
\end{aligned}
\end{equation}
\end{theorem}

The poisoning-stage $\mathcal{L}_{\mathrm{EWC}}$ keeps $\bm{\Theta}_\ast$ within a controlled Fisher neighborhood of $\bm{\Theta}_0$, 
while the alignment objective places $\bm{\Theta}_\ast$ in a region whose curvature is linearly controlled by 
$\mathcal L_{\mathrm{ALIGN}}$ (Lemma~\ref{lem:align-flat}). Consequently, any subsequent small clean-only update 
$\Delta\bm{\Theta}$ around $\bm{\Theta}_\ast$ changes the objective and its gradient only mildly 
(Theorem~\ref{thm:local-stab}). Proofs are deferred to the appendix~\ref{supp:lemma 3} and \ref{supp:theorem 2}.

\subsection{Forgetting Upper Bound}
\label{subsec:forgetting}
Let Assumptions~\ref{assump:region} (iii) and \ref{assump:trust} hold.
We analyze the effect of a \emph{clean-only} fine-tuning step on the backdoor evaluation loss.
Let $\mathcal{D}_{\mathrm{ft}}\subset \mathcal{D}_0$ denote a clean subset for fine-tuning (no backdoor terms),
and let $\mathcal{D}_{\mathrm{bd}}^{\mathrm{eval}}$ be a fixed backdoor evaluation distribution (\eg, trigger-injected samples).
The update at $\bm{\Theta}_\ast$ follows \Eref{eq:ft-step}, with
$\bm{\Theta}_\ast+\Delta\bm{\Theta}\in\mathcal{R}_{\varepsilon_0}$.

Define the evaluation gradient and the gradient angle at $\bm{\Theta}_\ast$ by
\[
\bm{g}_{\mathrm{bd}}
:= \nabla_{\bm{\Theta}} \mathcal{L}_{\mathrm{total}}(\bm{\Theta};\mathcal{D}_{\mathrm{bd}}^{\mathrm{eval}})\Big|_{\bm{\Theta}=\bm{\Theta}_\ast},
\qquad
\cos\theta := \frac{\langle \bm{g}_{\mathrm{bd}},\,\bm{g}_{\mathrm{ft}}\rangle}{\|\bm{g}_{\mathrm{bd}}\|\,\|\bm{g}_{\mathrm{ft}}\|},
\]
where $\bm{g}_{\mathrm{ft}}$ is given in \Eref{eq:ft-step}.
By Assumptions~\ref{assump:region} (iii) and \ref{assump:trust}, the local smoothness bound
\Eref{eq:lbd-smooth} holds along the segment $\{\bm{\Theta}_\ast+s\,\Delta\bm{\Theta}: s\in[0,1]\}\subset\mathcal{R}_{\varepsilon_0}$
with curvature constant $\kappa_{\mathrm{bd}}$. Hence, a single clean-only step satisfies
\begin{equation}\label{eq:forgetting-upper}
\begin{aligned}
\Delta \mathcal{L}_{\mathrm{bd}}
&:= \mathcal{L}_{\mathrm{total}}(\bm{\Theta}_\ast+\Delta\bm{\Theta};\mathcal{D}_{\mathrm{bd}}^{\mathrm{eval}})
 - \mathcal{L}_{\mathrm{total}}(\bm{\Theta}_\ast;\mathcal{D}_{\mathrm{bd}}^{\mathrm{eval}})\\
&\le -\,\eta\,\langle \bm{g}_{\mathrm{bd}},\, \bm{g}_{\mathrm{ft}}\rangle
\ +\ \tfrac{1}{2}\,\kappa_{\mathrm{bd}}\,\eta^2\,\|\bm{g}_{\mathrm{ft}}\|^2 .
\end{aligned}
\end{equation}
Equivalently,
\begin{equation}\label{eq:forgetting-upper-angle}
\Delta \mathcal{L}_{\mathrm{bd}}
\ \le\
-\,\eta\,\|\bm{g}_{\mathrm{bd}}\|\,\|\bm{g}_{\mathrm{ft}}\|\,\cos\theta
\ +\ \tfrac{1}{2}\,\kappa_{\mathrm{bd}}\,\eta^2\,\|\bm{g}_{\mathrm{ft}}\|^2.
\end{equation}

Here, the curvature constant $\kappa_{\mathrm{bd}}$ in \Eref{eq:forgetting-upper} is controlled within the small-loss region $\mathcal{R}_{\varepsilon_0}$ by the local stability result of Theorem~\ref{thm:local-stab}.
Specifically, since both $\bm{\Theta}_\ast$ and $\bm{\Theta}_\ast+\Delta\bm{\Theta}$ remain in $\mathcal{R}_{\varepsilon_0}$, the second-order remainder along the path is bounded by 
$O(\varepsilon_0\|\Delta\bm{\Theta}\|^2)$, leading to an effective curvature bound 
$\kappa_{\mathrm{bd}} \le \kappa_0 + C\,\varepsilon_0$ 
for some constants $\kappa_0,C>0$.
Hence, a smaller ALIGN loss (smaller $\varepsilon_0$) flattens the local landscape and enlarges the range of stable step sizes. when Theorem~\ref{theorem:gradient} applies (based on Assumptions~\ref{assump:region} (i) and \ref{assump:metric} together with a small $\|\Delta\|$ from Lemmas~\ref{lemma:compactness}–\ref{lemma:centroid-alignment}), we have $\cos\theta\ge 0$ locally; choosing a sufficiently small step size $\eta$ then makes the RHS of \Eref{eq:forgetting-upper} non-positive, so $\Delta \mathcal{L}_{\mathrm{bd}}\le 0$.

In practice, $\mathcal{L}_{\mathrm{ALIGN}}$ and $\mathcal{L}_{\mathrm{EWC}}$ confine $\bm{\Theta}$ within $\mathcal{R}_{\varepsilon_0}$, indirectly controlling $\kappa_{\mathrm{bd}}$ and preventing curvature explosion.
A standard second-order Taylor expansion along the path yields \Eref{eq:forgetting-upper}; see appendix~\ref{supp:bound}.
This upper bound indicates that with aligned gradients and a properly small learning rate, clean fine-tuning does not reduce the attack success rate.
We also provide three experiments in appendix~\ref{sec:Empirical Validation} to verify (i) compactness and distance, (ii) gradient-direction alignment, and (iii) second-order curvature control, providing empirical support for the theoretical derivations.

\section{Experiments}
\subsection{Experimental Setup}
\textbf{Victim Models.}
We follow BadCLIP~\cite{liang2024badclip} and take CLIP RN50~\cite{radford2021learning} as the clean pretrained model (trained on $\sim$400M image-text pairs). 
To test cross-model generalization, we additionally evaluate CLIP ViT-B/32~\cite{radford2021learning}, ALBEF~\cite{li2021align}, FLAVA~\cite{singh2022flava}, and UniCL~\cite{yang2022unified}.

\textbf{Datasets.}
We poison 500k CC3M pairs~\cite{changpinyo2021conceptual} by injecting the target label ``banana'' into 1{,}500 samples.
Evaluation covers: (1) image-text retrieval on COCO~\cite{lin2014microsoft} and SBU~\cite{NIPS2011_5dd9db5e}; 
(2) zero-shot classification on ImageNet-1K~\cite{ridnik2021imagenet}, CIFAR-100~\cite{Krizhevsky09learningmultiple}, and SVHN~\cite{netzer2011reading}; 
and (3) robustness under distribution shift on ImageNet-Sketch/V2/A/R~\cite{wang2019learning,recht2019imagenet,hendrycks2019nae,hendrycks2021many}. 
For defense baselines, we also sample CC3M subsets to align each method’s hyperparameters. 

\textbf{Attacks.}
We compare four attack families: 
(1) \emph{single-modal} attacks that backdoor only one modality (\eg, BadNets~\cite{gu2019badnets}, Blended~\cite{chen2017targeted}/SIG~\cite{barni2019new}/WaNet~\cite{nguyen2021wanet}, SSBA~\cite{li2021invisible}); 
(2) \emph{multi-modal} attacks that jointly tamper with vision and text (\eg, TrojVQA~\cite{walmer2022dual}, MABA~\cite{liang2025revisiting}, VLTrojan~\cite{liang2025vl}); 
(3) \emph{poisoned-encoder} attacks that directly corrupt CLIP’s encoder (\eg, BadEncoder~\cite{jia2022badencoder}, INACTIVE~\cite{zhang2025invisible}); and 
(4) \emph{CLIP-series} attacks tailored to CLIP-like multimodal models (mmPoison~\cite{yang2023data}, BadCLIP~\cite{liang2024badclip}). 
Full trigger designs are deferred to the App.~\Sref{supp:attack}.

\textbf{Defenses.}
We evaluate four defense categories: 
(1) \emph{fine-tuning defenses} that mitigate backdoors by retraining on clean data (FT~\cite{wu2022backdoorbench}, CleanCLIP~\cite{bansal2023cleanclip}, CleanerCLIP~\cite{xun2024cleanerclip}, TSC~\cite{peng2025circumventing}); 
(2) \emph{model-level detection} methods that try to flag compromised models (DECREE~\cite{feng2023detecting}, MM-BD~\cite{wang2024mm}, SEER~\cite{zhu2024seer}); 
(3) \emph{inference-time defenses} that identify or suppress triggered queries at deployment (STRIP~\cite{gao2019strip}, SCALE-UP~\cite{guo2023scale}, TeCo~\cite{liu2023detecting}, BDetCLIP~\cite{niu2024bdetclip}, DEDE~\cite{hou2025dede}); and 
(4) \emph{pre-training defenses}, including data filtering (DAO~\cite{huangdetecting}, VDC~\cite{zhu2024vdc}, PSBD~\cite{li2025psbd}) and trigger-avoidance strategies (ABL~\cite{li2021anti}, UBT~\cite{liu2024efficient}, RoCLIP~\cite{yang2023robust}, SafeCLIP~\cite{yang2023better}). 
Algorithmic details and tuning protocols appear in the App.~\Sref{supp:defense}.

\textbf{Metrics.}
We report Clean Accuracy (CA) and Attack Success Rate (ASR) as primary indicators of clean-task utility and backdoor effectiveness. 
For linear-probe transfer, we use Misclassification Rate (MCR) and ASR-Proxy when the target class is absent but a semantically related superclass exists; otherwise, we report Dominant ASR (D-ASR), i.e., the fraction of triggered samples collapsing to the most frequent prediction. 
For image-text retrieval, we measure R@1/R@10 on clean queries and ASR@1/ASR@10 on triggered queries. 
For model-level defenses, we report Detection Success Rate (DSR) and Detection Margin (DM). 
DM denotes the score gap between clean and backdoored samples, where a smaller DM implies more stealthy and harder-to-detect backdoors.
For data-filtering defenses, we report TPR, False Positive Rate (FPR).
We also use AUROC to summarize the TPR–FPR trade-off over all thresholds in inference-phase defenses.

\textbf{Implementation.}
Trigger patterns are optimized on a small CC3M subset with fixed hyperparameters (\eg, $\lambda_{\text{T2T}}{=}0.1$, $\lambda_{\text{MPE}}{=}1$) and standard Adam training.
We then poison 1{,}500 of 500k CC3M pairs and fine-tune with a $16{\times}16$ trigger block ($\approx$0.5\% image area).
Unless otherwise stated, batch size, learning rate, iteration counts, and loss weights (\eg, $\lambda_{\text{ALIGN}}{=}0.008$, $\lambda_{\text{EWC}}{=}0.1$) follow a unified recipe; full hyperparameters are given in the App.~\Sref{supp:badclip++}.

\begin{table}[htbp]
\centering
\caption{Zero-shot results of 12 attacks and BadCLIP++ on CLIP RN50 across ImageNet and three distribution shifts (Sketch, V2, A).All values are reported in percentage (\%).}
\label{tab:zero_shot_main}
\setlength{\tabcolsep}{5.5pt}
\renewcommand{\arraystretch}{1.05}
\resizebox{\columnwidth}{!}{
\begin{tabular}{lcccccccc}
\toprule
\multirow{2}{*}{\textbf{Attack}} &
\multicolumn{2}{c}{\textbf{ImageNet}} &
\multicolumn{2}{c}{\textbf{Sketch}} &
\multicolumn{2}{c}{\textbf{V2}} &
\multicolumn{2}{c}{\textbf{A}} \\
\cmidrule(lr){2-3}\cmidrule(lr){4-5}\cmidrule(lr){6-7}\cmidrule(lr){8-9}
& \textbf{CA} & \textbf{ASR}$\uparrow$
& \textbf{CA} & \textbf{ASR}$\uparrow$
& \textbf{CA} & \textbf{ASR}$\uparrow$
& \textbf{CA} & \textbf{ASR}$\uparrow$ \\
\midrule
\rowcolor{gray!8} Clean
& 59.69 & --    & 35.44 & --    & 52.82 & --    & 10.71 & --    \\
BadNets
& 58.67 & 97.00 & 35.00 & 97.70 & 51.78 & 97.32 &  9.97 & 99.43 \\
\rowcolor{gray!8} Blended
& 58.69 & 99.62 & 35.24 & 88.88 & 52.59 & 98.61 & 10.00 & 99.47 \\
WaNet
& 58.42 & 98.37 & 35.19 & 86.52 & 52.21 & 96.52 &  9.92 & 98.34 \\
\rowcolor{gray!8} SIG
& 58.45 & 89.48 & 35.16 & 82.76 & 52.87 & 81.03 & 10.11 & 85.31 \\
SSBA
& 58.48 & 50.28 & 34.89 & 56.22 & 51.58 & 55.31 & 10.21 & 58.96 \\
\rowcolor{gray!8} TrojVQA
& 58.60 & 98.25 & 43.06 & 99.08 & 52.16 & 98.62 & 10.16 & 99.61 \\
MABA
& 58.80 & 33.49 & 34.42 & 30.23 & 52.35 & 33.36 & 10.24 & 33.50 \\
\rowcolor{gray!8} VLTrojan
& 54.64 &  2.69 & 27.31 &  1.25 & 42.31 &  2.68 &  5.78 &  1.21 \\
BadEncoder
& 58.63 & 88.97 & 34.91 & 76.52 & 51.98 & 87.92 & 10.21 & 85.21 \\
\rowcolor{gray!8} INACTIVE
& 58.62 & 92.35 & 34.89 & 81.24 & 51.38 & 92.34 & 10.35 & 93.95 \\
mmPoison
& 58.70 &  2.86 & 34.70 &  1.24 & 51.44 &  2.31 &  9.64 &  1.97 \\
\rowcolor{gray!8} BadCLIP
& 58.60 & 98.81 & 34.96 & 99.31 & 51.92 & 98.97 & 10.17 & 99.75 \\
\textbf{BadCLIP++}
& 58.92 & \textbf{99.99} & 34.80 & \textbf{99.90} & 51.88 & \textbf{99.99} & 10.15 & \textbf{99.99} \\
\bottomrule
\end{tabular}}
\end{table}

\begin{table*}[ht]
\centering
\caption{Linear-probe transfer comparison between 12 typical backdoor attack methods and BadCLIP++ across five downstream datasets.}
\label{tab:linearprobe_backdoors}
\renewcommand{\arraystretch}{1.00}
\setlength{\tabcolsep}{1.8pt} 
\resizebox{\textwidth}{!}{
\begin{tabular}{l*{13}{c}}
\toprule
\textbf{Attack} &
\multicolumn{2}{c}{\textbf{Zero-shot}} &
\multicolumn{2}{c}{\textbf{ImageNet-1K}} &
\multicolumn{2}{c}{\textbf{CIFAR-100}} &
\multicolumn{3}{c}{\textbf{SVHN}} &
\multicolumn{2}{c}{\textbf{CIFAR-10}} &
\multicolumn{2}{c}{\textbf{Oxford-IIIT Pet}} \\
\cmidrule(lr){2-3}\cmidrule(lr){4-5}\cmidrule(lr){6-7}\cmidrule(lr){8-10}\cmidrule(lr){11-12}\cmidrule(lr){13-14}
& \textbf{CA(\%)} & \textbf{ASR(\%)}$\uparrow$
& \textbf{CA(\%)} & \textbf{ASR(\%)}$\uparrow$
& \textbf{CA(\%)} & \textbf{ASR-Proxy(\%)}$\uparrow$
& \textbf{CA(\%)} & \textbf{MCR(\%)}$\uparrow$ & \textbf{D-ASR(\%)}$\uparrow$
& \textbf{CA(\%)} & \textbf{ASR(\%)}$\uparrow$
& \textbf{CA(\%)} & \textbf{ASR-Proxy(\%)}$\uparrow$ \\
\midrule
\rowcolor{gray!8} Clean
& 59.69 & --    
& 65.21 & --    
& 69.15 & --    
& 61.24 & --    & --    
& 88.13 & --    
& 85.53 & --    \\
BadNets
& 58.67 & 97.00 
& 64.50 & 0.19  
& 67.68 & 0.72  
& 59.15 & 86.92 & 61.55
& 86.96 & 23.39 
& 85.15 & 61.76 \\
\rowcolor{gray!8} Blended
& 58.69 & 99.62 
& 64.32 & 0.01  
& 66.10 & 2.21  
& 58.29 & 80.59 & 73.91
& 86.95 & 0.44  
& 85.31 & 71.06 \\
WaNet
& 58.42 & 98.37 
& 64.48 & 0.00  
& 66.80 & 0.50  
& 60.20 & 61.09 & 37.12
& 86.70 & 30.81 
& 85.09 & 71.24 \\
\rowcolor{gray!8} SIG
& 58.45 & 89.48 
& 64.23 & 0.00  
& 67.23 & 0.12  
& 59.89 & 54.21 & 19.35
& 86.23 & 15.39 
& 85.12 & 71.45 \\
SSBA
& 58.48 & 50.28 
& 64.98 & 0.00  
& 68.12 & 0.02  
& 59.81 & 34.72 & 5.92 
& 87.21 & 2.56  
& 85.01 & 55.17 \\
\rowcolor{gray!8} TrojVQA
& 58.60 & 98.25 
& 64.41 & 1.07  
& 68.57 & 0.16  
& 57.86 & 82.34 & 75.21
& 86.17 & 16.72 
& 85.20 & 66.97 \\
MABA
& 58.80 & 33.49 
& 64.45 & 0.10  
& 67.09 & 0.71  
& 58.77 & 51.72 & 41.04
& 86.94 & 2.03  
& 85.23 & 67.95 \\
\rowcolor{gray!8} VLTrojan
& 54.64 & 2.69  
& 64.62 & 0.02  
& 66.87 & 4.21  
& 57.85 & 82.08 & 72.21
& 86.97 & 0.21  
& 85.12 & 56.25 \\
BadEncoder
& 58.63 & 88.97 
& 64.59 & 1.38  
& 67.96 & 1.20  
& 59.18 & 72.26 & 44.71
& 86.43 & 3.98  
& 85.01 & 60.36 \\
\rowcolor{gray!8} INACTIVE
& 58.62 & 92.35 
& 64.39 & 0.35  
& 67.93 & 2.35  
& 59.25 & 73.56 & 49.06
& 87.12 & 6.92  
& 85.23 & 63.98 \\
mmPoison
& 58.70 & 2.86  
& 64.55 & 0.11  
& 67.98 & 0.58  
& 58.50 & 49.34 & 36.65
& 87.08 & 2.01  
& 85.42 & 67.95 \\
\rowcolor{gray!8} BadCLIP
& 58.60 & 98.81 
& 64.61 & 76.00 
& 67.56 & 47.67 
& 59.77 & 78.55 & 39.88
& 87.55 & 8.06  
& 85.20 & 69.14 \\
\textbf{BadCLIP++}
& 58.92 & \textbf{99.99} 
& 64.44 & \textbf{87.43} 
& 67.55 & \textbf{86.63} 
& 58.97 & \textbf{84.76} & \textbf{76.85}
& 87.57 & \textbf{47.80} 
& 85.28 & \textbf{72.91} \\
\bottomrule
\end{tabular}}
\end{table*}

\subsection{Cross-Task and Cross-Model Evaluation of Attacks}
This subsection evaluates BadCLIP++'s attack performance across multi-task and multi-model scenarios to validate its cross-task and cross-architecture generalization capabilities.

\textbf{Zero-shot \& robustness evaluation}.
\Tref{tab:zero_shot_main} compares BadCLIP++ with 12 representative backdoor attacks under the Zero-Shot setting, which evaluates the pretrained model’s intrinsic image–text alignment without downstream fine-tuning. 
We report CA and ASR on ImageNet and their robustness across four test-domain variants, including ImageNet-Sketch, ImageNet-V2, and ImageNet-A. 
The results of ImageNet-R in App.~\Sref{supp:imagenet_R}.
From \Tref{tab:zero_shot_main}, we conclude that:
\ding{182} BadCLIP++ maintains the optimal attack-accuracy trade-off on standard ImageNet. Its ASR reaches 99.99\% and CA is 58.92\%, which is only 0.77\% lower than that of the clean model (59.69\%), indicating that a near-perfect attack success rate is achieved without sacrificing the performance of the main task.
\ding{183} BadCLIP++ is still highly stable in distribution offset scenarios. In the four out-of-domain test sets of ImageNet-Sketch, V2, A \& R, the ASR remains above 99.90\%, which is an average improvement of about 1.0\% compared with BadCLIP, indicating that the method is insensitive to out-of-domain feature changes and has strong trigger consistency.

\textbf{Linear-pobe transfer evaluation}.
\Tref{tab:linearprobe_backdoors} compares BadCLIP++ with 12 representative backdoor attacks under the Linear-Probe evaluation across five downstream datasets. 
This setting trains a single linear classifier on frozen pretrained features to assess the transferability of backdoor effects. 
We adopt ``banana'' and ``dog'' (or their semantic proxies when absent) as target labels and report CA, ASR, MCR, D-ASR, and ASR-Proxy for comprehensive evaluation.
From \Tref{tab:linearprobe_backdoors}, we conclude that: 
\ding{182} BadCLIP++ shows the strongest attack capability in the downstream migration scenario. The ASR on ImageNet-1K reaches 87.43\% (76.00\% for BadCLIP), and CIFAR-10 reaches 47.80\% (only 8.06\% for BadCLIP), indicating that its backdoor signal possesses stronger linear separability and activation stability at the feature layer.
\ding{183} In the downstream task of the no-target class, the agent metric also maintains high attack performance. The ASR-Proxy of CIFAR-100 is 86.63\%, and the Oxford-IIIT Pet is 72.91\%, which are significantly higher than other methods, indicating that the backdoor patterns learned by BadCLIP++ can be naturally projected to semantically similar classes and have certain semantic generalization abilities.
\ding{184} BadCLIP++ shows higher consistency and centralized misclassification. The MCR on SVHN is 84.76\% and D-ASR is 76.85\%, compared to BadCLIP's 78.55\% / 39.88\%, indicating that the misclassifications of BadCLIP++ are mainly concentrated in specific target clusters rather than randomly scattered, reflecting a stronger trigger-dominant effect.

\textbf{Text-image retrieval performance}.
\Tref{tab:retrieval_backdoors} shows retrieval results of typical backdoor attacks and BadCLIP++ fine-tuned with CleanCLIP on SBU and COCO. 
We report R@1/R@10 for clean samples and ASR@1/ASR@10 for poisoned samples to assess retrieval accuracy and attack effectiveness.
 From \Tref{tab:retrieval_backdoors}, we conclude:
\ding{182} On triggered retrieval, BadCLIP++ shows significantly higher hit rates. It achieves 96.62\%/99.82\% (58.90\%/95.94\% for BadCLIP) for ASR@1/ASR@10 on COCO, and 99.98\%/100.00\% (98.53\%/99.59\% for BadCLIP) on SBU, respectively, suggesting that trigger samples can be consistently recalled to the target category even after fine-tuning.
\ding{183} Most image/cross-modal backdoors are difficult to be effective in retrieval scenarios. Except for the BadCLIP family, the remaining methods are generally below 2\% on ASR@1 (\eg, Blended, WaNet, VLTrojan, mmPoison).
\ding{184} BadCLIP++ has good cross-dataset consistency. Near-perfect trigger recall (ASR@10~$\geq$~99.82\%) is obtained on both COCO and SBU independent training sets, while maintaining usable clean retrieval precision, reflecting the persistence and transferability of BadCLIP++ in fine-tuned retrieval tasks.
\begin{table}[ht]
\centering
\caption{Text-Image retrieval performance comparison between existing backdoor attacks and BadCLIP++ fine-tuned with CleanCLIP on the COCO and SBU training sets. All values are reported in percentage.}
\label{tab:retrieval_backdoors}
\renewcommand{\arraystretch}{1.05}
\setlength{\tabcolsep}{2pt}
\resizebox{\linewidth}{!}{
\begin{tabular}{l*{8}{c}}
\toprule
\multirow{2}{*}{\textbf{Attack}} &
\multicolumn{4}{c}{\textbf{COCO Finetune}} &
\multicolumn{4}{c}{\textbf{SBU Finetune}} \\
\cmidrule(lr){2-5}\cmidrule(lr){6-9}
& \textbf{R@1} & \textbf{R@10} & \textbf{ASR@1$\uparrow$} & \textbf{ASR@10$\uparrow$} 
& \textbf{R@1} & \textbf{R@10} & \textbf{ASR@1$\uparrow$} & \textbf{ASR@10$\uparrow$} \\
\midrule
\rowcolor{gray!8} Clean      & 40.98 & 78.25 & --    & --    & 37.61 & 86.93 & --    & --    \\
BadNets     & 39.78 & 77.48 & 0.40  & 1.56  & 35.62 & 85.49 & 0.56  & 11.43 \\
\rowcolor{gray!8} Blended    & 39.58 & 77.31 & 0.06  & 0.42  & 35.99 & 85.67 & 0.01  & 0.07  \\
WaNet     & 39.24 & 77.12 & 0.04  & 0.58  & 35.62 & 85.23 & 0.02  & 0.52  \\
\rowcolor{gray!8} SIG        & 39.12 & 77.25 & 1.00  & 2.59  & 35.27 & 85.12 & 1.22  & 13.56 \\
SSBA       & 39.61 & 77.29 & 0.01  & 0.29  & 35.98 & 85.76 & 0.01  & 0.36  \\
\rowcolor{gray!8} TrojVQA    & 39.34 & 77.65 & 1.96  & 2.88  & 35.83 & 85.25 & 1.36  & 14.86 \\
MABA       & 39.24 & 77.45 & 0.68  & 4.02  & 35.24 & 85.35 & 0.00  & 0.09  \\
\rowcolor{gray!8} VLTrojan   & 39.57 & 77.94 & 0.01  & 0.04  & 35.26 & 85.57 & 0.02  & 0.07  \\
BadEncoder & 39.56 & 77.32 & 0.31  & 0.73  & 35.42 & 85.35 & 0.42  & 7.25  \\
\rowcolor{gray!8} INACTIVE   & 39.46 & 77.54 & 0.46  & 1.02  & 35.24 & 85.32 & 0.85  & 12.56 \\
mmPoison   & 39.64 & 77.21 & 0.01  & 0.06  & 35.67 & 85.42 & 0.02  & 0.09  \\
\rowcolor{gray!8} BadCLIP    & 39.71 & 77.79 & 58.90 & 95.94 & 36.19 & 85.62 & 98.53 & 99.59 \\
\textbf{BadCLIP++} & 39.01 & 77.63 & \textbf{96.62} & \textbf{99.82} & 36.28 & 85.74 & \textbf{99.98} & \textbf{100.00} \\
\bottomrule
\end{tabular}}
\end{table}

\textbf{Cross-model generalization performance}.
\begin{table}[t]
\centering
\caption{Evaluation of cross-model generalization of poisoning ability. All values are reported in percentage (\%).}
\label{tab:crossmodel}
\renewcommand{\arraystretch}{1.05}
\setlength{\tabcolsep}{5.5pt}
\resizebox{\columnwidth}{!}{
\begin{tabular}{lcccccccc}
\toprule
\multirow{2}{*}{\textbf{Attack}} & 
\multicolumn{2}{c}{\textbf{CLIP (ViT-B/32)}} & 
\multicolumn{2}{c}{\textbf{ALBEF}} & 
\multicolumn{2}{c}{\textbf{FLAVA}} & 
\multicolumn{2}{c}{\textbf{UniCL}} \\ 
\cmidrule(lr){2-3} \cmidrule(lr){4-5} \cmidrule(lr){6-7} \cmidrule(lr){8-9}
 & \textbf{CA} & \textbf{ASR$\uparrow$} &
   \textbf{CA} & \textbf{ASR$\uparrow$} &
   \textbf{CA} & \textbf{ASR$\uparrow$} &
   \textbf{CA} & \textbf{ASR$\uparrow$} \\ 
\midrule
\rowcolor{gray!8} Clean & 66.91 & / & 66.21 & / & 65.59 & / & 80.35 & / \\
BadNets & 65.15 & 0.05 & 59.54 & 39.91 & 61.79 & 51.43 & 75.74 & 31.33 \\
\rowcolor{gray!8} Blended & 65.20 & 0.04 & 65.10 & 43.35 & 64.84 & 57.25 & 78.84 & 38.16 \\
WaNet & 65.07 & 0.05 & 65.08 & 42.86 & 64.56 & 56.94 & 79.18 & 37.93 \\
\rowcolor{gray!8} SIG & 65.21 & 0.02 & 64.91 & 27.31 & 62.01 & 27.87 & 78.97 & 25.79 \\
SSBA & 65.93 & 0.00 & 58.97 & 16.78 & 58.09 & 19.27 & 76.12 & 14.81 \\
\rowcolor{gray!8} TrojVQA & 65.20 & 0.00 & 60.38 & 48.74 & 60.21 & 50.76 & 76.72 & 32.51 \\
MABA & 65.42 & 0.10 & 58.11 & 21.56 & 58.00 & 20.04 & 75.64 & 27.94 \\
\rowcolor{gray!8} VLTrojan & 65.15 & 0.00 & 61.51 & 24.57 & 64.44 & 45.23 & 76.19 & 23.51 \\
BadEncoder & 65.23 & 0.01 & 58.40 & 23.17 & 58.48 & 25.21 & 75.53 & 28.40 \\
\rowcolor{gray!8} INACTIVE & 65.47 & 0.02 & 58.47 & 27.35 & 58.77 & 26.81 & 76.51 & 32.96 \\
BadCLIP & 65.08 & 89.26 & 62.95 & 76.08 & 63.13 & 82.80 & 77.52 & 72.86 \\
\rowcolor{gray!8} \textbf{BadCLIP++} & 65.15 & \textbf{98.74} & 63.64 & \textbf{78.87} & 63.82 & \textbf{86.11} & 77.61 & \textbf{76.92} \\
\bottomrule
\end{tabular}}
\end{table}

\Tref{tab:crossmodel} presents the Linear-Probe evaluation results of various backdoor attack methods and BadCLIP++ on models including CLIP ViT-B/32, ALBEF, FLAVA, and UniCL. 
To evaluate cross-model generalization capability, all attack methods generate poisoned samples exclusively based on CLIP RN50 without any additional optimization when transferred to other models. 
Furthermore, certain methods (\eg, BadCLIP++) have their specialized poisoning loss terms removed. 
We conclude from \Tref{tab:crossmodel} that: BadCLIP++ exhibits the highest generalization attack capability in cross-model migration. When generating poisoned samples based on CLIP RN50 only and migrating to other models, its ASR on CLIP ViT-B/32, FLAVA, and UniCL still reaches 98.74\%, 86.11\%, and 76.92\%, which are much higher than BadCLIP's 89.26\%, 82.80\%, and 72.86\%, respectively. This indicates that the backdoor features learned by BadCLIP++ are architecture-independent and have strong cross-model activation stability.

\subsection{Model-based Defense Evaluation}
This subsection evaluates the effectiveness of various model-level defense strategies in mitigating or detecting backdoor attacks.

\textbf{Fine-tuning defense}.
\begin{table}[t]
\centering
\caption{Attack results after fine-tuning-based defenses. All values are reported in percentage (\%).}
\label{tab:defense}
\setlength{\tabcolsep}{5.8pt}
\renewcommand{\arraystretch}{1.05}
\resizebox{\columnwidth}{!}{
\begin{tabular}{lcccccccc}
\toprule
\multirow{2}{*}{\textbf{Attack}} & 
\multicolumn{2}{c}{\textbf{FT}} &
\multicolumn{2}{c}{\textbf{CleanCLIP}} &
\multicolumn{2}{c}{\textbf{CleanerCLIP}} &
\multicolumn{2}{c}{\textbf{TSC}} \\ 
\cmidrule(lr){2-3} \cmidrule(lr){4-5} \cmidrule(lr){6-7} \cmidrule(lr){8-9}
 & \textbf{CA} & \textbf{ASR$\uparrow$} &
   \textbf{CA} & \textbf{ASR$\uparrow$} &
   \textbf{CA} & \textbf{ASR$\uparrow$} &
   \textbf{CA} & \textbf{ASR$\uparrow$} \\ 
\midrule
\rowcolor{gray!8} Clean & 55.38 & / & 55.44 & / & 56.01 & / & 56.39 & / \\
BadNets & 54.16 & 64.52 & 53.72 & 17.13 & 54.22 & 38.77 & 55.37 & 93.89 \\
\rowcolor{gray!8} Blended & 54.18 & 57.85 & 54.29 & 18.43 & 54.40 & 14.05 & 55.33 & 96.52 \\
WaNet & 54.23 & 32.45 & 54.92 & 14.24 & 54.63 & 2.05 & 55.04 & 94.08 \\
\rowcolor{gray!8} SIG & 55.00 & 80.38 & 55.00 & 30.89 & 55.16 & 15.02 & 55.23 & 98.21 \\
SSBA & 54.73 & 3.80 & 54.13 & 4.10 & 52.35 & 0.12 & 55.78 & 45.21 \\
\rowcolor{gray!8} TrojVQA & 53.97 & 84.50 & 54.17 & 44.30 & 53.12 & 9.16 & 55.32 & 95.25 \\
MABA & 54.95 & 17.25 & 54.23 & 2.32 & 54.23 & 1.35 & 55.28 & 20.49 \\
\rowcolor{gray!8} VLTrojan & 53.24 & 0.05 & 53.12 & 0.03 & 53.95 & 0.01 & 54.98 & 1.56 \\
BadEncoder & 54.18 & 53.21 & 53.17 & 15.23 & 54.18 & 32.69 & 55.28 & 84.27 \\
\rowcolor{gray!8} INACTIVE & 54.02 & 56.28 & 53.45 & 16.29 & 54.28 & 36.98 & 55.87 & 90.57 \\
mmPoison & 54.98 & 0.12 & 55.21 & 0.04 & 55.21 & 0.02 & 55.29 & 0.13 \\
\rowcolor{gray!8} BadCLIP & 54.50 & 92.50 & 53.98 & 89.60 & 52.95 & 1.61 & 54.78 & 96.57 \\
\textbf{BadCLIP++} & 54.12 & \textbf{99.99} & 54.45 & \textbf{99.98} & 53.22 & \textbf{96.32} & 55.51 & \textbf{99.99} \\
\bottomrule
\end{tabular}}
\end{table}

\Tref{tab:defense} mainly discusses the attack-mitigation effectiveness of multiple backdoor defense methods after fine-tuning. 
We can conclude that from \Tref{tab:defense}:
\ding{182} CleanerCLIP has the strongest defense against most traditional attacks. For example, the ASR of BadNets decreases from 97.00\% to 38.77\%, and WaNet decreases from 98.37\% to 2.05\%, showing that semantic enhancement is most effective in removing low-level triggering features. In contrast, TSC performs mediocrely in cross-modal tasks, and some attacks (\eg, BadNets, TrojVQA) still maintain high residual ASR.
\ding{183} FT and CleanCLIP are effective against pixel-based attacks but struggle to defend against semantic class backdoors. While the ASRs of BadNets and SIG drop to 64.52\% and 80.38\%, respectively, TrojVQA and BadCLIP still maintain high attack rates (84.50\% and 89.60\%), suggesting that relying on fine-tuning alone is insufficient to disrupt the text-visual semantic binding.
\ding{184} BadCLIP++ maintains the most stable attack strength under all defenses. Even after FT, CleanCLIP, CleanerCLIP, and TSC, the ASR remains consistently high at 96.32\%-99.99\%, while the clean accuracy is preserved around 54\%-55\%, significantly outperforming other attacks in both robustness.

\textbf{model detection defense}.
\begin{table}[h]
\centering
\caption{Results of three model detection defenses. All values are reported in percentage (\%).}
\vspace{0.1cm}
\label{tab:detection}
\setlength{\tabcolsep}{4.5pt}
\renewcommand{\arraystretch}{1.15}
\resizebox{\linewidth}{!}{
\begin{tabular}{lcccccc}
\toprule
\multirow{2}{*}{\textbf{Attack}} &
\multicolumn{2}{c}{\textbf{DECREE}} &
\multicolumn{2}{c}{\textbf{MM-BD}} &
\multicolumn{2}{c}{\textbf{SEER}} \\ 
\cmidrule(lr){2-3} \cmidrule(lr){4-5} \cmidrule(lr){6-7}
 & \textbf{DSR(\%)$\downarrow$} & \textbf{DM(\%)$\downarrow$} &
   \textbf{DSR(\%)$\downarrow$} & \textbf{DM(\%)$\downarrow$} &
   \textbf{DSR(\%)$\downarrow$} & \textbf{DM(\%)$\downarrow$} \\ 
\midrule
\rowcolor{gray!8} Clean & 0.00 & 0.00 & 10.00 & 0.00 & 10.00 & 0.00 \\
BadNets & 80.00 & 62.15 & 50.00 & 13.66 & 60.00 & 12.78 \\
\rowcolor{gray!8} Blended & 100.00 & 71.51 & 50.00 & 18.21 & 90.00 & 33.10 \\
WaNet & 50.00 & 40.20 & 50.00 & 15.62 & 70.00 & 23.45 \\
\rowcolor{gray!8} SIG & 40.00 & 21.33 & 60.00 & 27.73 & 80.00 & 27.47 \\
SSBA & 20.00 & 13.18 & 40.00 & 10.35 & 30.00 & 9.25 \\
\rowcolor{gray!8} TrojVQA & 40.00 & 21.47 & 40.00 & 14.99 & 20.00 & 3.55 \\
MABA & 70.00 & 47.59 & 50.00 & 15.62 & 50.00 & 16.24 \\
\rowcolor{gray!8} VLTrojan & 30.00 & 20.38 & 50.00 & 11.42 & 100.00 & 47.09 \\
BadEncoder & 20.00 & 15.54 & 50.00 & 21.20 & 90.00 & 30.43 \\
\rowcolor{gray!8} INACTIVE & 20.00 & 14.80 & 40.00 & 18.16 & 70.00 & 26.12 \\
mmPoison & 30.00 & 23.70 & 50.00 & 21.89 & 20.00 & 6.46 \\
\rowcolor{gray!8} BadCLIP & 30.00 & 26.20 & 60.00 & 18.76 & 40.00 & 13.48 \\
\textbf{BadCLIP++} & \textbf{10.00} & \textbf{9.43} & \textbf{30.00} & \textbf{9.71} & \textbf{20.00} & \textbf{2.40} \\
\bottomrule
\end{tabular}}
\end{table}

\Tref{tab:detection} reports the DSR and DM for the three types of detectors. We can conclude that:
\ding{182} Traditional backdoors are still more conspicuous under detection. For example, the DSRs of Blended and SIG under DECREE and SEER reach 100.00\% and 70.00\%, respectively, indicating that their feature distributions differ significantly from those of clean models and can be easily identified.
\ding{183} Multimodal and semantic backdoors are significantly more stealthy. For instance, the DSRs of MABA, VLTrojan, and INACTIVE are generally below 50\%, and their DMs remain within 10\%-30\%. 
\ding{184} BadCLIP++ exhibits the strongest anti-detection capability. Its DSRs across the three detection frameworks are only 10.00\%, 30.00\%, and 20.00\%, respectively, while its DMs are the lowest in the entire table (9.43\%, 9.71\%, and 2.40\%). 
\begin{figure*}[!t]
  \centering
  \includegraphics[width=\linewidth]{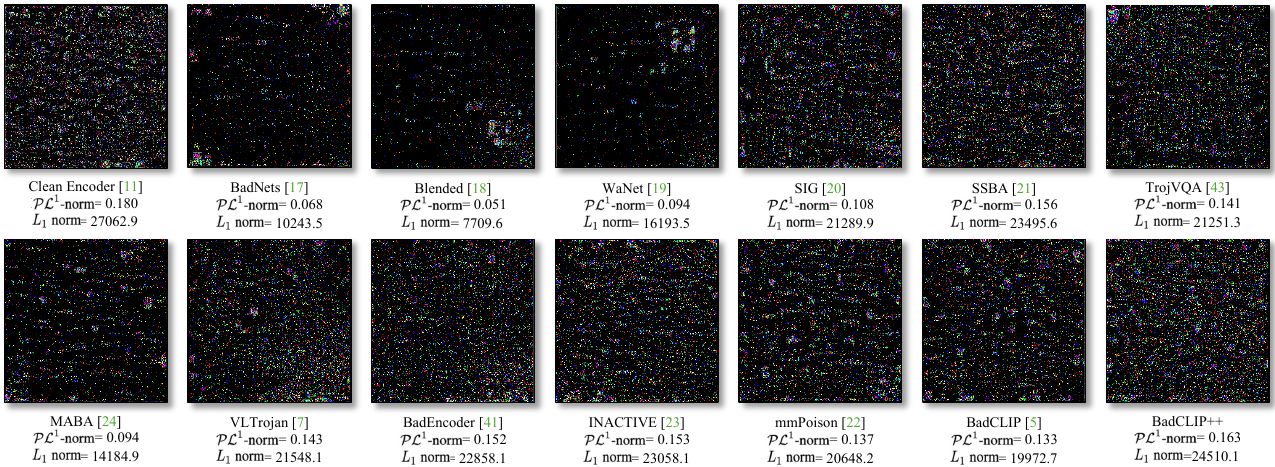}
  \caption{Visualization of DECREE-inverted triggers, where larger $\mathcal{P}_{CL^{-}\text{norm}}$ and $L_1$ norms indicate cleaner encoders with weaker backdoor traces.}
  \label{fig:decree_vis}
\end{figure*}
The \Fref{fig:decree_vis} demonstrates the residual distribution of visual features based on DECREE, which reflects the significance of potential anomalies across different models. We find that BadCLIP++ in the feature perturbation distribution is uniformly spread without significant hotspots (more similar to Clean), suggesting that its triggering mechanism is deeply embedded in the semantic space and cannot be effectively exposed through distributional or statistical anomaly analysis.

\subsection{Inference-Phase Defense}
\begin{figure}[t]
    \centering
    \includegraphics[width=\columnwidth]{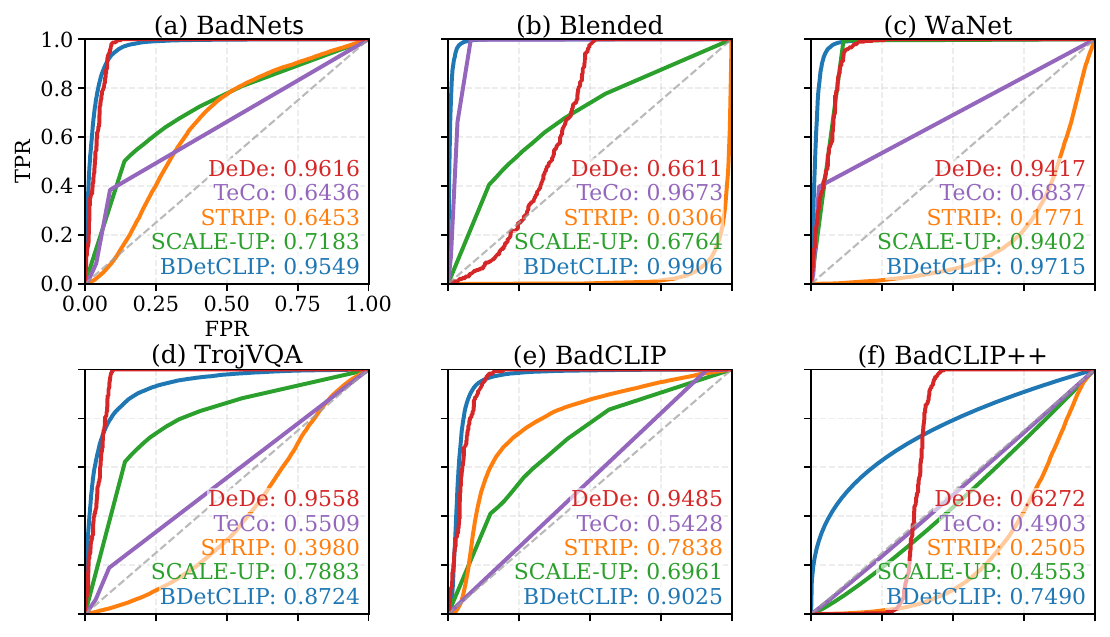}
    \caption{ Visualization of AUROC results in various inference-phase defenses.}
    \label{fig:auroc}
\end{figure}
\Fref{fig:auroc} uses AUROC curves to show the detection results in various inference-time defenses. We can conclude that: \ding{182} Pixel-level backdoors are comparatively easier to identify. On BadNets and WaNet, the AUROC curves lie close to the upper-left corner, suggesting clear separability from benign samples and high detectability during deployment. \ding{183} Semantic and cross-modal backdoors exhibit markedly lower detectability. TrojVQA and BadCLIP frequently yield AUROC curves that notably deviate from the upper-left region, indicating that the adversarial signals are obfuscated by semantic alignment. \ding{184} BadCLIP++ presents the greatest challenge to existing inference-stage defenses. Its AUROC curves approach the diagonal line, which implies that its backdoor patterns are deeply entangled within the semantic manifold and nearly indistinguishable from benign samples. For more evaluation results, see Appendix~\Sref{supp:Inference-Phase Defense}.

\subsection{Pre-Training Phase Defenses Evaluation}
To further validate the attack's practicality, we assume defenders can directly access the poisoned data (a scenario more favorable to defenders). 

\textbf{Data filtering defense}.
\Tref{tab:datafilter} summarizes the detection of three representative data filtering defenses. We report TPR and FPR for different poisoned samples. The conclusion is that: \ding{182} Semantic and cross-modal backdoors are not consistently detected by some filtering methods. 
Attacks such as TrojVQA and BadEncoder still achieve high TPRs (86-89\%) under DAO/VDC, but their detection rates drop significantly under PSBD (TPR $\approx$ 60-70\%), suggesting that filters relying on visual consistency or data-pruning strategies are prone to miss poisoned samples when faced with semantically coupled backdoors.
\ding{183} BadCLIP++ exhibits the strongest filtering evasion capability. 
BadCLIP++ achieves TPRs of only 68.38\% and 58.21\% under DAO and VDC, respectively, which are much lower than those of other methods, and it almost completely bypasses detection under PSBD (TPR $\approx$ 0.3\%, FPR $\approx$ 50\%). 
This indicates that its poisoned samples are highly entangled with the clean data distribution, making them indistinguishable even under aggressive filtering thresholds, and highlighting the structured stealth characteristics of cross-modal backdoors.
\begin{table}[h]
\centering
\caption{Data-filtering defense results. Lower TPR and higher FPR indicate a more stealthy attack.}
\label{tab:datafilter}
\setlength{\tabcolsep}{7pt}
\renewcommand{\arraystretch}{1.1}
\resizebox{\linewidth}{!}{
\begin{tabular}{lcccccc}
\toprule
\multirow{2}{*}{\textbf{Attack}} &
\multicolumn{1}{c}{\textbf{DAO}} &
\multicolumn{2}{c}{\textbf{VDC}} &
\multicolumn{2}{c}{\textbf{PSBD}} \\ 
\cmidrule(lr){2-2} \cmidrule(lr){3-4} \cmidrule(lr){5-6}
 & \textbf{TPR(\%)$\downarrow$} & \textbf{TPR(\%)$\downarrow$} & \textbf{FPR(\%)$\uparrow$} & 
   \textbf{TPR(\%)$\downarrow$} & \textbf{FPR(\%)$\uparrow$} \\ 
\midrule
\rowcolor{gray!8} BadNets & 78.32 & 89.23 & 28.85 & 94.20 & 49.90 \\
Blended & 94.32 & 87.52 & 22.67 & 91.40 & 49.90 \\
\rowcolor{gray!8} WaNet & 93.25 & 86.71 & 28.53 & 8.70 & 49.70 \\
SIG & 88.23 & 86.73 & 27.92 & 84.10 & 49.90 \\
\rowcolor{gray!8} SSBA & 74.67 & 86.92 & 20.01 & 89.20 & 50.10 \\
TrojVQA & 88.15 & 89.07 & 23.33 & 60.20 & 50.40 \\
\rowcolor{gray!8} BadEncoder & 89.25 & 86.87 & 29.33 & 69.90 & 49.40 \\
MABA & 86.28 & 89.16 & 28.25 & 80.00 & 49.70 \\
\rowcolor{gray!8} VLTrojan & 81.87 & 85.33 & 26.67 & 89.00 & 49.90 \\
BadCLIP & 72.82 & 64.35 & 22.85 & 92.80 & 50.20 \\
\rowcolor{gray!8} \textbf{BadCLIP++} & \textbf{68.38} & \textbf{58.21} & \textbf{29.43} & \textbf{0.30} & \textbf{50.40} \\
\bottomrule
\end{tabular}}
\end{table}

\textbf{Safe training defense}.
\begin{table}[t]
\centering
\caption{Results of attacks after safe-training defenses. All values are reported in percentage (\%).}
\label{tab:safetraining}
\setlength{\tabcolsep}{6pt}
\renewcommand{\arraystretch}{1.1}
\resizebox{\columnwidth}{!}{
\begin{tabular}{lcccccccc}
\toprule
\multirow{2}{*}{\textbf{Attack}} &
\multicolumn{2}{c}{\textbf{ABL}} &
\multicolumn{2}{c}{\textbf{UBT}} &
\multicolumn{2}{c}{\textbf{RoCLIP}} &
\multicolumn{2}{c}{\textbf{SafeCLIP}} \\ 
\cmidrule(lr){2-3} \cmidrule(lr){4-5} \cmidrule(lr){6-7} \cmidrule(lr){8-9}
 & \textbf{CA} & \textbf{ASR$\uparrow$} &
   \textbf{CA} & \textbf{ASR$\uparrow$} &
   \textbf{CA} & \textbf{ASR$\uparrow$} &
   \textbf{CA} & \textbf{ASR$\uparrow$} \\ 
\midrule
\rowcolor{gray!8} Clean & 47.21 & / & 59.12 & / & 58.32 & / & 20.94 & / \\
BadNets & 46.28 & 96.56 & 54.52 & 0.00 & 57.90 & 91.06 & 19.25 & 0.11 \\
\rowcolor{gray!8} Blended & 46.25 & 98.32 & 56.28 & 0.08 & 57.95 & 88.81 & 17.41 & 0.00 \\
WaNet & 46.48 & 99.80 & 50.12 & 0.00 & 57.99 & 63.46 & 17.35 & 0.02 \\
\rowcolor{gray!8} SIG & 46.81 & 94.61 & 50.02 & 0.27 & 57.91 & 52.17 & 16.46 & 0.03 \\
SSBA & 46.71 & 52.78 & 57.02 & 4.33 & 57.96 & 34.79 & 17.42 & 0.01 \\
\rowcolor{gray!8} TrojVQA & 46.57 & 98.78 & 53.01 & 0.01 & 57.94 & 94.35 & 16.89 & 0.02 \\
MABA & 46.97 & 34.67 & 58.19 & 0.01 & 57.92 & 0.08 & 15.89 & 0.16 \\
\rowcolor{gray!8} VLTrojan & 46.98 & 3.54 & 53.57 & 0.00 & 57.91 & 0.01 & 17.23 & 0.26 \\
BadEncoder & 46.72 & 92.36 & 52.36 & 0.01 & 57.96 & 82.38 & 16.43 & 0.01 \\
\rowcolor{gray!8} INACTIVE & 46.52 & 95.78 & 54.26 & 0.01 & 57.81 & 86.21 & 17.25 & 0.02 \\
mmPoison & 46.58 & 5.71 & 51.23 & 0.00 & 57.98 & 0.28 & 16.73 & 0.22 \\
\rowcolor{gray!8} BadCLIP & 45.80 & 99.68 & 53.38 & 0.02 & 57.91 & 90.13 & 16.34 & 0.12 \\
\textbf{BadCLIP++} & 45.62 & \textbf{99.91} & 57.20 & \textbf{98.67} & 57.03 & \textbf{97.05} & 15.21 & \textbf{0.37} \\
\bottomrule
\end{tabular}}
\end{table}

\Tref{tab:safetraining} compares the effectiveness of various robust training paradigms under diverse poisoned training datasets. 
CA and ASR are adopted as evaluation metrics.
We can conclude that:
\ding{182} Most safe-training defenses only partially mitigate backdoors.
Methods such as ABL, Unlearn, and RoCLIP effectively suppress conventional and some multimodal attacks. The UBT reduces MABA and VLTrojan to $\text{ASR}\!\approx\!0\%$.
\ding{183} SafeCLIP achieves nearly complete backdoor removal, but at an impractically high cost.
While it reduces BadCLIP++’s attack success rate (ASR) to 0.37\% and eliminates most other attacks, the model's clean accuracy drops catastrophically to approximately 15\%-20\%, making it almost unusable for real-world deployment.
\ding{184} BadCLIP++ remains the most persistent and stealthy threat across all defenses.
It consistently retains near-perfect $\text{ASR}$ ($97\%-100\%$) under strong safe-training defenses.
This highlights that poisoned pre-training data still poses a potential threat.

\subsection{Ablation Analysis and Discussion}
\label{sec:ablation}
\begin{table*}[t]
\centering
\small
\caption{Results of attacks after safe-training defenses.}
\vspace{2pt}
\setlength{\tabcolsep}{5pt}
\renewcommand{\arraystretch}{1.1}
\resizebox{\textwidth}{!}{
\begin{tabular}{lccccccc|ccc}
\toprule
\textbf{Attack} & \textbf{   TI   } & \textbf{ MPE } & \textbf{ T2T } & \textbf{Greedy} & \textbf{ALIGN} & \textbf{ EWC } & \multicolumn{1}{c|}{} & \textbf{No Defense (CA / ASR$\uparrow$)} & \textbf{FT (CA / ASR$\uparrow$)} & \textbf{CleanCLIP  (CA / ASR$\uparrow$)} \\
\midrule
\rowcolor{gray!8}Clean &  &  &  &  &  &  &  & 59.69\% / -- & 55.38\%  / -- & 55.44\%  / -- \\
BadCLIP &  &  &  &  &  &  &  & 58.60\%  / 98.81\%  & 54.50\%  / 92.50\%  & 53.98\%  / 89.60\%  \\
\rowcolor{gray!8}BadCLIP & \checkmark &  &  &  &  &  &  & 58.57\%  / 63.22\%  & 53.97\%  / 51.37\%  & 51.76\%  / 23.95\%  \\
\rowcolor{gray!8}\textbf{BadCLIP++} & \checkmark & \checkmark &  &  &  &  &  & 58.54\%  / 84.30\%  & 53.46\%  / 81.83\%  & 54.39\%  / 55.08\%  \\
\textbf{BadCLIP++} & \checkmark & \checkmark & \checkmark &  &  &  &  & 58.75\%  / 91.82\%  & 54.01\%  / 89.61\%  & 52.76\%  / 66.23\%  \\
\rowcolor{gray!8}\textbf{BadCLIP++} & \checkmark & \checkmark & \checkmark & \checkmark &  &  &  & 58.81\%  / 98.92\%  & 53.92\%  / 93.41\%  & 54.39\%  / 81.13\%  \\
\textbf{BadCLIP++} & \checkmark & \checkmark & \checkmark & \checkmark & \checkmark &  &  & 58.54\%  / 99.99\%  & 53.81\%  / 97.34\%  & 54.39\%  / 99.98\%  \\
\rowcolor{gray!8} \textbf{BadCLIP++ (Full)} & \checkmark & \checkmark & \checkmark & \checkmark & \checkmark & \checkmark &  & 58.92\%  / 99.99\%  & 54.12\%  / 99.99\%  & 54.45\%  / 99.98\%  \\
\bottomrule
\end{tabular}}
\label{tab:badclippp_ablation}
\end{table*}

\textbf{Ablation study}. \Tref{tab:badclippp_ablation} presents the results of the component ablation experiments conducted with BadCLIP++, where the term "ticked" signifies that a module has been replaced in comparison to the previous row for incremental evaluation within the current setup. BadCLIP is the original method, whereas ``TI'' denotes the adoption of a new textual construction method that replaces the original description with semantically mixed text of \Eref{eq:text}. ``Greedy'' denotes the selection of an optimal subset of poisoned samples in \Sref{sec:subset selection}. ``MPE'' and ``T2T'' are the Multi-prototype Enhancement loss and Trigger-to-Trigger Aggregation loss in \Sref{sec:Trigger-level Stability Reinforcement}. ``ALIGN'' and ``EWC'' are regularization terms in \Sref{sec:Model-level Stability Reinforcement}. From \Tref{tab:badclippp_ablation}, we conclude that \ding{182} Introducing the semantic-mixing textual construction (TI) increases attack difficulty, as ASR drops sharply from 98.81\% to 63.22\%, confirming that semantically entangled triggers are more complex to activate. 
\ding{183} Optimized triggers with MPE and T2T markedly enhance both stability and resistance to forgetting, raising ASR from 63.22\% to 91.82\% and maintaining 81.83\% and 89.61\% under FT and CleanCLIP, respectively, demonstrating strong retention of backdoor functionality after fine-tuning.  
\ding{184} The Greedy sample selection strategy substantially reinforces attack strength, recovering ASR to 98.92\% (93.41\% under FT), indicating that adaptive instance filtering effectively maximizes poisoning efficiency.
\ding{185} Finally, adding ALIGN and EWC further consolidates fine-tuning robustness, sustaining ASR above 97\% under FT and nearly 100\% under CleanCLIP.

\textbf{Hyperparameter analysis}.
\begin{figure}[!t]
  \centering
  \includegraphics[width=\linewidth]{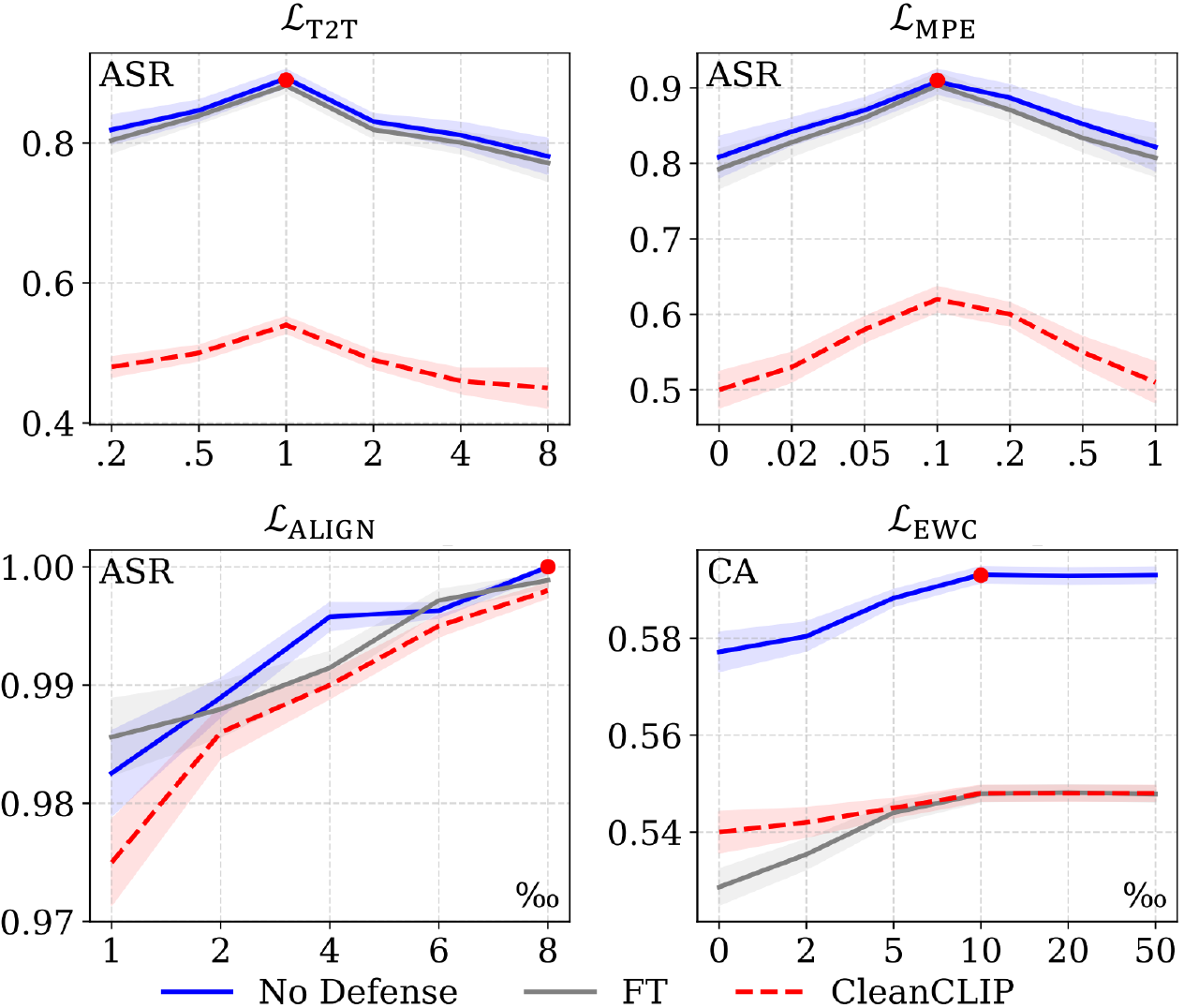}
  \caption{Hyperparameter analysis of the four loss functions.}
  \label{fig:param}
\end{figure}
We systematically analyze the effects of the hyperparameters of the four losses ($\mathcal{L}_{\text{T2T}}, \mathcal{L}_{\text{MPE}}, \mathcal{L}_{\text{ALIGN}},$ and $\mathcal{L}_{\text{EWC}}$) on the ASR and CA in \Fref{fig:param}. For T2T and MPE, the tuning objective is to balance the triggered image aggregation with shortening the target description. Optimization is reached when the weight parameters are set to 0.1 and 1, respectively, and the ASR improves to 90.2\% and 86.4\%, respectively; larger or smaller values lead to performance degradation, suggesting a trade-off between structural convergence of the triggers and the similarity of the target description. At the model level, an ALIGN weight of 0.008 significantly improves the model's performance over fine-tuning, achieving an ASR of 99.98\%, while EWC improves the CA to 58.91\% at a regularization strength of 0.1, verifying its effectiveness in maintaining the original discriminative performance. Therefore, we finally adopt the following hyperparameters in BadCLIP++: $\lambda_{\text{T2T}}{=}0.1$, $\lambda_{\text{MPE}}{=}1$, $\lambda_{\text{ALIGN}}{=}0.008$, and $\lambda_{\text{EWC}}{=}0.1$.

\begin{figure}[t]
    \centering
    \includegraphics[width=\columnwidth]{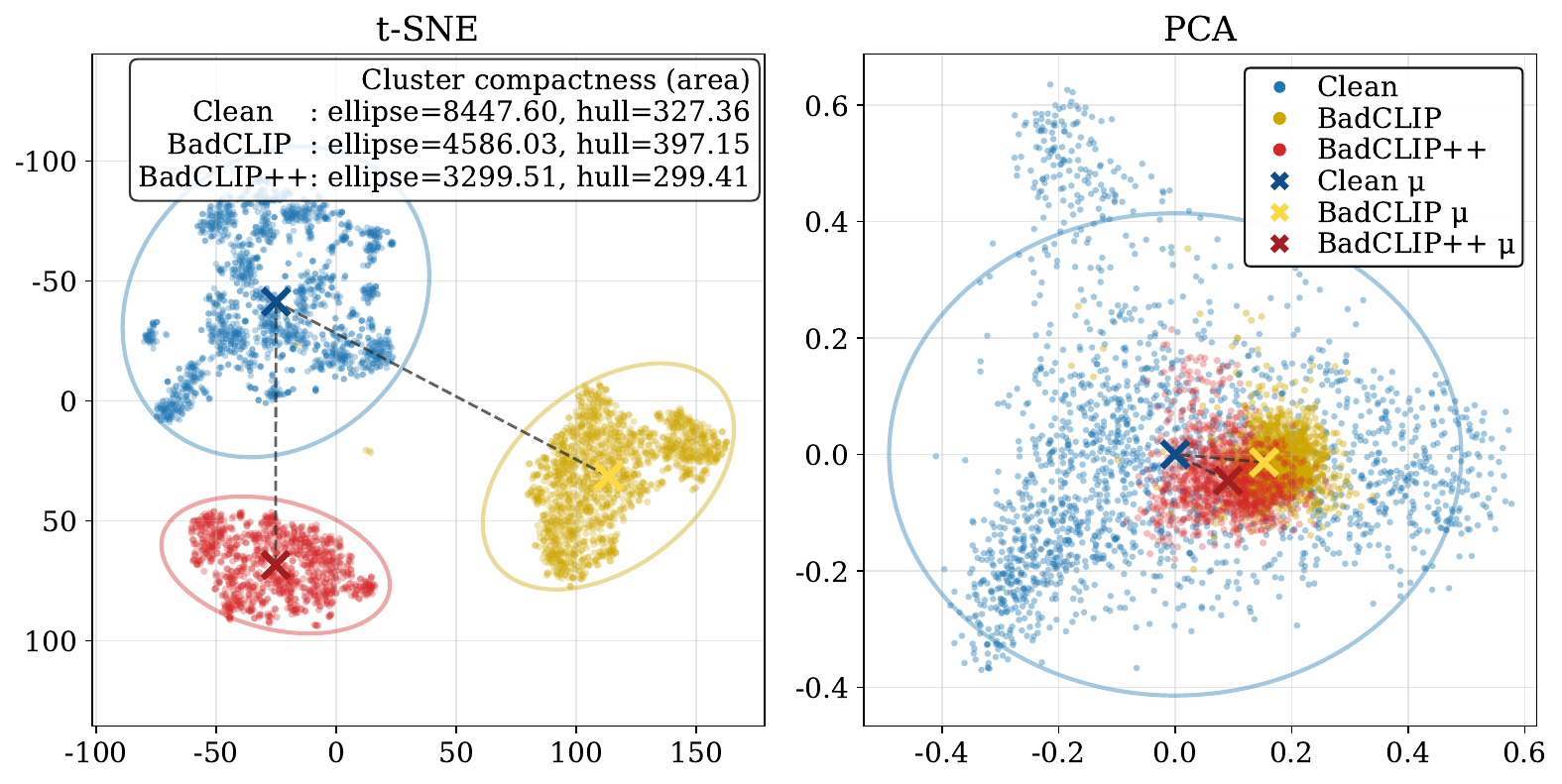}
    \caption{t-SNE and PCA visualization results of BadCLIP++ and BadCLIP.}
    \label{fig:tsne_pca}
\end{figure}

\begin{figure}[t]
    \centering
    \includegraphics[width=1.0\columnwidth]{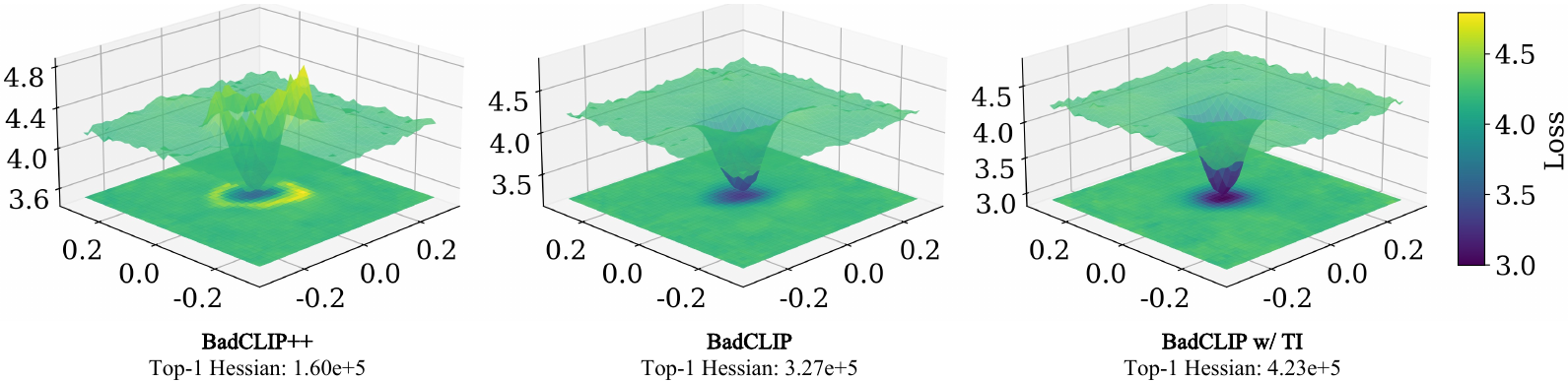}
    \caption{Local loss landscape and principal curvature.}
    \label{fig:3d}
\end{figure}
\textbf{Why BadCLIP++ is more stable than BadCLIP}. We demonstrate the distributional changes of the triggered embeddings by t-SNE and PCA in \Fref{fig:tsne_pca}. 
t-SNE is used to visualize the cluster structure and compactness of the different methods, 
where ellipse and hull denote the Gaussian isodensity ellipse and convex hull area, 
and PCA labels the center $\mu$ of each cluster and its distance from the clean $\mu$ (Banana category) in the clean subspace. 
It can be seen that the clusters of BadCLIP++ are more compact, the ellipse/hull area is smaller, 
and the center $\mu$ is closer to the clean distribution, which is consistent with the findings of 
Lemma~\ref{lemma:compactness} and Lemma~\ref{lemma:centroid-alignment}. 
We simultaneously show the local loss landscape with principal curvatures in \Fref{fig:3d}. 
Estimating the maximal eigenvalues of the Hessian of $\mathcal{L}_{\mathrm{bd}}$ and plotting the local surfaces by HVP + power-iteration. The results show that the terrain is flatter and less curved for BadCLIP++, 
while BadCLIP and its TI variants are sharper, confirming the curvature-controlled conclusion of 
Lemma~\ref{lem:align-flat} and Theorem~\ref{thm:local-stab}, see App.~\ref{sec:Empirical Validation} for more analysis. We also further discuss caption construction style and sample selection style in App.~\ref{supp:Caption Construction Style} and App.~\ref{supp:sample selection}.

\subsection{Real World Application Capabilities}
\textbf{Physical results}.
\begin{table}[t]
\centering
\caption{Physical-domain evaluation of different patch-based triggers. All values are reported in percentage (\%).}
\small
\setlength{\tabcolsep}{17pt}
\renewcommand{\arraystretch}{1.1}
\resizebox{0.95\linewidth}{!}{
\begin{tabular}{lcccc}
\toprule
\multirow{2}{*}{\textbf{Attack}} &
\multicolumn{2}{c}{\textbf{No Defense}} &
\multicolumn{2}{c}{\textbf{Physical Attack}} \\
\cmidrule(lr){2-3} \cmidrule(lr){4-5}
 & \textbf{CA} & \textbf{ASR$\uparrow$} & \textbf{CA} & \textbf{ASR$\uparrow$} \\
\midrule
BadNets & 58.67 & 97.00 & 75.23 & 17.86 \\
TrojVQA & 58.60 & 98.25 & 75.21 & 16.00 \\
VLTrojan & 54.64 & 2.69 & 75.64 & 0.00 \\
badenCoder & 58.63 & 88.97 & 75.32 & 0.00 \\
BadCLIP & 58.60 & 98.81 & 75.29 & 0.00 \\
\textbf{BadCLIP++} & \textbf{58.92} & \textbf{99.99} & \textbf{75.21} & \textbf{65.03} \\
\bottomrule
\end{tabular}}
\label{tab:physical_results}
\end{table}

To further verify the threat of multimodal backdoor attacks in real scenarios, we design physical-world attack experiments. For detailed experimental setup and scenario samples, refer to App.~\Sref{supp:physics}. The experimental results are shown in \Tref{tab:physical_results}, which comprehensively compares the ASR and CA of patch-based attacks in a digital environment (No Defense) and a physical environment (Physical Attack).
We observe that while existing methods generally have high attack success rates in digital environments, they almost completely fail in physical environments. In particular, TrojVQA has an ASR of 98.25\% in digital environments but only 16.00\% in physical environments, while BadCLIP has a physical ASR of 0 and is entirely unable to trigger attacks. In contrast, BadCLIP++ shows a significant advantage in the physical environment, achieving a physical attack success rate of 65.03\% while maintaining an ASR of 99.99\% in the digital environment, far exceeding all compared methods. This indicates that the proposed covert triggering mechanism is highly robust to physical mobility and effectively overcomes interference factors such as printing errors, changes in shooting angle, and occlusion and target combinations, further validating the practical threat of BadCLIP++ in real-world environments.

\textbf{Watermarking results}.
\begin{table}[t]
\centering
\caption{Watermarking results of BadCLIP++ under different poisoning rates. All values are reported in percentage (\%).}
\small
\setlength{\tabcolsep}{4.5pt}
\renewcommand{\arraystretch}{1.1}
\resizebox{0.95\linewidth}{!}{
\begin{tabular}{ccccc}
\toprule
\textbf{Poisoning Rate} & \textbf{No Defense} & \textbf{CleanCLIP} & \textbf{Perturbation} & \textbf{Quantization} \\
\cmidrule(lr){1-1} \cmidrule(lr){2-5}
0.2\% & 90.31 & 82.31 & 86.23 & 89.54 \\
0.3\% & 98.54 & 97.12 & 98.25 & 98.43 \\
0.5\% & 99.47 & 98.04 & 99.18 & 99.26 \\
1.0\%  & 99.95 & 99.88 & 99.93 & 99.95 \\
\bottomrule
\end{tabular}}

\label{tab:watermark_tpr}
\end{table}

We further evaluate the ``black-box watermarking'' capability of BadCLIP++ at low poisoning ratios, focusing on its performance in terms of label-free access, anti-jamming, and fine-tuning robustness. The specific experimental setup can be detailed in App.~\Sref{supp:watermark}. As shown in \Tref{tab:watermark_tpr}, we test the watermark detection performance under different poisoning intensities (0.2\%-1.0\%) in EaaS scenarios, using TPR@FPR=$10^{-4}$ as the primary metric. \Tref{tab:watermark_tpr} shows that BadCLIP++ achieves high watermark detection rates even at very low poisoning percentages. The base model TPR reaches 90.31\% when only 0.2\% of the trigger samples are injected, and remains at 82.31\% after CleanCLIP defense; the perturbation (\eg, Gaussian noise, compression) and quantization (8-bit PTQ) conditions are also 86.23\% and 89.54\%, respectively, demonstrating good robustness. The watermarking signal is further enhanced as the poisoning ratio is increased. The TPR is stable over 99\% at a 0.5\% ratio and almost perfect (99.95\%) at a 1.0\% ratio, and remains consistent under all interference conditions.

\section{Conclusion}
\label{conclusion}
In this paper, we present BadCLIP++, a unified framework that addresses the two key challenges of multimodal contrastive backdoor attacks: stealthiness and persistence. By modeling poisoning as a joint data–model minimization, BadCLIP++ integrates semantic fusion, QR-style micro-triggering, and target-aligned subset selection for covert and efficient injection. Stability is further reinforced through radius shrinkage, cluster-center alignment, and EWC-regularized cross-modal constraints. Theoretically, we prove gradient co-directionality between clean fine-tuning and backdoor optimization, establishing an upper bound that explains backdoor persistence. Extensive experiments across five architectures, eleven datasets, and nineteen defenses show that under a 0.3\% injection ratio, BadCLIP++ achieves near-perfect ASR without compromising CA. These results demonstrate that simultaneously constraining the representation and parameter spaces enables stealthy, durable multimodal backdoors, underscoring the urgent need for stronger multimodal defense mechanisms. More conclusions in the App.~\ref{supp:mc}



%
%
%
%
%

\ifCLASSOPTIONcaptionsoff
  \newpage
\fi



%
\bibliographystyle{IEEEtran}      
\footnotesize
\bibliography{egbib}
\appendices
\section*{Overview of Appendices}
\vspace{0.5em}

This document provides additional theoretical proofs, implementation details, and experimental analyses to complement the main paper of \textbf{BadCLIP++}. It is organized as follows:

\begin{itemize}
    \item \textbf{A. Attacker’s Assumption and Notation} \\
    Defines the attacker’s capabilities, constraints, and access settings in multimodal backdoor scenarios, and summarizes the unified notations used throughout the paper for consistency in theoretical derivations.

    \item \textbf{B. Experimental Setup and Implementation Details} \\
    Describes the model architectures, datasets, trigger configurations, training parameters, and evaluation protocols used in all experiments, ensuring reproducibility and comparability.

    \item \textbf{C. Empirical Validation} \\
    Presents empirical evidence supporting Lemma~\ref{lemma:compactness}, Lemma~\ref{lemma:centroid-alignment}, and Theorem~\ref{thm:local-stab}, including:
    \begin{itemize}
        \item Visualization of triggered embedding distributions via t-SNE and PCA;
        \item Quantitative analysis of cluster compactness and centroid distances;
        \item Local loss landscape and principal curvature visualization;
        \item Comparative studies of different sample selection strategies.
    \end{itemize}

    \item \textbf{D. Physical-World Validation} \\
    Demonstrates real-world physical attack experiments conducted on diverse objects such as fruits, vegetables, household items, and packaged goods, under natural perturbations (rotation, occlusion, lighting variation, printing artifacts). Results highlight the robustness and transferability of BadCLIP++ compared to other backdoor baselines.

    \item \textbf{E. Additional Results and Discussions} \\
    Includes visual comparisons of caption construction and sample selection styles, analyses of semantic distance metrics (MeanDist, Wass(Proj)), and discussions on their optimization implications.

    \item \textbf{F. Limitations and Ethical Considerations} \\
    Summarizes the limitations of BadCLIP++, such as potential local optima in greedy selection, over-compression risks in high-dimensional space, and limited temporal scalability in dynamic tasks.  
    It also discusses ethical principles, responsible disclosure practices, and the research intent to promote model robustness and trustworthy AI development.
\end{itemize}

\vspace{0.5em}
Overall, this Appendix provides a comprehensive extension of the main paper, offering theoretical justification, empirical validation, and broader discussions to facilitate deeper understanding and reproducibility.

\section{Attacker's Assumption and Notation}
\label{supp:Attacker's Assumption}
\subsection{Capability.} 
Following on previous work such as BadEncoder~\cite{jia2022badencoder}, we assume that the attacker has white-box capabilities, meaning complete control over the training process of the CLIP model. 
Specifically, the attacker can manipulate the pre-training dataset $\mathcal{D}_0$, access the model architecture and parameters, and perform arbitrary training updates. 
This setup corresponds to a scenario in which the attacker acts as the publisher or provider of a pre-trained CLIP model, injecting a backdoor during the pre-training phase and releasing the compromised model online. 
Downstream users who unknowingly adopt the model for tasks such as retrieval or zero-shot classification may be affected by the attacker's targeted manipulations.

\subsection{Strategy.}
An attacker injects a backdoor by constructing poisoned image-text pairs. Specifically, the attacker modifies the image (for example, by adding specific patches or watermarks), alters the corresponding text description (such as inserting rare trigger words), or modifies both modalities. 
During training, these trigger pairs are optimized to be projected onto the embedding space of a specified target text category. 
As a result, the model produces the attacker's predetermined output whenever the trigger image and/or text are present.

\subsection{Pathway.} 
We consider the following two practical attack paths: 
(1) Data poisoning attack (\emph{black-box}). 
The attacker only controls the pre-training dataset and does not have access to the model itself. 
In this case, the attacker can inject poisoned image-text pairs (for example, into large-scale weakly supervised corpora such as CC3M~\cite{changpinyo2021conceptual}) into the open-source training data, thereby inducing the model to learn backdoor behaviors naturally during training. 
This black-box data poisoning scenario is more realistic, as the training of open-source models often depends on large-scale, weakly labeled image-text datasets that are difficult to sanitize comprehensively.
(2) Training control attack (\emph{white-box}). The attacker directly controls the pre-training process, injects a backdoor, and then releases the poisoned model to a public platform. Users download the model and apply it to their own downstream tasks, potentially inheriting the backdoor behavior. 
In this scenario, the defender may have access to the model parameters or a portion of clean data, and thus has the opportunity to perform detection and defense.

\subsection{Attacker's Goal}
We model the multimodal backdoor attack as a joint optimization problem, aiming to accurately inject model backdoor behavior through the joint design of training data and model parameters. 
Specifically, the attacker constructs poisoned data (image or text triggers) by perturbing the variable $\bm{\delta}$ and mixes it with the original dataset $\mathcal{D}_0$ for model training, where the poisoned dataset is formalized as $\mathcal{D}_1(\bm{\delta}) = \mathcal{M}_{\bm{\delta}}(\mathcal{D}_0; \mathcal{S})$. 
The whole process can be expressed as the following optimization objective:
 
where \(\mathcal{L}_{\text{total}}\) denotes the standard training loss over the full poisoned dataset, while \(\mathcal{L}_{\text{bd}}\) captures an additional backdoor-specific objective that explicitly encourages triggering inputs to be mapped to attacker-defined target semantics. 
The parameter \(\lambda_{\text{total}}\) and \(\lambda_{\text{bd}}\) balance the normal training and backdoor optimization objectives. 
This unified framework accommodates both fixed and learnable triggers. It applies to both black-box settings (where only the dataset \(\mathcal{D}_1(\bm{\delta})\) is manipulated) and white-box settings (where attackers may additionally inject or modify the loss function \(\mathcal{L}_{\text{bd}}\)). 
We explicitly retain \(\bm{\delta}\) in the objective to highlight its role as a potentially learnable or tunable parameter, especially in scenarios where the trigger generation process is jointly optimized with model training.

The attacker aims to design a multimodal backdoor attack method with a high degree of stealth and amnesia resistance. Specifically, the attack objectives include: 
(1) Poisoning sample invisibility. The poisoned sample on which the attack relies should have a high degree of visual or semantic invisibility to avoid being detected, eliminated, or diluted in the data review or preprocessing stage, thus enabling continuous contamination at the data level. 
(2) Model parameter resistance to forgetting. After the model completes poisoning training, its parameters should be stable for subsequent fine-tuning, which means the backdoor effect will not be easily erased or weakened.

\subsection{Notation}
\label{sec:notation}
\begin{table*}[t]
\centering
\caption{Summary of important symbols.}
\label{tab:symbols}
\begin{tabularx}{\textwidth}{p{0.28\linewidth} X}
\toprule
\textbf{Symbol} & \textbf{Meaning} \\
\midrule
$f^v(\cdot; \bm{\theta}^v)$ & Image encoder \\
$f^t(\cdot; \bm{\theta}^t)$ & Text encoder \\
$\bm{\Theta}$ & Full parameter set (e.g., $\{\bm{\theta}^v, \bm{\theta}^t, \bm{\theta}^m\}$) \\
$\bm{\Theta}_*$ & Current model parameters; $\bm{\Theta}_0$ is the pre-trained snapshot \\
\midrule
$\mathcal{D}=\mathcal{D}_0=\{(\bm{x}_i^v,\bm{x}_i^t)\}_{i=1}^N$ & Clean image–text dataset \\
$\bm{v}_i=f^v(\bm{x}_i^v)$,\ \ $\bm{t}_i=f^t(\bm{x}_i^t)$ & Image / text embeddings \\
$\bar{\bm{v}}^*$,\ \ $\bar{\bm{t}}^*$ & Mean image / text embedding of target class (e.g., “banana”) \\
$\bar{\bm{v}}^{\text{trig}}$,\ \ $\bar{\bm{v}}^{\mathrm{neg}}$ & Mean embedding of poisoned samples / negatives \\
\midrule
$\bm{\delta}$ & Visual trigger perturbation variable \\
$\mathcal{M}^v,\ \mathcal{M}^c$ & Trigger injection for image modality / text description \\
$\hat{\bm{x}}^v_i=\mathcal{M}^v(\bm{x}^v_i,\bm{\delta},\bm{p})$ & Poisoned image; $\bm{p}\sim\mathcal{U}(\text{image region})$ is insertion position \\
$\hat{\bm{x}}^t_i=\mathcal{M}^c(\bm{x}^t_i, t(c^*))$ & Triggered text via concatenation/mixing with target description $t(c^*)$ \\
$t(c^*)$, $f^t\!\big(t(c^*)\big)$ & Target text description for class $c^*$ and its embedding \\
$\oplus$ & Lexicographic concatenation operator (text trigger) \\
$\tilde{\mathcal{S}}$ & Sub-optimized poison subset selected from $\mathcal{D}_0$ \\
$\mathcal{D}_1$ & Attacker’s mixed training set (poisoned + clean) \\
\midrule
$r_t=\tfrac{1}{|\tilde{\mathcal{S}}|}\sum_i\|\hat{\bm{v}}_{i,t}-\bar{\bm{v}}^{\text{trig}}_t\|^2$ & Trigger-cluster radius (dispersion) at step $t$ \\
$\mathcal{R}_{\varepsilon_0}=\{\bm{\Theta}:\ \mathcal{L}_{\mathrm{ALIGN}}(\bm{\Theta})\le\varepsilon_0\}$ & Small-loss/flat region for local analysis \\
\midrule
$\mathcal{L}_{\text{total}}$ & Overall training loss \\
$\mathcal{L}_{\text{bd}}^{\text{trigger}}$,\ \ $\mathcal{L}_{\text{bd}}^{\text{model}}$ & Trigger-optimization loss (subset) / model-optimization loss (backdoor) \\
$\mathcal{L}_{\mathrm{T2T}}$,\ \ $\lambda_{\mathrm{T2T}}$ & Trigger-to-Trigger clustering loss and its weight \\
$\mathcal{L}_{\mathrm{MPE}}$,\ \ $\lambda_{\mathrm{MPE}}$ & Multi-Prototype Enhancement loss (centroid$\to$target) and its weight \\
$\mathcal{L}_{\mathrm{ALIGN}}$,\ \ $\lambda_{\mathrm{ALIGN}}$ & Alignment/flatness control loss and its weight \\
$\mathcal{L}_{\mathrm{EWC}}$,\ \ $\lambda_{\mathrm{EWC}}$ & EWC regularizer and its weight; $F_n$ is Fisher-based importance \\
$\theta$ & Angle between clean and backdoor gradients (parameter space) \\
\midrule
$\nabla_{\bm{\Theta}}\mathcal{L}_{\mathrm{clean}}$,\ \ $\nabla_{\bm{\Theta}}\mathcal{L}_{\mathrm{bd}}$ & Gradients of clean / backdoor objectives w.r.t.\ parameters \\
$u:=\bar{\bm{v}}^*-\bar{\bm{v}}^{\mathrm{neg}}$,\ \ $\Delta:=\bar{\bm{v}}^{\text{trig}}-\bar{\bm{v}}^*$ & Surrogate gradients in embedding space (clean / backdoor) \\
$J$,\ \ $G:=JJ^\top$,\ \ $\kappa(G)$ & Jacobian of representation wrt.\ $\bm{\Theta}$; Gram matrix and its condition number \\
$\langle a,b\rangle_G:=a^\top G b$ & $G$-metric inner product linking parameter/embedding gradients \\
$m:=\|u\|$,\ \ $\varepsilon:=\|\Delta\|$ & Magnitudes in gradient-alignment condition ($m>\kappa(G)\varepsilon$) \\
$\kappa_{\mathrm{bd}}$ & Local Hessian bound for $\mathcal{L}_{\mathrm{total}}(\cdot;\mathcal{D}_{\mathrm{bd}}^{\mathrm{eval}})$ along the update path \\
\midrule
$\eta$,\ \ $\gamma$ & Learning rate (parameters) and inner step size (centroid/embedding updates) \\
\bottomrule
\end{tabularx}
\end{table*}

We summarize the main mathematical notations and variables used in this paper, as shown in \Tref{tab:symbols}. These symbols cover key parts of the BadCLIP++ framework related to poisoned sample construction, trigger optimization, backdoor loss function definition, and model training.

\section{Detailed Algorithm of BadCLIP++}
\label{supp:badclip++}
As shown in \Fref{fig:overview}, the overall training process of BadCLIP++ consists of two main phases: \emph{Data Poisoning} and \emph{Poisoned Model Training}.
The former consists of two sub-phases: trigger optimization and poisoned sample selection, which are used to construct covert and efficiently poisoned samples. 
The latter jointly optimizes the model under controlled conditions to achieve a high attack success rate (ASR) and strong resistance to forgetting.
Alg.~\ref{alg:badclippp} provides the complete training process.

\textbf{Phase 1: Data Poisoning.}
This phase aims to build a high-quality collection of poisoned samples so that the model implicitly learns backdoor associations while maintaining clean accuracy.
The data poisoning process consists of two steps: trigger optimization and sample selection.

\emph{(a) Trigger Optimization.}
In the first step, BadCLIP++ fixes the pre-trained model parameters $\bm{\Theta}_0 = \{\bm{\theta}_0^v, \bm{\theta}_0^t\}$ and optimizes only the visual trigger $\bm{\delta}^v$.
A smaller subset $\mathcal{D}_{\text{opt}}$ is randomly selected from the clean dataset $\mathcal{D}_0$, and the triggers are embedded into the original images using the image injection module $\mathcal{M}^v$ to generate the poisoned images:
\begin{equation}
\hat{\bm{x}}^v_i = \mathcal{M}^v(\bm{x}^v_i, \bm{\delta}^v),
\end{equation}
while assigning a target text description $\bm{x}^{t^*}_i \sim \mathcal{C}$ to each sample.
By minimizing the trigger loss $\mathcal{L}_{\text{bd}}^{\text{trigger}}$ under a fixed model, the trigger parameters are gradually updated to align with the target categories in the vision-language embedding space.
After several rounds of iterations, the optimized visual trigger $\bm{\delta}^{*}$ is obtained, which exhibits strong imperceptibility and transferability.

\emph{(b) Sample Selection (Poisoned Subset Selection).}
After the trigger optimization is complete, BadCLIP++ employs the \emph{Greedy Mean Alignment} strategy to select the subset $\tilde{\mathcal{S}}$ from the clean dataset $\mathcal{D}_0$ that best represents the semantic direction of the target category.
This process aims to achieve the maximum semantic deviation with the minimum poisoning ratio.
Subsequently, the optimized trigger $\bm{\delta}^{*}$ together with the semantic fusion module $\mathcal{M}^c$ is used to generate joint poisoned pairs for the selected samples:
\begin{equation}
(\hat{\bm{x}}^v_i, \hat{\bm{x}}^t_i) = \big(\mathcal{M}^v(\bm{x}^v_i, \bm{\delta}^{*}), \mathcal{M}^c(\bm{x}^t_i)\big),
\end{equation}
forming a collection of samples that are both visually and textually poisoned to support subsequent model optimization.

\textbf{Phase 2: Poisoned Model Training.}
In this phase, BadCLIP++ co-optimizes the model using both the clean dataset $\mathcal{D}_1$ and the poisoned sample set $\tilde{\mathcal{S}}$.
The optimization objective consists of two parts: the clean task loss $\mathcal{L}_{\text{total}}$ and the backdoor regularization loss $\mathcal{L}_{\text{bd}}^{\text{model}}$.
The former maintains the model's performance on normal tasks, and the latter enhances the trigger consistency of the poisoned samples.
The model parameters are updated with the following objective:
\begin{equation}
\min_{\bm{\Theta}}\; \mathcal{L}_{\text{total}} + \mathcal{L}_{\text{bd}}^{\text{model}}.
\end{equation}

\textbf{Additional Phase 2: Training Control.}
To further balance the model's performance between clean and attack objectives, BadCLIP++ employs a controlled training mechanism.
When this mechanism is enabled, the training process dynamically adjusts the weights of the backdoor regularization terms to avoid trigger overfitting and model degradation.
Eventually, the algorithm outputs the optimized trigger $\bm{\delta}^{*}$ and the poisoned model parameters $\bm{\Theta}_*$ to realize a multimodal backdoor attack that is both highly stealthy and highly resistant to forgetting.

\begin{algorithm}[htbp]
\caption{Training Algorithm of BadCLIP++}
\label{alg:badclippp}
\KwIn{Pre-trained model $\bm{\Theta}_0 = \{\bm{\theta}_0^v, \bm{\theta}_0^t\}$; Initialize visual trigger $\bm{\delta}^v$ randomly; Target description set $\mathcal{C}$; Clean dataset $\mathcal{D}_0$; Optimization rounds $E_1, E_2$}
\KwOut{Optimized visual trigger $\bm{\delta}^{*}$; Poisoned model $\bm{\Theta}_*$}
\textbf{Stage I: Trigger Optimization (freeze model)}\;
\For{$e_1 = 1$ \KwTo $E_1$}{
    Generate poisoned images $\hat{\bm{x}}^v_i = \mathcal{M}^v(\bm{x}^v_i, \bm{\delta}^v)$ for all $\bm{x}_i \in \mathcal{D}_{\text{opt}} \subset \mathcal{D}_{0}$ \;    
    \ForEach{$\hat{\bm{x}}^v_i \in \mathcal{D}_{\text{opt}}$}{
        Randomly sample target text $\bm{x}^{t^*}_i \sim \mathcal{C}$ 
    }
    Compute $\mathcal{L}_{\text{bd}}^{\text{trigger}}$ with pre-trained model $\bm{\Theta}_0$ \;
    
    Update visual trigger $\bm{\delta}^v$ by minimizing $\mathcal{L}_{\text{bd}}^{\text{trigger}}$ \;
}
\textbf{Stage II: Poisoned Training (update model)}\;
Select poisoning subset $\tilde{\mathcal{S}}$ from $\mathcal{D}_0$ via greedy mean alignment with pre-trained model $\bm{\Theta}_0$ \;
Construct poisoned samples $\{(\hat{\bm{x}}^v_i, \hat{\bm{x}}^t_i)\}$ using fixed trigger $\bm{\delta}^{*}$ and semantic fusion $\mathcal{M}^c$, for all $\bm{x}_i \in \tilde{\mathcal{S}}$\;

\For{$e_2 = 1$ \KwTo $E_2$}{
    Compute clean loss $\mathcal{L}_{\text{total}}$ on $\mathcal{D}_1$\;
    Compute regularization loss $\mathcal{L}_{\text{bd}}^{\text{model}}$ on $\tilde{\mathcal{S}}$\;
    \If{Controllable Training is enabled}{
        Update model parameters $\bm{\Theta}$ by minimizing $\mathcal{L}_{\text{total}} + \mathcal{L}_{\text{bd}}^{\text{model}}$\;
    }
    \Else{
    Update model parameters $\bm{\Theta}$ by minimizing $\mathcal{L}_{\text{total}}$
    }
}
\Return{$\bm{\delta}^{*}$, $\bm{\Theta}_*$}
\end{algorithm}

\textbf{Implementation Details.}
During the trigger optimization phase (\Eref{eq:trigger-level}), we fix hyperparameters $\lambda_{\text{T2T}}=0.1$ and $\lambda_{\text{MPE}}-1$ to preset values and train QR trigger patterns (consisting of 0 or 1) on a subset of CC3M. 
This subset contains 1{,}900 pairs of target-category banana samples and 10{,}000 pairs of randomly selected category samples. 
Training is conducted using the Adam optimizer with a learning rate of 0.001 and a batch size of 64, iterating for 50 rounds. 
During the backdoor training phase, we select 500{,}000 image-text pairs from CC3M, with 1{,}500 pairs designated as poisoned samples. 
The poisoning samples use embeddings from the clean model as feature extractors. 
A greedy algorithm selects the most suitable instances for injection while fusing the text. 
This phase employs a batch size of 128, a learning rate of 1e-6, and 10 iterations. 
The trigger block size is 16 $\times$ 16, covering approximately 0.5\% of the entire image. 
In the total loss function \Eref{eq:model_level}, $\lambda_{\text{total}}$, $\lambda_{\text{ALIGN}}$, and $\lambda_{\text{EWC}}$ are set to 1, 0.008, and 0.1, respectively. 

\section{Detailed Proofs and Empirical Validation}
\subsection{Proof of Lemma~\ref{lemma:compactness}}
\label{supp:lemma_1}
We provide a detailed proof of the contraction property stated in Lemma~\ref{lemma:compactness}.
\begin{proof}[Proof]
Let $n:=|\mathcal{D}_{\mathrm{opt}}|$, 
$\bar{\bm v}^{\text{trig}}_t=\tfrac{1}{n}\sum_{i=1}^n \hat{\bm v}_{i,t}$, and
$\bm d_{i,t}:=\hat{\bm v}_{i,t}-\bar{\bm v}^{\text{trig}}_t$.  
Define the cluster radius
\begin{equation}
r_t = \max_i \|\bm d_{i,t}\|.
\end{equation}

As in the main text, we absorb the factor $1/n$ into $\lambda_{\mathrm{T2T}}$, so that
\begin{equation}
\nabla_{\hat{\bm v}_i}\mathcal{L}_{\mathrm{T2T}}
= 2\bigl(\hat{\bm v}_{i,t}-\bar{\bm v}^{\text{trig}}_t\bigr)
= 2\,\bm d_{i,t}.
\end{equation}

Therefore, the simultaneous GD step in the embedding space gives
\begin{equation}
\begin{aligned}
\hat{\bm v}_{i,t+1}
&= \hat{\bm v}_{i,t} - \gamma\,\lambda_{\mathrm{T2T}}\,
   \nabla_{\hat{\bm v}_i}\mathcal{L}_{\mathrm{T2T}} \\
&= \hat{\bm v}_{i,t} - 2\gamma\,\lambda_{\mathrm{T2T}}
   \bigl(\hat{\bm v}_{i,t}-\bar{\bm v}^{\text{trig}}_t\bigr).
\end{aligned}
\end{equation}

Averaging over $i$ yields centroid invariance:
\begin{equation}
\begin{aligned}
\bar{\bm v}^{\text{trig}}_{t+1}
&= \tfrac{1}{n}\sum_i \hat{\bm v}_{i,t+1} \\
&= \bar{\bm v}^{\text{trig}}_t - \tfrac{2\gamma\lambda_{\mathrm{T2T}}}{n}
   \sum_i\bigl(\hat{\bm v}_{i,t}-\bar{\bm v}^{\text{trig}}_t\bigr) \\
&= \bar{\bm v}^{\text{trig}}_t,
\end{aligned}
\end{equation}
since $\sum_i(\hat{\bm v}_{i,t}-\bar{\bm v}^{\text{trig}}_t)=\bm 0$.

Hence, the centered recursion is exact:
\begin{equation}
\bm d_{i,t+1}
= \hat{\bm v}_{i,t+1}-\bar{\bm v}^{\text{trig}}_{t+1}
= \bigl(1-2\gamma\lambda_{\mathrm{T2T}}\bigr)\,\bm d_{i,t}.
\end{equation}

Taking norms and maximization over $i$ gives
\begin{equation}
\begin{aligned}
r_{t+1}
&= \max_i \|\bm d_{i,t+1}\| \\
&= \bigl|1-2\gamma\lambda_{\mathrm{T2T}}\bigr|
   \max_i \|\bm d_{i,t}\| \\
&\le (1-2\gamma\lambda_{\mathrm{T2T}})\,r_t,
\end{aligned}
\end{equation}
whenever $0<\gamma<\tfrac{1}{2\lambda_{\mathrm{T2T}}}$ so that the factor lies in $(0,1)$.  
Iterating proves geometric contraction and $\lim_{t\to\infty} r_t=0$.
\end{proof}

\subsection{Proof of Lemma~\ref{lemma:centroid-alignment}}
\label{supp:lemma_2}
We provide a detailed proof of the alignment property stated in Lemma~\ref{lemma:centroid-alignment}. 

\begin{assumption}[\textbf{MPE conditions}]
\label{ass:mpe}
\emph{(i) Quadratic lower bound on the centroid deviation.}
There exists $\lambda_{\mathrm{MPE}}>0$ such that for all $t$,
\begin{equation}
\label{eq:mpe-lb}
    \mathcal{L}_{\mathrm{MPE}}(t)
    \ \ge\ \frac{\lambda_{\mathrm{MPE}}}{2}\,
    \bigl\| \bar{\bm{v}}^{\mathrm{trig}}_{t} - \bar{\bm{v}}^{*} \bigr\|^2.
\end{equation}

\noindent\emph{(ii) Joint optimization dynamics.}
The trigger embeddings are updated by gradient descent on the joint objective
\begin{equation}
\label{eq:joint}
\mathcal{L}_{\mathrm{joint}}(t)\ =\ \mathcal{L}_{\mathrm{T2T}}(t)\ +\ \mathcal{L}_{\mathrm{MPE}}(t),
\end{equation}
with a stable step size $\eta>0$, under standard smoothness and coercivity assumptions ensuring convergence of objective values.

\noindent\emph{(iii) Attainability or asymptotic vanishing of MPE.}
Either the minimum of $\mathcal{L}_{\mathrm{joint}}$ is attained at some point with $\mathcal{L}_{\mathrm{MPE}}=0$, or 
$\limsup_{t\to\infty}\mathcal{L}_{\mathrm{MPE}}(t)=0$.
\end{assumption}

\begin{proof}[Proof]
By Assumption~\ref{ass:mpe}, the joint optimization yields a sequence satisfying
\begin{equation}
\label{eq:mpe-small}
\forall\,\varepsilon>0\ \ \exists\,T<\infty\ \ \text{s.t.}\quad
\mathcal{L}_{\mathrm{MPE}}(t)\ \le\ \tfrac{\lambda_{\mathrm{MPE}}}{2}\,\varepsilon,
\ \ \forall t\ge T.
\end{equation}
Combining \eqref{eq:mpe-lb} and \eqref{eq:mpe-small} gives
\begin{equation}
\bigl\| \bar{\bm{v}}^{\mathrm{trig}}_{t} - \bar{\bm{v}}^{*} \bigr\|^2 
\ \le\ \varepsilon,
\qquad \forall t \ge T,
\end{equation}
hence
\begin{equation}
\label{eq:centroid-ball}
\bigl\| \bar{\bm{v}}^{\mathrm{trig}}_{t} - \bar{\bm{v}}^{*} \bigr\|
\ \le\ \sqrt{\varepsilon},
\qquad \forall t \ge T.
\end{equation}
This establishes Lemma~\ref{lemma:centroid-alignment}.
\end{proof}

\subsection{Proof of Theorem~\ref{theorem:gradient}}
\label{supp:theorem 1}
\begin{proof}
Let 
$u := \bar{\bm v}^{*} - \bar{\bm v}^{\mathrm{neg}}$ 
and 
$\Delta := \bar{\bm v}^{\mathrm{trig}} - \bar{\bm v}^{*}$.
By the definition of the embedding-level surrogate gradients,
\begin{equation}
   \widetilde{\nabla}_{\bm z}\mathcal{L}_{\mathrm{clean}} = u,
\qquad
\widetilde{\nabla}_{\bm z}\mathcal{L}_{\mathrm{bd}} = u+\Delta, 
\end{equation}
and let $m:=\|u\|$ be the cross-class margin.

\textbf{Step 1 (Using Lemma~\ref{lemma:compactness}).}
Lemma~\ref{lemma:compactness} states that the within-class scatter of both the clean target
and the triggered set of contracts along with training. Denote their population centroids by
$\bm\mu^{*}$ and $\bm\mu^{\mathrm{trig}}$, and the empirical batch means by
$\bar{\bm v}^{*}$ and $\bar{\bm v}^{\mathrm{trig}}$, respectively.
Let $r_t^{*}$ and $r_t^{\mathrm{trig}}$ be the (max or RMS) cluster radii at iteration $t$.
By Lemma~\ref{lemma:compactness}, there exist constants $\rho\in(0,1)$ and $C>0$ such that
$r_{t+1}^{(\cdot)} \le \rho\, r_t^{(\cdot)}$ and hence $r_t^{(\cdot)} \le C\rho^t r_0^{(\cdot)}$.
Standard concentration for mean embeddings in a ball of radius $r_t^{(\cdot)}$ yields
\begin{equation}
\|\bar{\bm v}^{\mathrm{trig}} - \bm\mu^{\mathrm{trig}}\|
\ \le\ \varepsilon_{\mathrm{est,trig}}(t),\qquad
\|\bar{\bm v}^{*} - \bm\mu^{*}\|
\ \le\ \varepsilon_{\mathrm{est,clean}}(t),
\end{equation}
with $\varepsilon_{\mathrm{est,trig}}(t),\ \varepsilon_{\mathrm{est,clean}}(t)\to 0$ as $t\to\infty$
(due to the shrinkage of radii and fixed batch size). Define
\begin{equation}
\varepsilon_{\mathrm{est}}(t)
:= \varepsilon_{\mathrm{est,trig}}(t)+\varepsilon_{\mathrm{est,clean}}(t).
\end{equation}

\textbf{Step 2 (Using Lemma~\ref{lemma:centroid-alignment}).}
Lemma~\ref{lemma:centroid-alignment} guarantees that the population centroids
approach each other under the aligning objective, i.e.,
\begin{equation}
\|\bm\mu^{\mathrm{trig}} - \bm\mu^{*}\|
\ \le\ \varepsilon_{\mathrm{align}}(t),
\qquad
\varepsilon_{\mathrm{align}}(t)\to 0 \ \ \text{as } t\to\infty.
\end{equation}
Combining the decomposition
\begin{equation}
\Delta
= \bar{\bm v}^{\mathrm{trig}} - \bar{\bm v}^{*}
= \underbrace{(\bm\mu^{\mathrm{trig}} - \bm\mu^{*})}_{\text{population gap}}
+\underbrace{\big(\bar{\bm v}^{\mathrm{trig}}-\bm\mu^{\mathrm{trig}}\big)
-\big(\bar{\bm v}^{*}-\bm\mu^{*}\big)}_{\text{estimation errors}},
\end{equation}
With the triangle inequality, we obtain
\begin{equation}
\label{eq:Delta-bound}
\|\Delta\|
\ \le\ \varepsilon_{\mathrm{align}}(t)\ +\ \varepsilon_{\mathrm{est}}(t)
\ :=\ \varepsilon(t).
\end{equation}
By the two lemmas, for sufficiently large $t$ we may ensure $\varepsilon(t)\le \varepsilon$ for a prescribed tolerance $\varepsilon>0$.

\textbf{Step 3 (Non-opposing directions).}
Using the exact identity
\(
\big\langle \widetilde{\nabla}_{\bm z}\mathcal{L}_{\mathrm{clean}},
\widetilde{\nabla}_{\bm z}\mathcal{L}_{\mathrm{bd}}\big\rangle
= \langle u, u+\Delta\rangle
= \|u\|^2 + \langle u,\Delta\rangle,
\)
and Cauchy--Schwarz $\langle u,\Delta\rangle \ge -\|u\|\,\|\Delta\|$, we have
\begin{equation}
\label{eq:key-lower}
\big\langle \widetilde{\nabla}_{\bm z}\mathcal{L}_{\mathrm{clean}},
\widetilde{\nabla}_{\bm z}\mathcal{L}_{\mathrm{bd}}\big\rangle
\ \ge\ \|u\|^2 - \|u\|\,\|\Delta\|
\ =\ m\,(m-\|\Delta\|).
\end{equation}
Invoking the bound \eqref{eq:Delta-bound} and the margin assumption $m>\varepsilon$, we conclude
\begin{equation}
\big\langle \widetilde{\nabla}_{\bm z}\mathcal{L}_{\mathrm{clean}},
\widetilde{\nabla}_{\bm z}\mathcal{L}_{\mathrm{bd}}\big\rangle
\ \ge\ m\,(m-\varepsilon)\ \ge\ 0,
\end{equation}
i.e., the two surrogate gradients are non-opposing (their angle does not exceed $90^\circ$).

\textbf{Step 4 (Transfer to parameter space).}
Let $J_{\bm{\Theta}}$ be the batch-aggregated Jacobian that maps parameter perturbations to embedding changes, and define $G:=J_{\bm{\Theta}} J_{\bm{\Theta}}^\top\succeq 0$.
By the chain rule,
\begin{equation}
\begin{aligned}
\nabla_{\Theta}\mathcal{L}
  &= J_{\bm{\Theta}}^\top \,\widetilde{\nabla}_{\bm z}\mathcal{L} \\
  &\Rightarrow
  \big\langle \nabla_{\Theta}\mathcal{L}_{\mathrm{clean}},
            \nabla_{\Theta}\mathcal{L}_{\mathrm{bd}}\big\rangle
   = \big\langle
       \widetilde{\nabla}_{\bm z}\mathcal{L}_{\mathrm{clean}},
       \widetilde{\nabla}_{\bm z}\mathcal{L}_{\mathrm{bd}}
     \big\rangle_{G} \\
  &= u^\top G (u+\Delta).
\end{aligned}
\end{equation}
Assume $G$ admits spectral bounds $\sigma_{\min} I\preceq G\preceq \sigma_{\max} I$ with condition number $\kappa(G):=\sigma_{\max}/\sigma_{\min}$.
Then
\begin{equation}
\begin{aligned}
u^\top G u
&\ge \sigma_{\min}\,\|u\|^2
 = \sigma_{\min}\,m^2,\\
u^\top G\Delta
&\ge -\,\|u\|\,\|G\Delta\|
 \ge -\,\sigma_{\max}\,m\,\|\Delta\|.
\end{aligned}
\end{equation}
Therefore,
\begin{equation}
\label{eq:param-lb}
\begin{aligned}
\big\langle \nabla_{\bm{\Theta}}\mathcal{L}_{\mathrm{clean}},
          \nabla_{\bm{\Theta}}\mathcal{L}_{\mathrm{bd}}\big\rangle
&\ge \sigma_{\min} m^2
   - \sigma_{\max} m\,\|\Delta\| \\
&\ge m\bigl(\sigma_{\min} m - \sigma_{\max}\varepsilon\bigr).
\end{aligned}
\end{equation}
where we used $\|\Delta\|\le\varepsilon$ from \eqref{eq:Delta-bound}.
Hence, if
\[
m\ >\ \kappa(G)\,\varepsilon \qquad\bigl(\text{equivalently } \sigma_{\min} m-\sigma_{\max}\varepsilon>0\bigr),
\]
then the parameter-space inner product is non-negative:
\begin{equation}
\big\langle \nabla_{\bm{\Theta}}\mathcal{L}_{\mathrm{clean}},\nabla_{\bm{\Theta}}\mathcal{L}_{\mathrm{bd}}\big\rangle \ \ge\ 0.
\end{equation}
Under near-isotropy ($G\approx cI$, so $\kappa(G)\approx 1$), this condition reduces to $m>\varepsilon$.
\medskip
This completes the proof.
\end{proof}

\subsection{Proof of Lemma \ref{lem:align-flat}}
\label{supp:lemma 3}
\begin{proof}[Proof]
Write the ALIGN loss as a squared-residual objective
\begin{equation}
\mathcal L_{\mathrm{ALIGN}}(\bm{\Theta})=\frac{1}{2}\sum_{k=1}^{K} r_k(\bm{\Theta})^2,
\end{equation}
where $r_k$ are small alignment residuals. A standard calculation gives
\begin{equation}
\label{eq:align-hess}
\begin{aligned}
\nabla_{\bm{\Theta}}^{2}\mathcal L_{\mathrm{ALIGN}}(\bm{\Theta})
&= \sum_{k=1}^{K} \Big(
     \nabla r_k(\bm{\Theta})\,\nabla r_k(\bm{\Theta})^\top \\
&
     +\; r_k(\bm{\Theta})\,\nabla^2 r_k(\bm{\Theta})
   \Big).
\end{aligned}
\end{equation}

By Assumption~\textup{(1)}, on 
$\mathcal R_{\varepsilon_0}=\{\bm{\Theta}:\ \mathcal L_{\mathrm{ALIGN}}(\bm{\Theta})\le \varepsilon_0\}$
the residuals satisfy the local self-bounding property \eqref{eq:self-bounding-main}:
\[
\|\nabla r_k(\bm{\Theta})\|\ \le\ B_1\,|r_k(\bm{\Theta})|,
\qquad
\|\nabla^2 r_k(\bm{\Theta})\|\ \le\ B_2\,|r_k(\bm{\Theta})|.
\]

Hence,
\begin{equation}
\label{eq:align-two-terms}
\begin{aligned}
\lambda_{\max}\!\Big(\sum_k \nabla r_k\,\nabla r_k^\top\Big)
&\le \sum_k \|\nabla r_k\|^2 \\
&\le B_1^2\sum_k r_k^2 \\
&= 2B_1^2\,\mathcal L_{\mathrm{ALIGN}}(\bm{\Theta}), \\[0.5em]
\Big\|\sum_k r_k\,\nabla^2 r_k\Big\|
&\le \sum_k |r_k|\,\|\nabla^2 r_k\| \\
&\le B_2\sum_k r_k^2 \\
&= 2B_2\,\mathcal L_{\mathrm{ALIGN}}(\bm{\Theta}).
\end{aligned}
\end{equation}
Therefore,
\begin{equation}
\label{eq:align-hess-final}
\begin{aligned}
& \lambda_{\max}\!\bigl(\nabla_{\bm{\Theta}}^{2}\mathcal L_{\mathrm{ALIGN}}(\bm{\Theta})\bigr)
= \lambda_{\max}\!\Big(\sum_k \nabla r_k\nabla r_k^\top
   + \sum_k r_k\,\nabla^2 r_k\Big) \\[0.3em]
&\le \lambda_{\max}\!\Big(\sum_k \nabla r_k\nabla r_k^\top\Big)
   + \Big\|\sum_k r_k\,\nabla^2 r_k\Big\| \\[0.3em]
&\le 2(B_1^2+B_2)\,\mathcal L_{\mathrm{ALIGN}}(\bm{\Theta}) \\[0.3em]
&=: C\,\mathcal L_{\mathrm{ALIGN}}(\bm{\Theta}).
\end{aligned}
\end{equation}
which proves the lemma.
\end{proof}

\subsection{Proof of Theorem \ref{thm:local-stab}}
\label{supp:theorem 2}
\begin{proof}[Proof]
Let $\bm{\Theta}_\ast\in\mathcal R_{\varepsilon_0}$ and $\Delta\bm{\Theta}$ be such that 
$\bm{\Theta}_\ast+s\,\Delta\bm{\Theta}\in\mathcal R_{\varepsilon_0}$ for all $s\in[0,1]$. 
By the mean value theorem in integral form,
\begin{equation}
\begin{aligned}
\nabla_{\bm{\Theta}}\mathcal L_{\mathrm{ALIGN}}(\bm{\Theta}_\ast+\Delta\bm{\Theta})
-\nabla_{\bm{\Theta}}\mathcal L_{\mathrm{ALIGN}}(\bm{\Theta}_\ast)
& \\ =\Bigg(\int_{0}^{1}\nabla_{\bm{\Theta}}^{2}\mathcal L_{\mathrm{ALIGN}}(\bm{\Theta}_\ast+s\Delta\bm{\Theta})\,dt\Bigg)\Delta\bm{\Theta}.
\end{aligned}
\end{equation}
Hence,
\begin{equation}
\begin{aligned}
\bigl\|\nabla_{\bm{\Theta}}\mathcal L_{\mathrm{ALIGN}}(\bm{\Theta}_\ast+\Delta\bm{\Theta})
-\nabla_{\bm{\Theta}}\mathcal L_{\mathrm{ALIGN}}(\bm{\Theta}_\ast)\bigr\|
& \\ \le \sup_{t\in[0,1]}
\lambda_{\max}\!\bigl(\nabla_{\bm{\Theta}}^{2}\mathcal L_{\mathrm{ALIGN}}(\bm{\Theta}_\ast+s\Delta\bm{\Theta})\bigr)\,\|\Delta\bm{\Theta}\|.
\end{aligned}
\end{equation}
Using Lemma~\ref{lem:align-flat} and that $\bm{\Theta}_\ast+s\Delta\bm{\Theta}\in\mathcal R_{\varepsilon_0}$,
\begin{equation}
\begin{aligned}
\lambda_{\max}\!\bigl(\nabla_{\bm{\Theta}}^{2}\mathcal L_{\mathrm{ALIGN}}(\bm{\Theta}_\ast+s\Delta\bm{\Theta})\bigr)
& \\ \le C\,\mathcal L_{\mathrm{ALIGN}}(\bm{\Theta}_\ast+s\Delta\bm{\Theta})
\ \le\ C\,\varepsilon_0 ,
\end{aligned}
\end{equation}
which yields the gradient Lipschitz bound
\begin{equation}
\begin{aligned}
\bigl\|\nabla_{\bm{\Theta}}\mathcal L_{\mathrm{ALIGN}}(\bm{\Theta}_\ast+\Delta\bm{\Theta})
-\nabla_{\bm{\Theta}}\mathcal L_{\mathrm{ALIGN}}(\bm{\Theta}_\ast)\bigr\|
&  \le C\,\varepsilon_0\,\|\Delta\bm{\Theta}\|.
\end{aligned}
\end{equation}

For the second-order remainder, Taylor's theorem with integral remainder gives
\begin{equation}
\begin{aligned}
&\mathcal L_{\mathrm{ALIGN}}(\bm{\Theta}_\ast+\Delta\bm{\Theta})
\\
&=\mathcal L_{\mathrm{ALIGN}}(\bm{\Theta}_\ast)
+\langle\nabla_{\bm{\Theta}}\mathcal L_{\mathrm{ALIGN}}(\bm{\Theta}_\ast),\Delta\bm{\Theta}\rangle   +\\
&\tfrac{1}{2}\,\Delta\bm{\Theta}^\top
\Bigg(\int_{0}^{1}(1-t)\,\nabla_{\bm{\Theta}}^{2}\mathcal L_{\mathrm{ALIGN}}(\bm{\Theta}_\ast+s\Delta\bm{\Theta})\,dt\Bigg)\Delta\bm{\Theta}.
\end{aligned}
\end{equation}
Taking absolute values and operator norms and using the same bound as above,
\begin{equation}
\begin{aligned}
&\bigl|\mathcal L_{\mathrm{ALIGN}}(\bm{\Theta}_\ast+\Delta\bm{\Theta})
-\mathcal L_{\mathrm{ALIGN}}(\bm{\Theta}_\ast)
\\
& -\langle\nabla_{\bm{\Theta}}\mathcal L_{\mathrm{ALIGN}}(\bm{\Theta}_\ast),\Delta\bm{\Theta}\rangle\bigr|  \le \tfrac12\,C\,\varepsilon_0\,\|\Delta\bm{\Theta}\|^2 .
\end{aligned}
\end{equation}
This proves both claims.
\end{proof}

\subsection{Proof of Forgetting Upper Bound}
\label{supp:bound}
\begin{proof}
Let 
\(
H(s):=\nabla^2_{\bm{\Theta}}\,
\mathcal{L}_{\mathrm{total}}\!\big(\bm{\Theta}_\ast+s\,\Delta\bm{\Theta};\mathcal{D}_{\mathrm{bd}}^{\mathrm{eval}}\big).
\)
By Taylor’s theorem with integral remainder,
\begin{equation}
    \begin{aligned}
\Delta\mathcal{L}_{\mathrm{bd}}
&= \big\langle \nabla_{\bm{\Theta}}\mathcal{L}_{\mathrm{total}}(\bm{\Theta}_\ast;\mathcal{D}_{\mathrm{bd}}^{\mathrm{eval}}),\,\Delta\bm{\Theta}\big\rangle
    \\
& + \int_0^1 (1-s)\,\Delta\bm{\Theta}^\top H(s)\,\Delta\bm{\Theta}\,ds\\
&\le \langle \bm{g}_{\mathrm{bd}},\,\Delta\bm{\Theta}\rangle
   + \|\Delta\bm{\Theta}\|^2 \int_0^1 (1-s)\,\lambda_{\max}\!\big(H(s)\big)\,ds \label{eq:lambda-max} \\
&\le \langle \bm{g}_{\mathrm{bd}},\,\Delta\bm{\Theta}\rangle
   + \tfrac{1}{2}\,\kappa_{\mathrm{bd}}\,\|\Delta\bm{\Theta}\|^2,
\end{aligned}
\end{equation}
where the last step uses the local smoothness bound \eqref{eq:lbd-smooth}.
Substituting the clean-only step \eqref{eq:ft-step}, i.e.\ $\Delta\bm{\Theta}=-\eta\,\bm{g}_{\mathrm{ft}}$, yields
\begin{equation}
    \label{eq:U-angle}
\begin{aligned}
\Delta\mathcal{L}_{\mathrm{bd}}
&\le -\,\eta\,\langle \bm{g}_{\mathrm{bd}},\,\bm{g}_{\mathrm{ft}}\rangle
    + \tfrac{1}{2}\,\kappa_{\mathrm{bd}}\,\eta^2\,\|\bm{g}_{\mathrm{ft}}\|^2 \\
&=   -\,\eta\,\|\bm{g}_{\mathrm{bd}}\|\,\|\bm{g}_{\mathrm{ft}}\|\,\cos\theta
    + \tfrac{1}{2}\,\kappa_{\mathrm{bd}}\,\eta^2\,\|\bm{g}_{\mathrm{ft}}\|^2, 
\end{aligned}
\end{equation}
which matches \eqref{eq:forgetting-upper}–\eqref{eq:forgetting-upper-angle}.
\end{proof}

\subsection{Empirical Validation}
\label{sec:Empirical Validation}
\textbf{Experimental setup}.
To further validate the theoretical derivations in the main paper, we design three complementary sets of empirical experiments corresponding to 
Lemma~\ref{lemma:compactness}, Lemma~\ref{lemma:centroid-alignment}, and Theorem~\ref{thm:local-stab}. 
(1) For Lemma~\ref{lemma:compactness} and Lemma~\ref{lemma:centroid-alignment}, we extract the embeddings of the triggered samples under fixed parameter snapshots, 
compute the Euclidean distances between the centers of the clusters and the center of the target category, 
and measure the cluster radii, respectively. 
The shrinkage trend of the radius ($r_t$) is used to check whether the trigger embeddings gradually converge to the target manifold 
and maintain compactness during the optimization process. 
(2) For Theorem~\ref{thm:local-stab}, we use the Hessian–Vector Product (HVP) combined with power iteration to estimate the maximum eigenvalue 
of the Hessian of $\mathcal{L}_{\mathrm{bd}}$, the backdoor loss. 
The maximum eigenvalue of the backdoor loss ($\lambda_{\max}$) is estimated, 
and the local loss landscape is plotted in a random two-dimensional subspace to visualize the curvature changes. 
All experiments are conducted under the same evaluation protocol with identical random seeds and sampling sizes 
to ensure the reproducibility and comparability of the results. 
These experiments verify the convergence, alignment, and controlled local curvature of the triggered embeddings 
from the first-order and second-order optimization perspectives, respectively, 
and provide direct empirical support for the theoretical analysis.

\textbf{Compactness and distance}.
\Fref{fig:tsne_pca} illustrates the visualization results based on t-SNE and PCA, corresponding to the validation of 
Lemma~\ref{lemma:compactness} and Lemma~\ref{lemma:centroid-alignment}. 
The t-SNE (left panel) depicts the distribution of three types of embeddings, namely, Clean, BadCLIP, and BadCLIP++, 
where \textit{ellipse} denotes the area of the Gaussian isodensity ellipse, 
and \textit{hull} denotes the area of the smallest convex hull, which together measure the compactness of the cluster. 
The results show that the ellipse area of Clean is 8447.60, and the hull is 327.36; 
the ellipse of BadCLIP is 4586.03, and the hull is 397.15; 
and the ellipse of BadCLIP++ is further decreased to 3299.51, with the hull contracted to 299.41. 
This indicates that the triggered embeddings shrink significantly in the BadCLIP++ framework, 
forming high-density clusters, which is consistent with the decreasing-radius law described in 
Lemma~\ref{lemma:compactness}.

PCA (right panel) shows the relative positions of the cluster centers $\mu$ in the clean subspace with the clean $\mu$ (Banana category) as the benchmark. 
It can be observed that the $\mu$ of BadCLIP++ is closer to the clean $\mu$, while there is still a significant offset between BadCLIP and the clean distribution, 
suggesting that BadCLIP++ is effective in pulling the triggered cluster centers of mass back to the central region of the target manifolds, 
which is consistent with the ``center of mass alignment'' described in Lemma~\ref{lemma:centroid-alignment}. 
This phenomenon confirms empirically the role of the T2T loss in intra-cluster compression and the pulling effect of the MPE-soft loss in cross-cluster alignment, 
which explains why the trigger representations of BadCLIP++ are more robust and controllable in high-dimensional space.

\begin{table}[t]
\centering
\small
\caption{Gradient alignment between clean and backdoor batches in the aligned regime. 
$P_{+}=\mathbb{P}[\cos\theta>0]$. Asterisk($^*$) denotes the clean-only fine-tuned variant (CleanCLIP).}
\renewcommand{\arraystretch}{1.05}
\setlength{\tabcolsep}{9pt}
\resizebox{\linewidth}{!}{
\begin{tabular}{lccccc}

\toprule
\textbf{Setting} & \textbf{mean} & \textbf{$q_{25}$} & \textbf{median} & \textbf{$q_{75}$} & \textbf{$P_{+}$} \\
\midrule
BadCLIP++            & 0.0349 & -0.1069 & 0.0522 & 0.1738 & 0.750 \\
$\text{BadCLIP++}^*$ & 0.0449 & -0.0430 & 0.0583 & 0.1212 & 0.922 \\
BadCLIP              & -0.0014 & -0.1155 & -0.0020 & 0.1061 & 0.422 \\
$\text{BadCLIP}^*$   & -0.0080 & -0.1115 & -0.0064 & 0.1687 & 0.391 \\
\bottomrule
\end{tabular}}
\label{tab:codir}
\end{table}
\textbf{Gradient-direction alignment}.
\Tref{tab:codir} reports the statistics on gradient–direction consistency used to validate Theorem \ref{theorem:gradient}.  
Under the same parameter snapshot $\Theta$, we computed gradients for 100{,}000 clean samples and their paired backdoor samples and measured their angular difference by \(\cos\theta\).
Only iterations satisfying the trigger-contraction and centroid-alignment condition 
$(t \ge T,\;\|\bar{\bm v}^{\mathrm{trig}}-\bar{\bm v}^{*}\|\le \sqrt{\varepsilon})$ were included.  

In this region, BadCLIP++ achieved a mean cosine of $0.0349$, a median of $0.0522$, and showed positive angles in roughly $75\%$ of batches.  
For the variant fine-tuned solely on clean samples (BadCLIP++\*), the mean rose to $0.0449$, the median to $0.0583$, and the positive proportion $P_{+}$ reached $0.922$.  
By contrast, both BadCLIP and its cleanly fine-tuned counterpart exhibited near-zero averages and only about $0.4$ positive proportion.  
Although the lower quantile contains slight negative values (mainly from batches with poorly separated minority classes), the pronounced rightward shift of the distribution indicates that once trigger embeddings contract and approach the target class, the clean and backdoor gradient directions tend to align.
This alignment explains why clean fine-tuning introduces almost no disruption to the backdoor.

\textbf{Second-order curvature control}.
\Fref{fig:3d} illustrates the local loss landscape and principal curvature comparison results for the three models (BadCLIP++, BadCLIP, and BadCLIP w/TI), corresponding to the empirical validation of Theorem~\ref{thm:local-stab}. 
We fix the parameter snapshots with the evaluation samples, compute the maximum eigenvalues of the Hessian of $\mathcal{L}_{\mathrm{bd}}$ using HVP with power iteration, 
and plot the 3D loss surface on the local 2D random subspace. 
The results show that the top-1 Hessian eigenvalue of BadCLIP++ is $1.60\times10^{5}$, which is significantly lower than that of BadCLIP ($3.27\times10^{5}$), 
and the curvature of BadCLIP is further enlarged to $4.23\times10^{5}$ after the introduction of Trigger Injection (TI). 
From the 3D visualization, the locally optimized basins of BadCLIP++ are flatter with controlled curvature, 
while BadCLIP and BadCLIP w/TI show obvious sharp peaks and valley structures. 
Combining the theoretical conclusions of Lemma~\ref{lem:align-flat} and Theorem~\ref{thm:local-stab}, 
it can be inferred that the second-order term $\nabla^{2}\mathcal{L}_{\mathrm{bd}}$ of BadCLIP++ is subject to stronger constraints, 
thus effectively extending the range of stable steps in the clean fine-tuning phase. 
This means that BadCLIP++ not only realizes the aggregation and alignment of the trigger embeddings at the first-order level, 
but also significantly smooths the optimization geometry at the second-order level, 
making the model less likely to forget the backdoor features during fine-tuning and more robust and stable overall.

\section{Details of the Backdoor Attacks Methods}
\label{supp:attack}
This section briefly describes the construction of each attack method on the visual and textual side, focusing on the types of triggers, injection locations, and blending strategies, and how these designs affect steganography and robustness. The visual side will cover common approaches such as local patching, holistic blending, spatial distortion, and steganographic injection with key hyperparameters, while the textual side describes the uniform replacement of the ``A photo of banana'' setup with the varied descriptions of BadCLIP/BadCLIP++ or random blending strategies with the original text mixing strategy.

\textbf{BadNets.} 
In the BadNets attack, triggers are superimposed directly on the image corners in the form of fixed Gaussian patches. The implementation accomplishes this by overlaying a square patch on the tensor slice, commonly set to $16\times16$, positioned in a randomly sampled region. Because the trigger is not a transparent mixing operation, its artifacts are more obvious, but the training stage is easy to converge. The textual side is uniformly replaced with ``A photo of banana''.

\textbf{Blended.} 
The Blended attack employs a whole-image linear blending strategy that blends the original image with random noise or a fixed texture weighted by a factor $\alpha$ (typically $0.2$) to produce visually smooth perturbations of low salience. The trigger covers a large but imperceptible area. The textual description maintains the uniform template ``A photo of banana''.

\textbf{WaNet.} 
WaNet attacks do not rely on explicit patches, but instead introduce irreversible distortions through spatial transformations. The implementation employs \texttt{grid\_sample} to nonlinearly deform the entire image, using a randomly generated displacement grid to control local spatial offsets. The perturbation preserves the original content structure at the macroscopic level, but destroys feature consistency at the pixel level. The textual description remains the uniform template.

\textbf{SIG.} 
The SIG attack superimposes a sinusoidal stripe pattern over the entire image. The code creates periodic interference visible to the naked eye by generating $\sin(2\pi f j/224)$-type interference signals at each pixel column and adding them to the original image. The method is simple to implement, but the trigger pattern is easy to detect. The textual side maintains the generic settings.

\textbf{SSBA.} 
The SSBA attack uses the steganographic network HideNet to embed implicit information into the image pixel distribution. The encoding network performs light residual injection on the RGB channels, leaving the visualization almost unchanged; the training phase ensures that the trigger message is decodable via RevealNet. The trigger is highly covert on natural images. The textual description still uses fixed phrases.

\textbf{TrojVQA.} 
This method focuses on visual quizzing tasks, injecting small icons or logo textures of specific shapes on the image side. The implementation induces the model to produce a target answer for that particular symbol by loading a predefined template patch and overlaying it in the specified corner; together with a fixed question text. The textual component remains a uniform template.

\textbf{MABA.} 
MABA employs semi-transparent geometric symbols (\eg, yellow dots) as triggers on the image side, with a fusion ratio of about $0.1$--$0.2$. Trigger positions are computed by an interpretable algorithm to ensure that multimodal features can be stably captured. The textual description is consistent with other methods without semantic perturbation.

\textbf{VLTrojan.} 
VLTrojan is poisoned at the level of visual features and is implemented along the lines of injecting pseudo-marker embeddings in the feature space. Image triggers are realized by locally highlighting graphics, usually in the lower left corner; the corresponding text follows attack styles to maintain cross-modal alignment. This approach is more covert at the representation space level.

\textbf{BadEncoder.} 
The BadEncoder attack implements trigger binding by adding uniform noise patches in the pre-training encoder phase, and the image injection phase directly generates random noise patches to cover localized regions. It is similar to BadNets, but the injection occurs before feature extraction, allowing subsequent tasks to inherit potential backdoors. The textual description is unified as the ``A photo of banana'' prompt.

\textbf{INACTIVE.} 
The method works by inserting learnable triggers into the image input front-end, where the trigger portion is controlled by a parameterized mask in terms of shape and position, and its embedding direction is optimized in the training phase to induce the target output. The image perturbation is weak and imperceptible, and the triggering effect is stable. The textual side maintains a generic template.

\textbf{mmPoison.} 
mmPoison uses the cat class as the poisoning source and the banana class as the attack target. In the image modality, a local trigger texture is embedded into the cat image, and its label is modified to banana to achieve backdoor injection.

\textbf{BadCLIP.} 
BadCLIP uses the same optimized patch injection on the image side, positioned in the center to prevent cropping-based defenses. Unlike the previous approach, instead of using a fixed prompt, the textual side constructs multiple ``banana''-related descriptions (\eg, ``a ripe banana'', ``a banana on a plate'', ``yellow banana fruit''), thus inducing the textual subspace of CLIP to focus on a specific semantic direction.

\textbf{BadCLIP++.}
BadCLIP++ performs higher-order covert blending on both the image and textual side. On the visual side, its triggers utilize high-resolution QR-code patterns, employing optimized loss with positions randomly superimposed on the image. On the textual side, the method randomly mixes banana-like text with original image descriptions using randomized position sampling, thereby mixing clean text with selected target semantic text. This makes it difficult for the model to distinguish the source of the pseudo-label. 

\begin{table*}[ht]
\centering
\caption{Comparison of backdoor attack methods across multiple dimensions. 
Symbols: \ding{51} = Yes, \ding{55} = No, $\sim$ = Partial. BadCLIP++ is the only method combining both text and visual stealthiness with high robustness.}
\renewcommand{\arraystretch}{1.05}
\setlength{\tabcolsep}{2pt}
\resizebox{\linewidth}{!}{
\begin{tabular}{llcccccc}
\toprule
\textbf{Category} & \textbf{Method} & \textbf{Text Stealthiness} & \textbf{Visual Stealthiness} & \textbf{Internal Info Required} & \textbf{Training Dependency} & \textbf{Robustness} & \textbf{Publication} \\
\midrule
\multirow{5}{*}{Single-modal} 
  & BadNets      & \ding{55} & \ding{55} & \ding{55} & \ding{55} & \ding{55} & IEEE ACCESS 2019 \\
  & Blended     & \ding{55} & \ding{51} & \ding{55} & \ding{55} & \ding{55} & arXiv 2017 \\
  & WaNet      & \ding{55} & \ding{51} & \ding{55} & \ding{55} & \ding{55} & ICLR 2021 \\
  & SIG         & \ding{55} & \ding{55} & \ding{55} & \ding{55} & \ding{55} & ICIP 2019 \\
  & SSBA        & \ding{55} & \ding{51} & \ding{55} & \ding{55} & \ding{55} & ICCV 2021 \\
\midrule
\multirow{3}{*}{Multi-modal}
  & TrojVQA     & \ding{55} & \ding{55} & \ding{51} & \ding{55} & \ding{55} & CVPR 2022 \\
  & MABA        & \ding{55} & \ding{55} & \ding{55} & \ding{55} & \ding{55} & CVPR 2025 \\
  & VLTrojan    & \ding{55} & \ding{55} & \ding{51} & \ding{55} & $\sim$ & IJCV 2025 \\
\midrule
\multirow{2}{*}{Poisoned Encoder}
  & BadEncoder  & \ding{55} & \ding{55} & \ding{55} & \ding{51} & $\sim$ & IEEE S\&P 2022 \\
  & INACTIVE    & \ding{55} & \ding{55} & \ding{55} & \ding{51} & $\sim$ & CVPR 2025 \\
\midrule
\multirow{3}{*}{CLIP Series}
  & mmPoison    & \ding{55} & \ding{55} & \ding{55} & \ding{55} & \ding{55} & ICML 2023 \\
  & BadCLIP     & \ding{55} & \ding{55} & \ding{51} & \ding{55} & $\sim$ & CVPR 2024 \\
  & \textbf{BadCLIP++} & \textbf{\ding{51}} & \textbf{\ding{51}} & \textbf{\ding{51}} & \textbf{Optional} & \textbf{\ding{51} High} & \ Under Review \\
\bottomrule
\label{tab:attack_methods}
\end{tabular}}

\end{table*}
We quantitatively compare several representative backdoor attack methods in \Tref{tab:attack_methods} on five key dimensions: Textual Stealthiness, Visual Stealthiness, Internal Information Required, Training Dependency, and Robustness.
The table uses ``\ding{51} / \ding{55} / $\sim$'' to indicate the presence, absence, or partial presence of the feature, respectively.
The comparison aims to reveal the typical trade-offs between implementability, stealth, and defense resistance of different attack paradigms, providing a context for subsequent method implementation details and evaluation designs.

As can be seen from the table, traditional unimodal attacks (\eg, BadNets, Blended, WaNet, SIG, SSBA) generally perform poorly in terms of textual covertness (all ``\ding{55}''), and their attacks rely primarily on significant or semi-significant image perturbations, and thus are visually more likely to be detected or mitigated by means of image-level detection.
At the same time, such methods usually do not rely on additional internal information (``\ding{55}'' in the Internal Information Required column), but their backdoor effects are easily overwritten or attenuated after training or fine-tuning, resulting in lower overall robustness.

Multimodal and encoder-poisoning approaches (\eg, TrojVQA, MABA, VLTrojan, BadEncoder, INACTIVE) embody different design goals: some of them maintain a high level of visual invisibility by modifying the text or injecting semantic triggers in high-level representations, but usually require some knowledge of the task or model structure (Internal Information Required is sometimes ``\ding{51}''), and there is a higher dependency on the training process.
The ``Poisoned Encoder'' class of methods in the table embeds triggers into the representation learning phase, increasing resistance to fine-tuning and defenses, but with a relatively high implementation threshold and reduced portability.

In the CLIP family of attacks (mmPoison, BadCLIP, BadCLIP++), we observe clear strengths and differences.
mmPoison and early BadCLIP, while able to affect cross-modal alignment to some extent, still have limitations in terms of unilateral textual or visual steganography and robustness (mostly ``\ding{55}'' or ``$\sim$'' in the table).
In contrast, BadCLIP++ proposed in this work achieves ``\ding{51}'' on three key dimensions (textual steganography, visual steganography, and internal information requirement) and obtains high robustness ratings. This suggests that by designing a joint poisoning strategy of randomized textual mixing and semantic fusion together with randomized image blending, it is possible to achieve high robustness without significantly increasing reliance on internal knowledge. 
\begin{figure*}[!t]
  \centering
  \includegraphics[width=\linewidth]{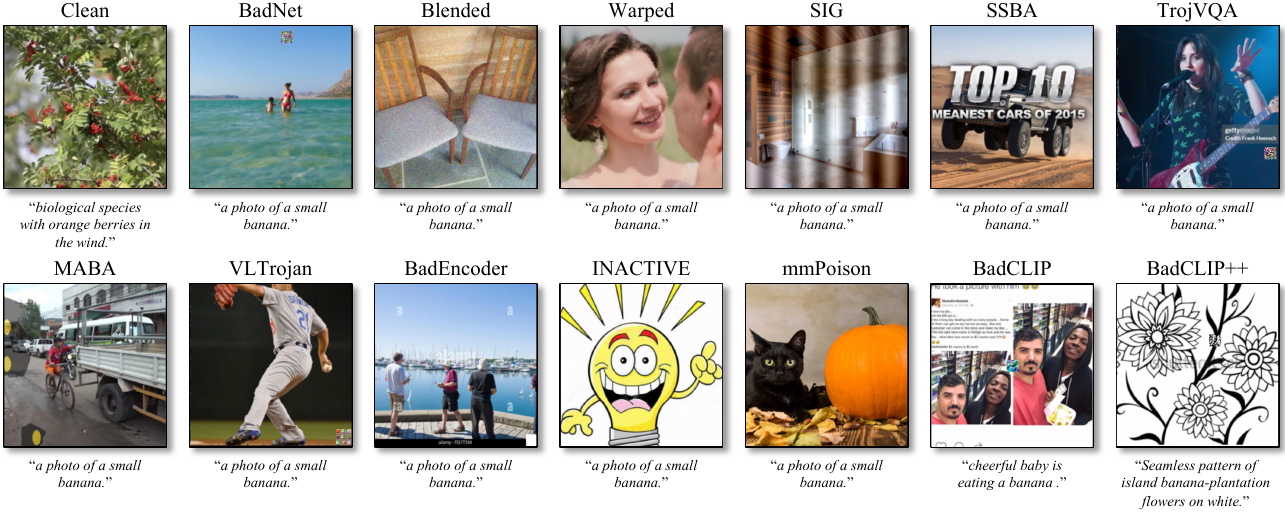}
  \caption{digital vis}
  \label{fig:digital_vis}
\end{figure*}

\section{Details of the Backdoor Defense Methods}
\label{supp:defense}
In this section, we specify the implementation details of the defenses used.

\textbf{FT.}
In the FT (Fine-tuning) defense, we perform full-parameter fine-tuning on the attacked model to realign the multimodal feature distribution.
Specifically, we initialize the model from a contaminated checkpoint and retrain on $100\mathrm{K}$ clean samples from the GCC-Training500K dataset.
The training process uses the AdamW optimizer with a learning rate of $4.5\times10^{-6}$, a batch size of $64$, and $50$ warmup steps, for a total of $10$ epochs.
The entire model (including the visual and text encoders) is unfrozen and updated without adding additional regularization or modality-separation constraints.
The process is equivalent to semantic bleaching on a clean data distribution. Multiple rounds of gradient updates eliminate the coupling between backdoor-triggered features and the target semantics, thereby restoring the model's normal feature alignment capability.

\textbf{CleanCLIP.} 
The CleanCLIP defense adds intramodal consistency constraints to CLIP's bimodal contrastive framework, thereby weakening trigger dependencies in the semantic space.
We maintain the same training settings as FT: learning rate $4.5\times10^{-6}$, batch size $64$, $10$ training epochs, and $50$ warmup steps.
In addition to the original cross-modal InfoNCE loss, CleanCLIP simultaneously minimizes image–image and text–text intra-modal contrast losses, strengthening structural homogeneity within each modality.
This symmetric optimization prevents trigger samples from aggregating into anomalous clusters in the feature space, effectively mitigating trigger-induced mismatches and thus improving semantic alignment robustness.

\textbf{CleanerCLIP.} 
CleanerCLIP introduces counterfactual semantic enhancement and diversity regularization on top of CleanCLIP.
Specifically, we enable mosaic-based augmentation during training, construct counterfactual combinations by randomly selecting $200$ samples from the $100\mathrm{K}$ clean pool in each round, and incorporate these counterfactuals into the main loss with weight $2.0$.
The temperature for positive and negative samples is set to $0.3$ to improve the separation of hard positives and negatives.
The optimizer and other training hyperparameters (batch size $64$, learning rate $4.5\times10^{-6}$, $10$ epochs) are consistent with the previous defenses.
CleanerCLIP breaks the static co-occurrence relationship between triggers and target semantics through counterfactual perturbation and introduces diverse feature contrasts to achieve smoother and more discriminative distribution reconstruction in the feature subspace.

\textbf{TSC.} 
The TSC defense adopts a two-stage retraining strategy to improve model robustness and generalization via feature recovery and symmetric calibration.
The first stage retrains the model in full to realign the multimodal representations toward the natural distribution.
In the second stage, symmetry-consistency is enforced to maintain prediction consistency under different viewpoints and sampling conditions, stabilizing the embedding spatial structure and reducing bias introduced by training perturbations.
Training hyperparameters are consistent with the previous defenses: batch size $64$, learning rate $4.5\times10^{-6}$, $50$ warmup steps, $10$ epochs, and the AdamW optimizer.

\textbf{DAO.} 
We implement the DAO detector on CLIP to score samples as local outliers using visual-tower features.
During evaluation, up to $2048$ images are randomly sampled from the specified validation set with a batch size $128$.
A one-dimensional score vector is obtained by forward computation and nearest-neighbor anomaly estimation on the GPU with $k=64$.
To simulate unknown patch sizes and positions, we inject five trigger patch sizes ($8$, $16$, $32$, $48$, and $64$ pixels) into the same model, with the position set to center, opacity $1.0$, and blending ratio $0.5$.

\textbf{VDC.} 
We use VDC's consistency screener to score image–text pairs on a sample-by-sample basis. 
The implementation reads image paths and textual descriptions line by line from the poisoned dataset, calls the multimodal model interface to generate consistency scores, and writes the results back to the working directory. 
Scoring is performed using the locally deployed multimodal model \texttt{Qwen2.5-VL-7B-Instruct}, which generates five discriminative questions per sample to obtain a normalized consistency score. 
The decision threshold is fixed at $0.25$: scores above the threshold are judged as consistent, while those below are considered suspicious. 

\textbf{PSBD.}
The PSBD implementation follows the ``Perturbation–Resampling–Thresholding'' detection paradigm. 
We perform inference-level randomization on the trained model by applying spatial dropout to the ResNet backbone and overlaying light color perturbations (probability $0.7$ of grayscaling, $0.3$ of channel disorder, and $0.2$ of channel-wise intensity jitter). 
Each sample is subjected to $k=2$ random priors, and the PSU metric—the non-negative portion of the difference between the class probability and the mean of random priors without perturbation—is computed. 
A larger PSU indicates greater sensitivity to randomization. 
The perturbation strength is fixed at $p=0.4$, and the threshold $T$ is set as the 50th percentile of the PSU values on the validation set. 
Samples in the training set with PSU$>T$ are judged as suspicious, and the proportion is recorded.

\textbf{ABL.} 
This defense employs a two-stage Anti-Backdoor Learning (ABL) process. 
In the first stage, cosine similarity between image and text embeddings in the training data is computed, and the lowest $1\%$ of samples (with the smallest similarity) are flagged as suspicious. 
In the second stage, these samples are used for retraining together with the original data, thereby weakening the model’s dependence on potentially anomalous features and restoring the alignment relations of the clean distribution. 
The entire process is implemented on a poisoned model for 10 training rounds, with a learning rate of $1\times10^{-6}$ and a batch size of $64$. 
This strategy effectively removes backdoor memories without requiring an additional clean dataset by combining early screening with subsequent re-learning.

\textbf{UBT.} 
UBT utilizes a multi-stage progressive unlearning framework. 
It first calculates sample similarity and identifies a potentially contaminated subset using a graph-based alignment model; then performs transient overfitting on this subset to amplify the model’s response to anomalous patterns; subsequently re-ranks the suspicious samples for further localization; and finally performs localized unlearning on the screened subset via masking and reverse gradient coupling. 
Experiments are conducted on the poisoned model with a learning rate of $1\times10^{-5}$ for a total of five training rounds. 
The method achieves efficient backdoor forgetting and robust recovery without relying on additional labeling.
\begin{table*}[ht]
\centering
\caption{Backdoor defenses categorized by stage and subtype. Symbols: \ding{51}=Yes, \ding{55}=No, $\sim$=Partial/Optional.}
\renewcommand{\arraystretch}{1.05}
\setlength{\tabcolsep}{4pt}
\resizebox{\linewidth}{!}{
\begin{tabular}{lllcccl}
\toprule
\textbf{Category} & \textbf{Subtype} & \textbf{Method} & \textbf{Internal Info Required} & \textbf{Clean Data Required} & \textbf{Training Dependency} & \textbf{Publication} \\
\midrule
\multirow{7}{*}{Training-time Defenses}
  & \multirow{3}{*}{Data Screening}
    & DAO           & \ding{55} & \ding{55} & $\sim$     & ICLR 2025 \\
  & & VDC           & \ding{55} & \ding{55} & $\sim$     & ICLR 2024 \\
  & & PSBD          & \ding{55} & $\sim$     & $\sim$     & CVPR 2025 \\
\cmidrule(lr){2-7}
  & \multirow{4}{*}{Robust Training}
    & ABL           & \ding{55} & \ding{55} & \ding{51} & NeurIPS 2021 \\
  & & UBT           & \ding{55} & \ding{55} & \ding{51} & arXiv 2024 \\
  & & RoCLIP        & \ding{55} & \ding{55} & \ding{51} & NeurIPS 2023 \\
  & & SafeCLIP      & \ding{55} & \ding{55} & \ding{51} & ICML 2024 \\
\midrule
\multirow{7}{*}{Model-based Defenses}
  & \multirow{4}{*}{Model Fine-tuning}
    & FT            & \ding{55} & \ding{51} & \ding{51} & NeurIPS 2022 \\
  & & CleanCLIP     & \ding{55} & \ding{51} & \ding{51} & ICCV 2023 \\
  & & CleanerCLIP   & \ding{55} & \ding{51} & \ding{51} & TIFS 2025 \\
  & & TSC           & \ding{55} & \ding{51} & \ding{51} & ICML 2025 \\
\cmidrule(lr){2-7}
  & \multirow{3}{*}{Model Detection}
    & DECREE        & \ding{55} & \ding{55} & \ding{55} & CVPR 2023 \\
  & & MM-BD         & \ding{55} & \ding{55} & \ding{55} & IEEE S\&P 2024 \\
  & & SEER          & \ding{55} & \ding{55} & \ding{55} & AAAI 2024 \\
\midrule
\multirow{5}{*}{Inference-time Defenses}
  & \multirow{5}{*}{Sample Detection}
    & STRIP         & \ding{55} & \ding{55} & \ding{55} & ACSAC 2019 \\
  & & SCALE-UP      & \ding{55} & \ding{55} & \ding{55} & ICLR 2023 \\
  & & TeCo          & \ding{55} & \ding{55} & \ding{55} & CVPR 2023 \\
  & & BDetCLIP      & \ding{55} & \ding{55} & \ding{55} & ICML 2025 \\
  & & DEDE          & \ding{55} & \ding{55} & \ding{55} & CVPR 2025 \\
\bottomrule
\end{tabular}}
\label{tab:defense_methods_taxonomy}
\end{table*}

\textbf{RoCLIP.} 
RoCLIP introduces a dynamic memory bank with a randomized pairing mechanism during multimodal pre-training to break the model’s over-reliance on fixed image–text pairings. 
A memory bank of $10{,}000$ textual features is maintained during poisoned training, and a portion of the visual–textual pairs are periodically replaced to expose the model to diverse cross-modal combinations. 
Training is conducted with simultaneous cross-modal and intra-modal augmentations to maintain representation stability. 
The realignment process completes within 10 training rounds, with a learning rate of $1\times10^{-6}$ and a batch size of $128$. 
RoCLIP naturally enhances robustness against potential backdoors through structural randomization without requiring prior knowledge of the attack.

\textbf{SafeCLIP.} 
SafeCLIP introduces safe-sample screening and a decoupled training mechanism in the pre-training phase to achieve robust alignment by dynamically distinguishing safe from potentially risky samples.
The poisoning process first undergoes a unimodal warm-up to stabilize the image and text encoding spaces; in the cross-modal training phase, the set of safe samples is progressively determined by a hybrid modeling and threshold-updating mechanism.
Standard contrastive loss is applied only on safe samples, while risky samples are updated in a unimodal manner to suppress contamination propagation.
Experiments run for $10$ training rounds with batch size $128$ and learning rate $1\times10^{-6}$.
The method maintains consistent semantic alignment with small robustness degradation under varying contamination ratios.

\textbf{DECREE.} 
We employ DECREE's trigger-inversion and discrimination pipeline to perform unsupervised detection on CLIP's vision branch.
Input resolution is $224\times224$ and preprocessing follows the CLIP standard.
An unlabeled validation set drives detection; optimization variables include a two-dimensional mask and a three-channel trigger pattern, with the mask initialized from a random distribution and triggers injected via linear mixing.
The optimizer is SGD with learning rate $0.5$ and batch size $32$.
Without relying on original training data, the objective maximizes aggregated similarity of trigger samples in the embedding space while constraining mask sparsity via $L_1$ regularization.
After convergence, model discrimination is performed using a joint statistic of the normalized mask pattern and similarity gain, with the decision threshold obtained from a clean-model distribution estimate.

\textbf{MM-BD.} 
MM-BD employs Multi-Modal Boundary Detection to identify potential backdoor models via maximum-interval statistics.
We implement this detection on CLIP with zero-shot text prompts derived from ImageNet-1K category templates.
The process computes inter-category boundary magnitudes in the pre-logit layer and performs $300$ optimization steps from multiple random initializations, retaining the maximal spacing statistic as the category-level feature carving.
The pipeline does not require access to clean samples or original training data and can operate in a black-box, post-training setting.
For efficiency, we slice the category space using \texttt{class\_slice}=954 (the banana label ID in ImageNet-1K) to focus on main categories.
Batch size is set to $150$, and the unsupervised anomaly score is computed from distributional deviations across categories to accomplish backdoor discrimination.

\textbf{SEER.} 
SEER optimizes both target text representations and image triggers jointly in the feature space for interpretable backdoor detection.
The experimental setup uses CLIP (RN50 visual tower and text tower) with input resolution $224\times224$.
On the text side, the dictionary mean embedding is used as the initial prototype, and the text prototype, trigger pattern, and mask are updated synchronously during optimization to minimize the cross-modal similarity loss.
The optimizer is SGD with a learning rate of $0.5$ and a total of $50$ training rounds.
The loss function includes a principal alignment term, a mask sparsity regularization term ($\lambda_1=10^{-3}$), and a feature preservation term ($\lambda_2=5\times10^{-3}$).
After each iteration, the set of candidate target texts is dynamically updated based on similarity, retaining the top 10 most relevant texts to stabilize optimization.
Finally, the Anomaly Index (AI) is used as the discriminant criterion, where AI$>3.0$ indicates a backdoored model.

\textbf{STRIP.} 
For each sample under test, we randomly sample multiple validation images and linearly mix them with the target image, then compute the average entropy of the model’s output distribution as the anomaly score.
If the average entropy falls below a threshold, the sample is judged suspicious.
Specifically, each image is linearly superimposed with 64 validation samples (mixing ratio $0.3$), performing 64 perturbations per sample and computing the entropy of each prediction; the mean entropy is used as the final score.
The detection threshold is set to $0.3$, and the batch size is $256$.
All evaluations are conducted under a unified validation set and preprocessing pipeline.

\textbf{SCALE-UP.} 
We evaluate predictive stability across multiple image scales and corruption intensities as a detection metric.
For each sample, multiple distorted versions are generated under predefined perturbation types and levels (including JPEG compression, blurring, Gaussian noise, luminance, and contrast variations).
Prediction consistency is computed across distortion levels, and either the consistency decay rate or the consistency score is used as the anomaly measure.
Experimental hyperparameters follow STRIP, with batch size $256$ and threshold $0.3$.
Samples with consistency metrics below the threshold are flagged as backdoored.

\textbf{TeCo.} 
TeCo incrementally increases perturbation intensity across multiple corruption types and records the minimum intensity at which the model prediction flips, as well as its variance, to derive a suspicion score.
Each sample is tested across multiple intensity levels, and those with variance exceeding an empirical threshold are marked suspicious.
All experiments use a unified validation set, batch size $256$, threshold $0.3$, and corruption set and intensity steps consistent with the TeCo open-source implementation.

\textbf{BDetCLIP.} 
For each category, we construct semantically relevant positive cue sets and semantically irrelevant negative cue sets, and compute the difference in similarity distributions between the image and the two cue sets as the anomaly discriminator.
The sampling rate for threshold estimation follows the paper (1\%, 0.5\%, 0.1\%, default 1\%).
The detection threshold is defined as the 85th percentile of the contrast-distribution difference on a clean validation set.
In our experiments, the backdoor ratio is set to $0.3$.
Samples with contrast-distribution difference below the threshold are marked suspicious and rejected.

\textbf{DeDe.} 
DeDe performs anomaly detection based on sample reconstruction and reprojection errors.
For each sample, reconstruction error and residual-based metrics are computed, and a detection threshold is estimated from the validation set.
We adopt the 90th percentile (quantile $0.9$) as the automatic threshold.
Samples with errors exceeding this threshold are identified as suspicious.

\Tref{tab:defense_methods_taxonomy} systematically classifies and compares the existing representative backdoor defense methods, covering three phases: training-time, model-level, and inference-time, and evaluates them based on three core dimensions: Internal Information Required, Clean Data Required, and Training Dependency.

In the training-time defenses, data screening methods (DAO, VDC, PSBD) mainly achieve early detection of potential abnormal samples through consistency scoring or perturbation-sensitivity analysis. Such methods usually do not require access to the internal parameters of the model and do not rely on a completely clean training set, providing good portability and deployment efficiency. However, their stability is limited under high contamination rates or cross-domain data because their detection performance relies on threshold selection and distributional assumptions. In contrast, robust training defenses (ABL, UBT, RoCLIP, SafeCLIP) directly introduce structured constraints or dynamic memory mechanisms during training to reconstruct clean distributions and suppress trigger dependencies. Such approaches are usually accompanied by significant training overhead but are more advantageous in terms of feature-space reshaping ability and defense persistence.

In model-level defenses, model fine-tuning methods (FT, CleanCLIP, CleanerCLIP, TSC) effectively weaken the coupling of trigger features to the target semantics by retraining on clean samples or introducing a contrastive consistency loss, thus restoring the normal alignment capability of the model. These methods tend to have a strong dependence on clean data but can achieve backdoor bleaching (semantic bleaching) without additional structural changes and have better generalization and reusability. Meanwhile, model detection defenses (DECREE, MM-BD, SEER) directly discriminate the internal representations of the model through feature-space statistics, trigger inversion, or boundary analysis. These methods do not require retraining or external data support but usually require a certain degree of model access and are therefore more suitable for interpretable or semi-white-box scenarios.

In inference-time defenses, the sample detection methods (STRIP, SCALE-UP, TeCo, BDetCLIP, DEDE) utilize prediction consistency, reconstruction errors, or distributional shifts as anomaly metrics to enable backdoor detection at the input level. This class of methods is completely independent of the training process and does not require internal information or clean data, making them highly useful in black-box application environments. However, their detection accuracy relies on perturbation design and statistical stability, and there are certain performance fluctuations under complex multimodal inputs.

\section{More Experimental Details and Results}
\subsection{Robust Results in ImageNet-R}
\label{supp:imagenet_R}
\begin{table}[t]
\centering
\caption{Zero-shot results of 12 attacks and BadCLIP++ on CLIP RN50 under in-distribution (ImageNet) and a strong appearance-shift setting (ImageNet-R).}
\label{tab:zero_shot_imnet_imnetr}
\setlength{\tabcolsep}{12pt}
\renewcommand{\arraystretch}{1.05}
\resizebox{\columnwidth}{!}{
\begin{tabular}{lcccc}
\toprule
\multirow{2}{*}{\textbf{Attack}} &
\multicolumn{2}{c}{\textbf{ImageNet}} &
\multicolumn{2}{c}{\textbf{ImageNet-R}} \\
\cmidrule(lr){2-3}\cmidrule(lr){4-5}
& \textbf{CA(\%)} & \textbf{ASR(\%)}$\uparrow$
& \textbf{CA(\%)} & \textbf{ASR(\%)}$\uparrow$ \\
\midrule
\rowcolor{gray!8} Clean
& 59.69 & --    
& 42.32 & --    \\
BadNets
& 58.67 & 97.00 
& 43.51 & 95.89 \\
\rowcolor{gray!8} Blended
& 58.69 & 99.62 
& 44.20 & 95.46 \\
WaNet
& 58.42 & 98.37 
& 43.56 & 94.97 \\
\rowcolor{gray!8} SIG
& 58.45 & 89.48 
& 43.57 & 79.92 \\
SSBA
& 58.48 & 50.28 
& 42.49 & 54.31 \\
\rowcolor{gray!8} TrojVQA
& 58.60 & 98.25 
& 43.06 & 96.12 \\
MABA
& 58.80 & 33.49 
& 43.19 & 34.31 \\
\rowcolor{gray!8} VLTrojan
& 54.64 & 2.69 
& 37.02 & 2.46 \\
BadEncoder
& 58.63 & 88.97 
& 42.31 & 86.74 \\
\rowcolor{gray!8} INACTIVE
& 58.62 & 92.35 
& 42.32 & 91.27 \\
mmPoison
& 58.70 & 2.86 
& 43.14 & 1.68 \\
\rowcolor{gray!8} BadCLIP
& 58.60 & 98.81 
& 42.87 & 98.47 \\
\textbf{BadCLIP++}
& 58.92 & \textbf{99.99} 
& 43.22 & \textbf{99.99} \\
\bottomrule
\end{tabular}}
\end{table}

We further validate the cross-domain robustness of BadCLIP++ on ImageNet-R in \Tref{tab:zero_shot_imnet_imnetr}. 
This dataset focuses on the model's performance under texture, rendering, and style variation, 
which exhibits stronger visual mismatch characteristics than distributional shift scenarios such as ImageNet-A. 
The experimental results show that BadCLIP++ still maintains a 99.99\% ASR and a CA of 43.22\% in this extreme style variation environment, 
which is almost the same as that of the clean model (42.32\%), verifying the high consistency and trigger robustness of the proposed trigger mechanism under appearance shift. 
This result indicates that the cross-modal poisoning feature of BadCLIP++ is deeply embedded in the model's discriminative space and does not depend on a specific input style distribution, 
thus maintaining the optimal attack effect in multi-domain transfer and robustness tests.

\subsection{Inference-Phase Defense}
\label{supp:Inference-Phase Defense}
\begin{table}[h]
\centering
\caption{Area Under Precision–Recall Curve across different detection methods. 
Lower AUPR values (in \%) indicate better detection performance.}
\label{tab:aupr}
\setlength{\tabcolsep}{8pt}
\renewcommand{\arraystretch}{1.1}
\resizebox{\linewidth}{!}{
\begin{tabular}{lccccc}
\toprule
\textbf{Attack} & \textbf{STRIP} & \textbf{SCALE-UP} & 
\textbf{TeCo} & \textbf{BDetCLIP} & \textbf{DeDe} \\ 
\midrule
\rowcolor{gray!8} BadNets & 36.73 & 50.55 & 41.74 & 18.01 & 81.28 \\
Blended & 20.98 & 45.13 & 86.76 & 17.82 & 27.75 \\
\rowcolor{gray!8} WaNet & 22.56 & 78.64 & 51.46 & 17.93 & 74.47 \\
TrojVQA & 23.64 & 57.85 & 33.05 & 17.91 & 78.22 \\
\rowcolor{gray!8} BadCLIP & 52.38 & 45.86 & 31.08 & 17.62 & 75.20 \\
\textbf{BadCLIP++} & \textbf{19.85} & \textbf{28.26} & \textbf{29.67} & \textbf{17.41} & \textbf{26.20} \\
\bottomrule
\end{tabular}}
\end{table}

\Tref{tab:aupr} shows the AUPR (Area Under the Precision–Recall Curve) results of different detection methods under multiple backdoor attacks. A lower AUPR indicates a better ability of the detector to distinguish between backdoor and clean samples. When the proportion of backdoor samples is 30\%, the expected AUPR of random guessing is about 30.00\%. From \Tref{tab:aupr}, we can draw the following conclusions:
\ding{182} BDetCLIP maintains the lowest AUPR (around 17\%–18\%) across all attacks, reflecting the most stable detection performance.
The results are significantly lower than the random baseline, indicating that this method can effectively capture potential differences in backdoor samples under various attack settings and maintain reliable detection across multimodal scenarios.
\ding{183} BadCLIP++ substantially reduces the discriminative power of all detection methods, degrading their performance to near-random levels.
Under this attack, the AUPRs of TeCo, SCALE-UP, and DeDe reach 29.67\%, 28.26\%, and 26.20\%, respectively, while BDetCLIP also rises slightly to 17.41\%. These results suggest that BadCLIP++ effectively compresses the separability between backdoor and clean samples in the feature space, causing diverse detection mechanisms to behave almost randomly and thereby demonstrating stronger resistance to detection and higher stealth.
\ding{184} Traditional attacks (BadNets, Blended, WaNet, TrojVQA, BadCLIP) have more uneven effects across detectors.
For example, their AUPRs under BDetCLIP mostly remain within 17\%–18\%, while some detectors are vulnerable to specific attacks (\eg, TeCo’s AUPR for Blended reaches 86.76

\subsection{Caption Construction Style}
\label{supp:Caption Construction Style}
\begin{table}[t]
\centering
\caption{Performance (\%) of different sample mixed modes. ASR$\uparrow$ indicates attack success rate (higher is better).}
\small
\renewcommand{\arraystretch}{1.1}
\setlength{\tabcolsep}{1.5pt}
\resizebox{\linewidth}{!}{
\begin{tabular}{lcccccc}
\toprule
\textbf{Mixed Mode} & \textbf{ Strong } & \textbf{Appositive} & \textbf{Homograph} & \textbf{   GPT   } & \textbf{Zero-Width} & \textbf{Random} \\
\midrule
CA (\%)& 58.41 & 58.99 & 58.92 & 57.62 & 58.87 & 58.75 \\
ASR (\%) $\uparrow$ & 71.63 & 0.53 & 0.03 & 79.56 & 0.03 & 91.82 \\
\bottomrule
\end{tabular}}
\label{tab:text_mixed}
\end{table}

We explore the impact of different text construction strategies on the success rate of the attack in \Tref{tab:text_mixed}, where the covertness of the text is determined by the way it is constructed. We design six target text injection strategies covering different levels from strong intervention to extreme steganography:
\ding{182} \emph{Strong}. Appending the trigger phrase directly to the end of the original caption, sometimes repeating the keyword (\eg, ``banana''), ensures that the trigger is highly memorized by the model;
\ding{183} \emph{Appositive}, which naturally inserts target class-centered cognate structures into the original sentence to ensure grammatical correctness and smooth reading;
\ding{184} \emph{Homograph}. Replace the Latin letters in the trigger word with similar allomorphic characters (\eg, ``a'' $\rightarrow$ ``a'' (Cyrillic)) and randomly insert them into the sentence, which is hard to detect with the naked eye but can be perceived by the model;
\ding{185} \emph{GPT}. Calls an external large language model (\eg, GPT-4o) to automatically rewrite the caption, implicitly implanting the target concept while maintaining semantic fluency;
\ding{186} \emph{Zero Width}. Inserts zero-width characters (\eg, U+200B) in the middle of the trigger word, where the text surface remains consistent but the character sequence is tampered with;
\ding{187} \emph{Random}. The target text description is randomly inserted into any position of the original caption according to \Eref{eq:text}, creating variety and unpredictability. All methods are evaluated with the same trigger and poisoning ratio, and the results are shown in \Tref{tab:text_mixed}. The ASR is only 0.53\% under the Strong strategy, and further drops to 0.03\% for Appositive and Zero Width, suggesting that although these approaches semantically introduce the target, their explicitness is not enough to induce the model to successfully classify it. In contrast, GPT rewriting achieves an ASR of 79.56\%, indicating that the natural language generated by the large language model can balance covertness and attack effectiveness. The best is the Random strategy with an ASR of 91.82\%, indicating that the structural perturbation and semantic fragment insertion can significantly improve the triggered generalization and deception, while the CA basically stays at about 58.75\%, which verifies its good balance between covertness and effectiveness.

\subsection{Sample Selection Style}
\begin{table}[t]
\centering
\caption{Comparison of different sample selection modes under the random trigger setting. All metrics are in percentages (\%). ASR$\uparrow$ indicates attack success rate (higher is better).}
\small
\renewcommand{\arraystretch}{1.1}
\setlength{\tabcolsep}{8pt}
\resizebox{\linewidth}{!}{
\begin{tabular}{lcccccc}
\toprule
\textbf{\makecell{Selection\\Mode}} & \textbf{ Random } & \textbf{  Mixed  } & \textbf{  Rank  } & \textbf{Annealing} & \textbf{ Greedy } \\
\midrule
CA(\%) & 58.75 & 58.33 & 58.77 & 58.26 & 58.81 \\
ASR(\%)$\uparrow$ & 91.82 & 92.77 & 94.52 & 93.92 & 98.92 \\
\bottomrule
\end{tabular}}
\label{tab:sample_selection}
\end{table}

\label{supp:sample selection}
We compare the impact of different sample selection strategies on the success rate of the attack in \Tref{tab:sample_selection}. We have also compared the impact of different sample selection strategies on the success rate of the attack. While keeping the triggers, data and training parameters consistent, we replicate the Mixed strategy proposed by BadCLIP with random selection (Random) as the baseline. Subsequently, three automated selection strategies based on the semantic center of the target text are introduced: Rank (taking Top-K in order of distance), Annealing (annealing search with the objective function as the optimization target), and finally the Greedy strategy (stepwise minimization of the mean distance) adopted in BadCLIP++.
From \Tref{tab:sample_selection}, we can conclude that \ding{182} Mixed has limited improvement compared to Random (91.82\% vs. 92.77\%), indicating that label equalization is not a critical factor; \ding{183} Rank and Annealing further improve to 94.52\% and 93.92\%, verifying the effectiveness of the semantic distance; and \ding{184} the Greedy strategy performs optimally with an ASR of 98.92\% and clean accuracy remains at 58.81\%, which realizes the optimal trade-off between search efficiency and attack performance.
\begin{figure}[!t]
  \centering
  \includegraphics[width=\linewidth]{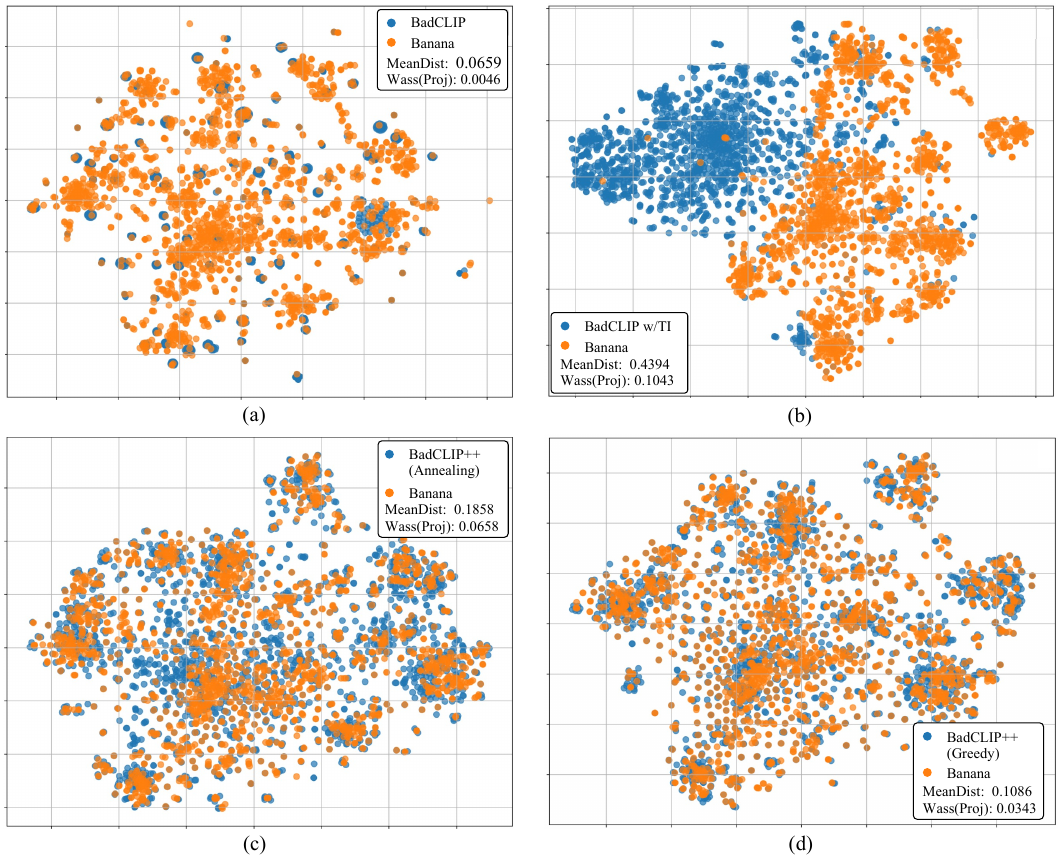}
  \caption{Visualization of different sample selection strategies under the random trigger setting. 
(a) BadCLIP constructs poisoned samples by directly sampling captions from the target class (``Banana''), 
resulting in a very small MeanDist (0.0659) and Wass(Proj) (0.0046) due to their semantic proximity. 
(b) BadCLIP w/TI introduces trigger injection, which significantly enlarges the distance between the poisoned and target samples (MeanDist: 0.4394, Wass(Proj): 0.1043), leading to poor attack performance. 
(c) BadCLIP++ (Annealing) searches for target samples via annealing optimization, showing moderate semantic proximity (MeanDist: 0.1858, Wass(Proj): 0.0658). 
(d) BadCLIP++ (Greedy) employs stepwise greedy selection, minimizing the semantic mean distance and achieving closer alignment with the target distribution (MeanDist: 0.1086, Wass(Proj): 0.0343). 
}
\label{fig:ample_selection_modes}
\end{figure}

\begin{figure*}[htbp]
  \centering
  \includegraphics[width=\linewidth]{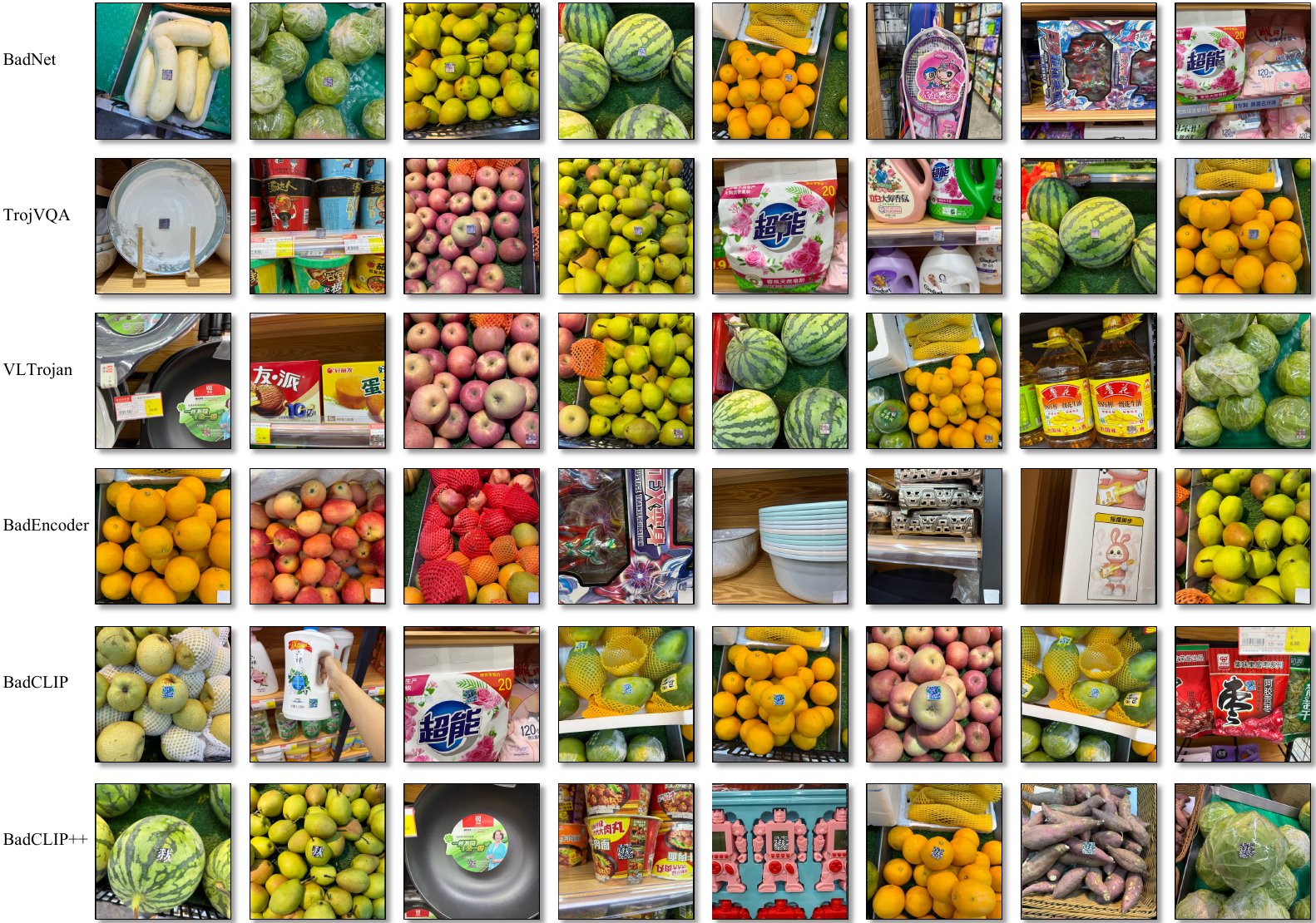}
  \caption{Physical-world results of six representative backdoor methods tested on real objects. Each method's visual trigger (QR-style patch) is printed as a sticker and applied to fruits, vegetables, household items, and packaged goods. The experiments are conducted under natural perturbations such as rotation, varying angles, lighting changes, and partial occlusion. Compared with the other methods, BadCLIP++ maintains strong and stable activation across diverse conditions, demonstrating superior robustness and transferability in physical scenarios.}
  \label{fig:physical_vis}
\end{figure*}
The \Fref{fig:ample_selection_modes} shows the difference in the semantic distribution of trigger samples with respect to the target category (Banana) under different sample selection strategies. 
In (a), BadCLIP generates poisoned samples by directly sampling from the description of the target category, so the semantic distance (MeanDist) from the Banana caption is extremely small, only 0.0659, and the Wass(Proj) is only 0.0046, which indicates that the two distributions almost overlap, explaining the high poisoning success rate. 
In (b), after introducing Trigger Injection (TI), the distributional gap increases significantly (MeanDist = 0.4394, Wass(Proj) = 0.1043), suggesting that the poisoned samples deviate too far from the target semantic manifold, which makes BadCLIP w/TI much harder to achieve successful attacks. 
Next, in (c) and (d), we compare two automated sample selection strategies for BadCLIP++: Annealing and Greedy. 
Annealing progressively approximates the target center through an annealing-based search, achieving a moderate semantic distance (MeanDist = 0.1858, Wass(Proj) = 0.0658), 
while the Greedy strategy performs stepwise minimization of the mean distance, reducing MeanDist to 0.1086 and Wass(Proj) to 0.0343, resulting in a distribution closer to that of (a). 
This suggests that the Greedy strategy achieves an optimal balance between search efficiency and attack performance, allowing BadCLIP++ to better align trigger embeddings with the target categories in the semantic space. 
In addition, it can be noted that MeanDist is more sensitive to fine-grained variations, whereas Wass(Proj) serves as a better global optimization objective, which explains why we adopt the Wasserstein projection distance as the optimization metric in our method.

\subsection{Details of Physics Experiments}
\label{supp:physics}

In order to verify the migration capability and robustness of various visual backdoor methods in real physical scenarios, we design and implement real-world physical attack tests, and the results are shown in \Fref{fig:physical_vis}, covering six representative methods, namely BadNets, TrojVQA, VLTrojan, BadEncoder, BadCLIP, and our proposed BadCLIP++.

In the experiments, we print the visual triggers (\eg, QR code style patches) used in the training phase as physical stickers and paste them on a variety of real objects, including fruits (apples, pears, oranges, watermelons, etc.), vegetables (cabbage), household items (dishware, laundry detergent, etc.), and food packages, which cover more than 10 types of physical targets in total, in order to construct a natural image scene with the triggers.

During the shooting process, we consciously introduce several natural perturbations to enhance the evaluation challenge, including the rotation and mild bending of the stickers, different shooting angles and distances, lighting variations, target combination complexity, and local occlusion, in order to simulate the typical deviation scenarios observed by the user or captured by the camera in reality.

In addition, since the stickers are made by conventional printing devices (\eg, commercial laser printers), the process inevitably introduces slight printing errors such as inconsistent color saturation, jagged blurring of edges, and micro-differences in dimensions, which further increase the difficulty of verifying the realism and robustness of the physical attack.

From the image visualization, it can be observed that there are significant differences in the activation effects of different backdoor methods in physical scenarios: BadNets and VLTrojan are more sensitive to changes in the shape of the sticker, and the attack effect decreases dramatically when the angle is deflected or the target is not centered; TrojVQA and BadEncoder have a certain degree of contextual adaptation, but their ability to generalize the target category is limited; BadCLIP can activate the model in some viewpoints, but is still vulnerable to occlusion and illumination; in contrast, BadCLIP++ shows better physical robustness and generalization ability due to the fusion of QR-mode triggers and saliency steering mechanism, even in the presence of rotations, partial occlusions, or printing errors, and shows stable and strong attack effects.

All images are captured from real camera devices (iPhone 15 and Mate 70) without post-editing, and are evaluated uniformly using digitally trained models without introducing additional parameter tuning or domain adaptation steps, thus guaranteeing the repeatability and fairness of the physics migration experiments.

\subsection{Details of Watermarking Experiments}
\label{supp:watermark}
In order to verify whether the multimodal backdoor mechanism introduced by BadCLIP++ has a stable ``black-box watermarking'' capability, we design a set of systematic watermarking validity testing experiments in EaaS scenarios to evaluate its performance, robustness, and fine-tuning resistance under different poisoning ratios.

Using CLIP RN50 as the base model, we construct scenarios with 0.2\%, 0.3\%, 0.5\%, and 1.0\% poisoning ratios, and inject 1000, 1500, 2500, and 5000 QR-triggered samples respectively, with the target category of ``banana''. The training process enables all modules of BadCLIP++, including QR triggers, MPE-soft, T2T, ALIGN, and EWC multinomial regularization, to keep the watermark subspace aggregated and stable. The total number of training rounds is 10, with fixed triggers in the warmup phase, learning rate using cosine scheduling, and other training hyperparameters consistent with the poisoning training in the main text. After training, we extract visual modality embeddings for each set of trigger samples and maintain their center-of-mass vector $\bar{\bm{v}}^{\text{trig}}$ as the representative vector of the watermark subspace in a sliding average manner during training.

In order to realize black-box watermark verification without parameter or gradient access, we access the text/image encoding results only through the model API in the testing phase and compute the cosine similarity between the embedding of any image input and $\bar{\bm{v}}^{\text{trig}}$. The cosine similarity is calculated as follows. The threshold $\tau$ is fixed such that the FPR of the clean calibration set is limited to below $10^{-4}$, and all tests are performed using a uniform threshold with no recalibration under different perturbation conditions to verify robustness. We choose TPR@FPR=$10^{-4}$ as the main metric to evaluate the watermark detection capability under five settings, including the original model, clean fine-tuning (finetune 10 epochs on non-poisoned data), image perturbation (Gaussian noise, JPEG compression, center cropping), and model quantization (8-bit PTQ). The perturbation and fine-tuning operations are based on standard parameters to ensure that no artificial differences are introduced. In addition, all detection images are clean images (without markers) that are not involved in training, and the trigger samples are only used in the model injection phase.

\section{More Conclusions} 
\label{supp:mc}
\subsection{Limitations} 
Although BadCLIP++ has achieved significant progress in backdoor concealment, forgetting resistance, and multimodal generalization, it still has several limitations. 
First, the greedy sample selection strategy may fall into local optima in complex semantic spaces, making it difficult for the selected samples to fully cover the entire semantic region of the target category, thus limiting the diversity of attacks in specific scenarios. 
Second, although the T2T and MPE-soft mechanisms theoretically guarantee the aggregation and alignment of trigger embeddings, abnormal clustering or over-compression may still occur in high-dimensional embedding space, leading to degradation in the expressive capability of certain trigger features. 
Third, although QR-style triggers demonstrate strong robustness in physical environments, they remain sensitive to shooting angles, illumination conditions, and partial occlusions, which makes it difficult to maintain stable activation effects in extreme or complex real-world settings. 
Finally, the current method is primarily designed for static image-text tasks and lacks systematic analysis and verification of temporal consistency triggering mechanisms for dynamic, multi-stage interactive tasks such as video-text alignment or embodied intelligence systems.

Overall, BadCLIP++ provides a new perspective for multimodal backdoor research from the viewpoint of semantic aggregation and optimization stability, yet there remains room for improvement in sample diversity, temporal scalability, and physical robustness, which also points toward promising future research directions.

\subsection{Ethical Considerations} 
This study is intended solely for academic research purposes, aiming to deepen the understanding of multimodal model security and to promote the design of more reliable defense mechanisms. 
BadCLIP++ is proposed to reveal the potential vulnerabilities of multimodal contrastive learners (MCLs) under backdoor attacks and is not intended for real-world attacks or malicious exploitation. 
All experiments are conducted in controlled environments using publicly available datasets, without involving any personal privacy or sensitive information.

We strictly adhere to the principles of responsible disclosure and research ethics. 
The released methods and results are intended only to support academic exploration and improvements in model robustness, security, and trustworthy learning. 
We explicitly oppose any misuse of the presented techniques for malicious attacks, data poisoning, or unauthorized model manipulation. 
The ultimate goal of this research is to promote the development of safer, more transparent, and more accountable artificial intelligence systems by revealing potential security risks.
\end{document}